\documentclass[12pt]{article}
\usepackage{amsmath}
\usepackage{amsthm}
\usepackage[pdftex]{graphicx}
\usepackage{psfrag,epsf}
\usepackage{enumerate}
\usepackage{natbib}
\usepackage{subcaption}
\usepackage[colorlinks,citecolor=blue,urlcolor=blue]{hyperref}

\usepackage[boxruled,scleft]{algorithm2e}

\usepackage{geometry}
\RestyleAlgo{boxruled}

\numberwithin{equation}{section}

\newtheorem{thm}{Theorem}
\newtheorem{lem}{Lemma}

\bibpunct{(}{)}{;}{a}{}{,}

\begin{document}

\begin{center}
{\large\bf Community Detection in Multi-Relational Data Through Restricted Multi-Layer Stochastic Blockmodel}

\vspace{0.2in}

Subhadeep Paul and Yuguo Chen
\footnote{Subhadeep Paul is Ph.D Candidate, Department of Statistics,
University of Illinois at Urbana-Champaign, Champaign, IL 61820 (E-mail: {\it spaul10@illinois.edu}). Yuguo Chen
is Professor, Department of Statistics, University of Illinois at
Urbana-Champaign, Champaign, IL 61820 (E-mail: {\it yuguo@illinois.edu}).
This work was supported in part by National Science Foundation grant DMS-1406455.
%The authors thank the editor, the associate editor, and two referees for valuable suggestions.
}

\end{center}

\begin{abstract}

In recent years there has been an increased interest
in statistical analysis of data with multiple types of relations among
a set of entities. Such multi-relational data can be represented as multi-layer graphs where the set of vertices represents the entities and multiple types of edges represent the different relations among them. For community detection in multi-layer graphs, we consider two random graph models, the multi-layer stochastic blockmodel (MLSBM) and a model with a restricted parameter space, the restricted multi-layer stochastic blockmodel (RMLSBM). We derive consistency results for community assignments of the maximum likelihood estimators (MLEs) in both models where MLSBM is assumed to be the true model, and either the number of nodes or the number of types of edges or both grow. We compare MLEs in the two models with other baseline approaches, such as separate modeling of layers, aggregating the layers and majority voting. RMLSBM is shown to have advantage over MLSBM when either the growth rate of the number of communities is high or the growth rate of the average degree of the component graphs in the multi-graph is low. We also derive minimax rates of error and sharp thresholds for achieving consistency of community detection in both models, which are then used to compare the multi-layer models with a baseline model, the aggregate stochastic block model. The simulation studies and real data applications confirm the superior performance of the multi-layer approaches in comparison to the baseline procedures.

\vspace{ 6mm}

\noindent
KEY WORDS: Community detection; Consistency; Minimax rates;
%Maximum likelihood inference;
Multi-layer networks; Sharp thresholds; Stochastic blockmodel.

\end{abstract}

\newpage

%\bibliographystyle{natbib}

%\def\spacingset#1{\renewcommand{\baselinestretch}%
%#1}\small\normalsize} \spacingset{1}

%%%%%%%%%%%%%%%%%%%%%%%%%%%%%%%%%%%%%%%%%%%%%%%%%%%%%%%%%%%%%%%%%%%%%%%%%%%%%%

\newpage
%\spacingset{1.437} % DON'T change the spacing!

\section{Introduction}

Over the last decade, relational data has become ubiquitous in all
forms of human activities. In many applications of statistics and machine
learning, one encounters relational data where the entities are represented
as nodes or vertices and the relations or interactions between the
entities as edges of a graph. Applications of such graphs or networks include
many information systems such as social networks, World Wide Web,
user information databases in e-commerce,
metabolic networks, gene regulatory networks, protein-protein interaction
networks and food web.

In majority of the cases dealt with in the literature, the relations
are assumed to be of the same type such as web page linkage, friendship,
co-authorship and protein-protein interaction. However in modern complex
relational databases and networks, we often have information regarding
relationships of multiple types among the nodes. For example, in the
context of internet services a set of users may be connected through
email, messaging, social media, etc., each one of them creating one
layer or type of the user-user interaction network \citep{pai13}. Similarly, users in a social network can have ``friendship'',
``mentions'', ``following'', etc. \citep{gp13} or
researchers in academia may have co-authorship, citations, title/abstract
similarity, etc., as different types of relations among themselves.
In genomics data, cellular components can have different aspects of
interactions among them, e.g., protein-protein physical interactions
and gene co-expressions \citep{nvsz10}.
Such multi-relational data can
be represented as multi-layer graphs where multiple types of edges
represent the relations and the set of vertices/nodes represents the
entities \citep{jlbo12}.

One of the most important and widely investigated learning goals in
an information network is clustering the entities on the basis of
the relationships between them into densely connected subsets called
``communities''. From a probabilistic point of
view, communities can be thought of as groups of vertices which are
more likely to be connected to each other compared to the rest of the
graph, i.e., the probability of having an edge between two vertices belonging
to the same group is higher than that of having an edge between vertices
belonging to different communities. Consequently we would observe
the number of intra community edges to be higher than inter community
edges.

Many researchers have proposed methods and algorithms for community
detection in networks. Such methods can broadly be divided into three
categories: methods based on probabilistic models, methods based on
the maximization of a global objective function and those based on spectral
or matrix factorization of the adjacency matrix or the Laplacian matrix.
The stochastic blockmodel \citep{hll83,ns01} is a statistical model for random graphs with a natural community structure. It is one of a large class of statistical models
described in the literature for community detection in complex networks,
which includes the latent variable \citep{hrt07} and latent space
models \citep{hrh02}, the degree corrected blockmodel \citep{kn11, zlz12} and the mixed membership blockmodel \citep{abfx08}.
Various likelihood maximization based inference strategies have been
proposed in the literature to simultaneously infer the block assignments
and the parameters in the stochastic blockmodel, e.g., profile likelihood
maximization \citep{bc09}, maximizing the conditional likelihood
\citep{cwa12}, and variational EM under mixture model settings
\citep{dpr08}. Other strategies involve Bayesian inference
using Gibbs sampling or variational methods \citep{lba11} and optimizing a modularity function over all possible partitions of the graph \citep{ng04}. See \citet{gzfa10} for a detailed review of statistical inference in networks.

Several authors have also studied the conditions required on the growth
of the number of communities and the degree density of networks for the
estimation strategies to be consistent. \citet{bc09} and
\citet{zlz12} studied the conditions for community detection
through modularity maximization under the stochastic blockmodel and the degree
corrected stochastic blockmodel respectively. \citet{cwa12}
laid down the conditions necessary for the consistency of maximum likelihood estimation
under the stochastic blockmodel. This work was extended by  \citet{rqf12} with a regularized estimator to high dimensional
settings where the number of communities grows roughly as fast as
the number of nodes.
\citet{cdp11} derived consistency and \citet{bccz13}
derived asymptotic normality of the maximum likelihood estimators and
their variational approximations in the mixture model settings.

In this paper our primary focus is on the problem of detecting an underlying community structure
in multi-layer networks.
% or multi-relational networks,
%where multiple types of edges represent the relations and the set of vertices/nodes represents the
%entities.
We assume that such networks have an implicit community structure and different observed layers manifest that underlying structure with varying amount of information and noise. As an example of a network where such an assumption is reasonable, we analyze a twitter network of British Members of Parliament (see Figure \ref{twitter_example}) where the underlying communities are based on their party memberships and the three observed layers, ``mentions", ``follows" and ``re-tweets" manifest that structure in varying proportions. In such cases the multi-layer graph is a more accurate representation of the underlying similarity of the objects and each layer can provide only ``partial" information about the data \citep{ma11}.  The goal in such cases would be to correctly identify the underlying set of communities combining information from all three layers.

\begin{figure}[htb]
\begin{subfigure}{0.3\textwidth}
\centering{}
\includegraphics[width=.95\linewidth]{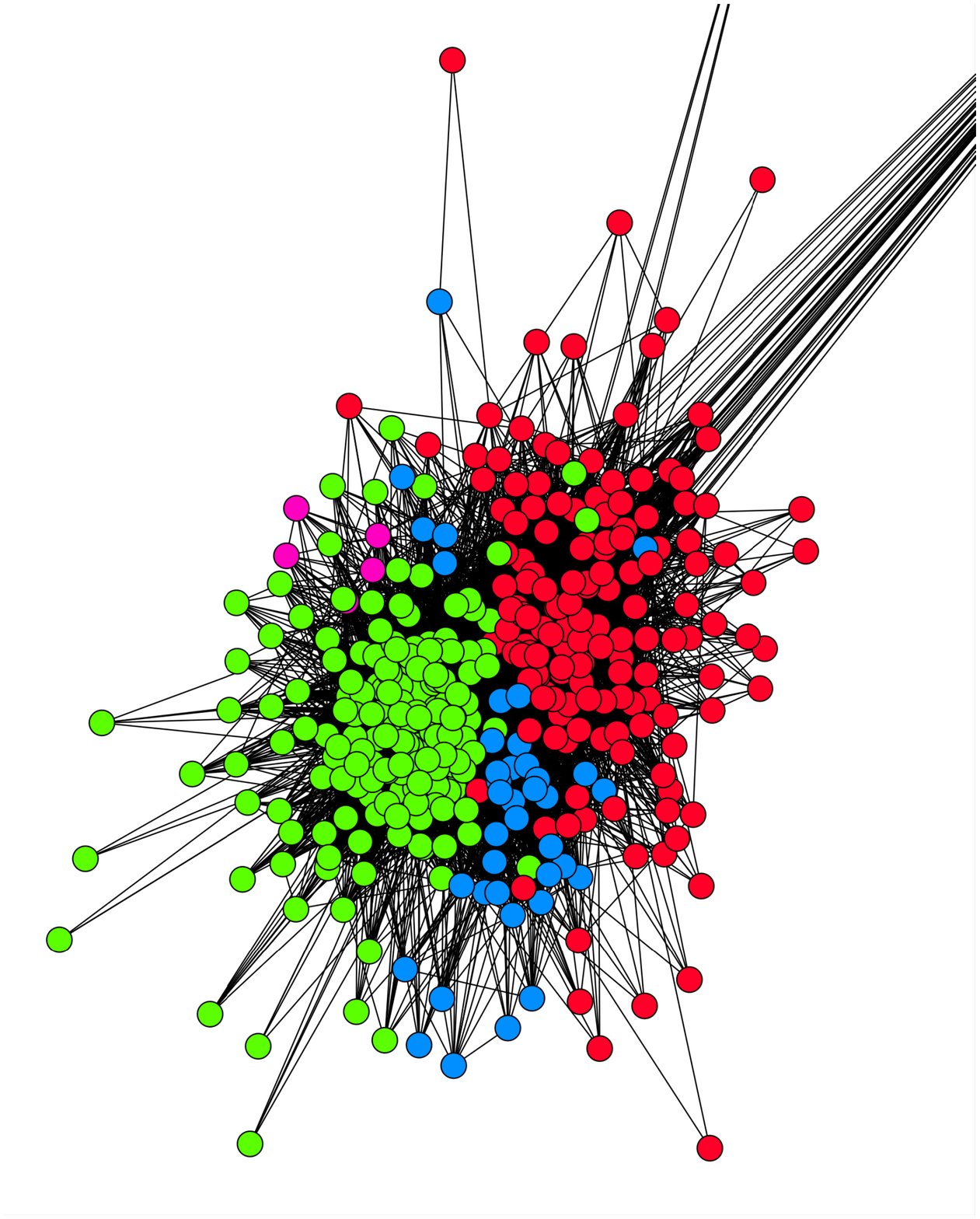}
\caption{Mention}
\end{subfigure}%
\begin{subfigure}{0.3\textwidth}
\centering{}
\includegraphics[width=.95\linewidth]{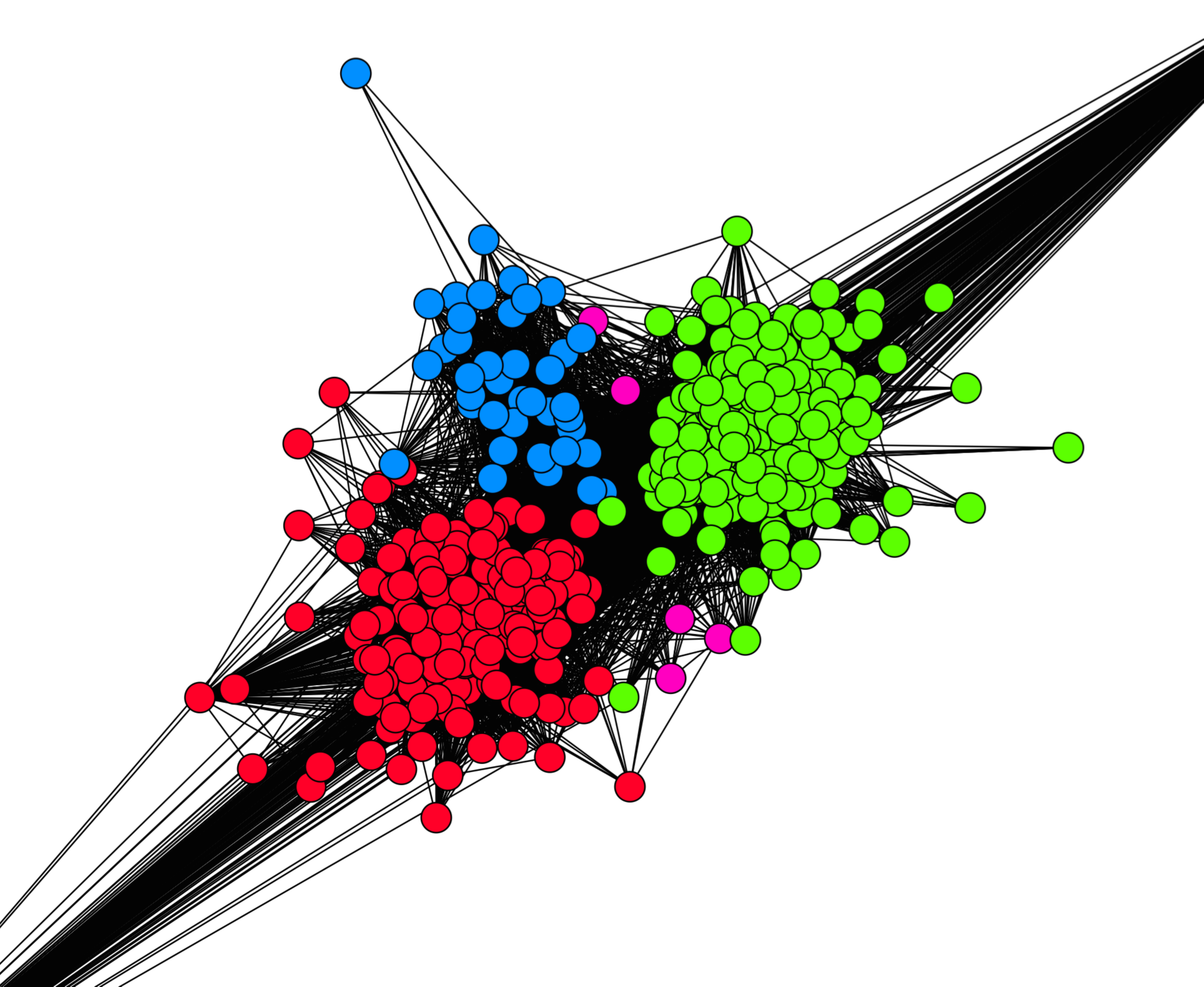}
\caption{Follows}
\end{subfigure}%
\begin{subfigure}{0.3\textwidth}
\centering{}
\includegraphics[width=.95\linewidth]{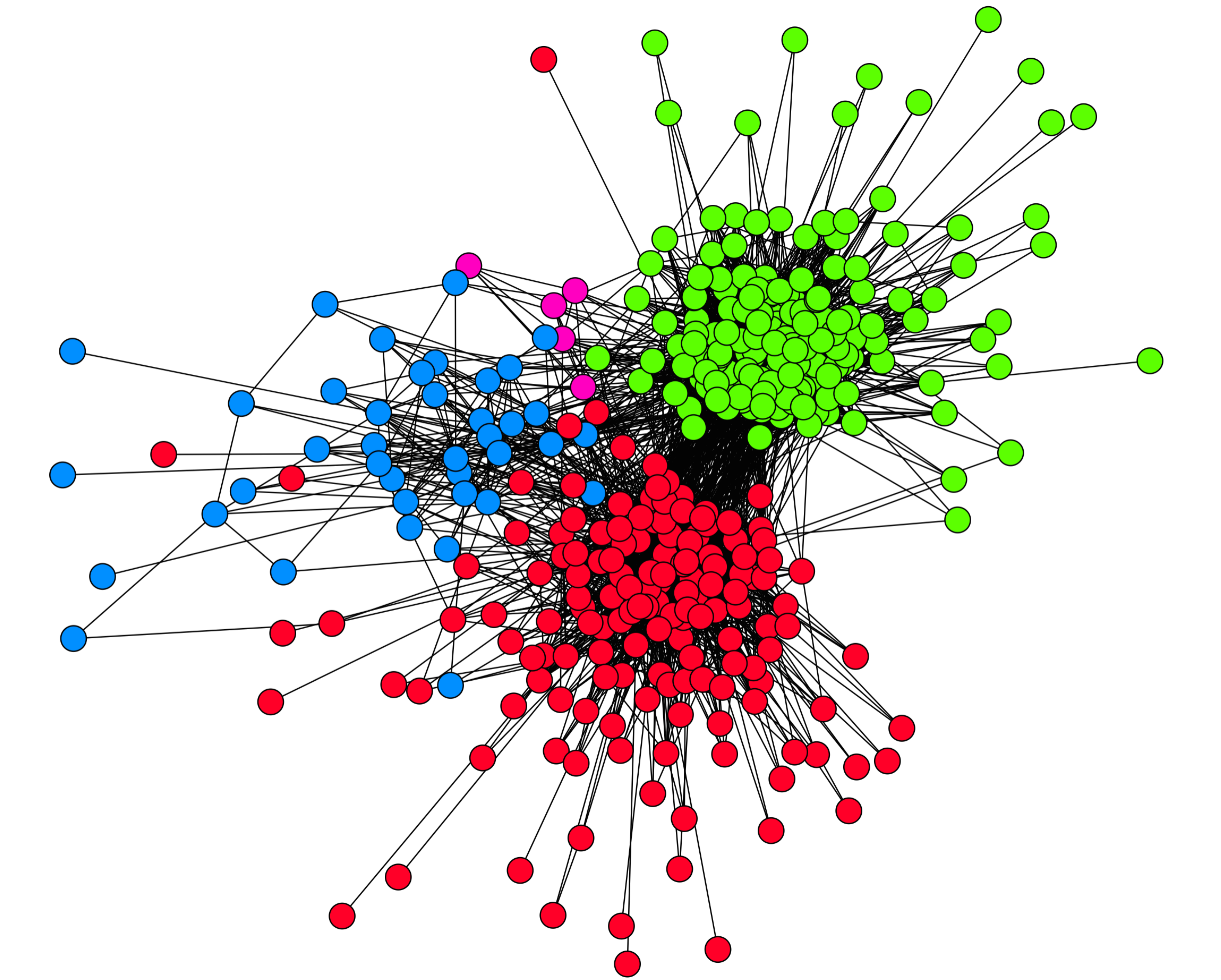}
\caption{Re-tweets}
\end{subfigure}%
\caption{A 3-layer twitter network of British MPs. The nodes are colored according to an underlying community structure: the party memberships. }
\label{twitter_example}
\end{figure}

Earlier approaches towards multi-relational data or multi-layer graph clustering suffer from the deficiency that
they either cluster each graph independently and combine the results, or aggregate the graphs and cluster the aggregated graph. These approaches
fail to take into account the dependency among the different layers,
in particular the correlation among different types of edges that
share the same pair of nodes. Moreover, the multiple network layers
can have different characteristics in terms of sparsity and noise.
Some layers may be dense but may carry little worthwhile information,
whereas some layers may be extremely sparse but may carry valuable
information. The aggregation process of graphs could lose the intrinsic
heterogeneity of the network layers. Here we attempt to address
the problem of how to efficiently cluster the nodes or entities in
a network taking into account all types of layers or relations
among them. Several approaches have been recently proposed in the literature for this purpose. Among them are approaches based on collective or joint matrix factorization \citep{ntk11, wzs09, ma11}, non-parametric Bayesian models and latent factor models \citep{jlbo12}, extensions of spectral clustering \citep{dfvn12} and modularity \citep{mrmpo10} to multi-layer graphs. However there is a lack of statistical analysis of the properties of those methods.

For community detection in multi-layer networks, we consider a natural extension of the standard stochastic blockmodel to multi-layer settings that we will call ``multi-layer
stochastic blockmodel'' (MLSBM). This model, also considered in \citet{hxa14} as ``multi-graph SBM", is in the spirit of multi-relational
models described in \citet{hll83}, \citet{tsk01} and \citet{ktgyu06}. \citet{hxa14} proved the consistency of the maximum likelihood estimates (MLEs) in this model when the number of relations grows. They keep the number of nodes (and hence the number of communities) fixed. However, as we will see later in both the asymptotic analysis and simulation studies that MLE in this model does not perform very well when either the number of communities grows fast or the network layers are sparse on average. Hence, we propose a restricted version
of this model through restrictions on the parameter space which is capable
of handling networks with a large number of communities. We call this
model ``restricted multi-layer stochastic blockmodel'' (RMLSBM). We derive conditions on the growth
of the number of communities and the average edge density of the networks
under which the MLE of the class assignment
vector is consistent (in the sense that the proportion of misclassified
nodes tends to $0$ as the number of nodes, and possibly the number of relations
as well, grows). We further derive the minimax rates of error for community detection in MLSBM and obtain thresholds for consistent community detection. To compute the unknown class assignments
and block model parameters simultaneously, we follow \citet{dpr08} and propose
a variational estimation strategy.

The rest of the paper is organized as follows. Section 2 extends the
stochastic blockmodel to multi-layer settings and defines the two
models, MLSBM and RMLSBM. Section 3 settles the consistency of the community assignments
through maximum likelihood estimation in the two models when the true data generating model is MLSBM. Section 4 describes a few baseline procedures and Section 5 compares the multi-layer models with the baseline models in terms of minimax error rate and sharp threshold results.
Section 6 describes two estimation strategies for the MLEs in the two models.
Section 7 describes the results of a simulation study
to validate the theoretical results. Section 8 presents the application of the methods to the Twitter UK politics data set.  Section 9 gives concluding remarks.

\section{Extension of blockmodels to multi-layer settings}

We consider an undirected multi-layer graph $G=\{V,E\}$, where the
vertex set $V$ consists of $N$ vertices and the edge set $E$ consists
of edges of $M$ different types representing different relations.
We can view the multi-graph as a graph with vector valued edge information,
i.e., the adjacency matrix $A$ consists of elements $A_{ij}$, who
are themselves $M$ dimensional vectors: $A_{ij}=\{A_{ij}^{(1)},A_{ij}^{(2)},$ $\ldots,A_{ij}^{(M)}\}$.
An alternative way to approach the problem is to view the multi-graph
as a collection of $M$, $N\times N$ adjacency matrices $\{A^{(1)},A^{(2)},\ldots,A^{(M)}\}$,
each corresponding to one particular type of relation. The rest of the
set up is similar to the regular stochastic block model (SBM) for one-layer case with $K$ blocks \citep{ns01}. We assume
the number of communities $K$ is known. Let $z=\{z_{1},z_{2},\ldots,z_{N}\}$
be the community indicator vector for the $N$ nodes, such that each $z_{i}$ takes exactly one value from the set $\{1,\ldots, K\}$ and $z_{i}=q$
if and only if node $i$ belongs to community $q$. Conditional
on the community indicator vector $z$, the edges
are formed independently as Bernoulli random variables with probabilities
depending only on the community assignments and the type of edges.  In what follows we describe the two
extensions of the standard SBM to multi-layer settings.
%The two models are formally described below.

Except for the estimation algorithm, the model is always represented
as a conditional block model and $z$ is assumed to be a fixed unknown
parameter of the model and needs to be estimated from data.
Conditioned on the community assignments of the nodes $z_{i}$ and
$z_{j}$, the edges are formed independently following Bernoulli distribution
\[A_{ij}^{(m)}|(z_{i}=q,z_{j}=l)\sim Bernoulli(P_{ql}^{(m)}). \]

The first model assigns a separate probability for the $m$th type
of edge between nodes belonging to the $q$th and the $l$th community
independent of all other edges. We call this model the ``multi-layer
stochastic blockmodel'' (MLSBM). The probability of an $m$th type
of edge between nodes $i$ and $j$ belonging to communities $q$
and $l$ respectively can be written as
\begin{equation*}
P_{ij}^{(m)}=\pi_{z_iz_j}^{(m)}=\pi_{ql}^{(m)}, \quad i,j\in\{1,\ldots,N\},\; m\in\{1,\ldots,M\}, \; q,l\in\{1,\ldots,K\}.
\end{equation*}
The set of parameters for the model, $\pi=\{{\pi}_{ql}^{(m)};\ q\leq l,\ q,l\in\{1,\dots,K\},\ m\in\{1,\ldots,M\}\}$ has $K(K+1)M/2$
elements. This model is ``saturated'' in the sense that
we have a different parameter for each of the different types of edges
between nodes belonging to different communities. Denote the range of this parameter set or array as $\Pi=\{\pi \in [0,1]^{K(K+1)M/2} \}$.

In our asymptotic settings, where both $N$ and $M$ grow and $K$ grows with $N$, the number of parameters to be estimated in the MLSBM grows as $K^{2}M$ and quickly becomes large. Hence the MLE performs poorly especially when the individual network layers are sparse. This problem does not arise in the asymptotic settings of \citet{hxa14} where only $M$ grows and $N,K$ remain fixed. However, it has been empirically shown that in most real world networks the average cluster size does not grow with the size of the network \citep{leskovec08,rqf12,thesis15} and consequently, $K$ grows with $N$. Hence in our asymptotic settings where $N$ grows, keeping $K$ fixed would be rather unrealistic. This motivates us to propose the second related model whose number of parameters grows much slowly compared to MLSBM.

%When $K$ grows with $N$ and/or $M$, the number of parameters to be estimated in this model grows linearly both with $N$ and $M$ and quickly becomes large. Hence the MLE performs poorly especially when the individual network layers are spare. This problem does not arise in the asymptotic settings where only $M$ grows and $K$ remains fixed. However, it has been empirically shown that in most real world networks the average cluster size does not grow with the size of the network and consequently, $K$ grows with $N$. Hence this settings is rather unrealistic. This motivates us to propose the second related model whose number of parameters grows much slowly compared to MLSBM.

%We propose a second related model here that
The second model assumes the probability of
the $m$th type of edge appearing between nodes $i$ and $j$ is
governed by two factors: the first one being the community assignment
of the two nodes and the second one being the type of edge. Hence
the model has two sets of parameters: a $K\times K$ parameter
matrix $\pi_{K\times K}$ corresponding to the community structure, and an $M\times1$ vector
$\beta_{M\times1}$ which contains the parameters for different types
of edges. We call this model the restricted multi-layer
stochastic blockmodel (RMLSBM).

%The justification for
Notice that in the second model,
%comes from the fact that,
if the edges were all of the same type, we would just
have $\beta_{m}=\beta$ for all $m\in\{1,\ldots,M\}$ and then we will
recover the standard stochastic blockmodel, with probabilities of
edges determined solely by the community assignments. On the other
hand, if we did not have a community structure, but $M$ types of
edges, then $\pi_{ql}$ would be identical for all communities $q,l$
and the probability of an edge between nodes $i$ and $j$
will solely be determined by the type of edge.
%Note that
This model
can retrieve information from sparse but highly informative
edge types as the sparsity of the network layers will be captured in
the $\beta_{m}$ parameters. Hence, although we assume the edges to
be conditionally independent, this model induces two types of correlations
unconditionally --- among the edges of the same type and among the edges
that share nodes of the same community.

The probability $P_{ij}^{(m)}$ in RMLSBM , which denotes the probability of an $m$th type
of edge between nodes $i$ and $j$ belonging to communities $q$
and $l$ respectively,
can be modeled in the following way with the logit link function
\begin{equation*}
\text{logit}(P_{ij}^{(m)})
%=\pi_{z_iz_j}+\beta_{m}
=\pi_{ql}+\beta_{m}, \: \:  i,j\in\{1,\ldots,N\}, \: m\in\{1,\ldots,M\}, \: q,l\in\{1,\ldots,K\}.
\end{equation*}
This model has $K(K + 1)/2 + M$ parameters for an undirected graph. Hence, when both $K$  and $M$ grow, the growth rate in the number of parameters for this model is the same as the maximum of the growth rates in $K^{2}$ and $M$. In comparison, the number of parameters in MLSBM would grow as $K^{2}M$. This makes the maximum likelihood estimator in RMLSBM a regularized estimator.

%This model has $K(K+1)/2+M$ parameters for an undirected graph. Hence, when both $K^{2}$ and $M$ grows, the growth rate in the number of parameters for this model is same as the maximum of the growth rates. In comparison, the number of parameters in MLSBM would grow linearly with both $K^{2}$ and $M$. This makes the maximum likelihood estimator in RMLSBM a regularized estimator.

For the RMLSBM to be identifiable, we require the parameters $\beta_{m}$ to satisfy the condition $\sum_{m} \beta_{m}=0$. Hence we have one less free parameter. Denote the set of parameters for RMLSBM as $\pi^{R}=\{({\pi}_{ql},\beta_{m}):\ q\leq l,\ q,l\in\{1,\dots,K\},\ m\in\{1,\ldots,M\}\}$ and its range as $\Pi^{R}=\{\pi^{R} \in \mathcal{R}^{K(K+1)/2+M},\ \sum_{m} \beta_{m}=0 \}$.
To prove the consistency of maximum likelihood estimation under MLSBM, we assume
$\pi_{ql}, \beta_{m} \in (-C \log (MN^{2}),$ $\ C \log (MN^2))$ for some constant $C>0$. This condition ensures that $\pi_{ql}$ and $ \beta_{m}$ are bounded away from $\pm \infty$.

\section{Consistency}
\label{sec:consistency}
In this section, we discuss the consistency of maximum likelihood
estimation of the proposed models under three asymptotic regimes
with varying conditions imposed on the growth of the number of communities ($K$) and the expected total number of edges of the multi-layer graph ($L$). We first define a one to one transformation of the parameters of RMLSBM as
\begin{equation}
\phi_{ql}^{(m)}=\mbox{logit}^{-1}(\pi_{ql}+\beta_{m})=\frac{\exp(\pi_{ql}+\beta_{m})}{1+\exp(\pi_{ql}+\beta_{m})}.\label{eq:phi}
\end{equation}
Now we assume that the data are generated from the more general model MLSBM and view RMLSBM as a MLSBM with the following restrictions on the parameters:
\begin{align}
\Phi=\{\phi \in[0,1]^{K(K+1)M/2}&:\ \phi_{ql}^{(m)}=\mbox{logit}^{-1}(\pi_{ql}+\beta_{m}), \label{piR} \\
 & \pi_{ql},\beta_{m}\in (-C \log (MN^{2}),\ C \log (MN^2)) \}. \nonumber
\end{align}
This way the MLE in RMLSBM can be thought of as a restricted MLE (RMLE) of MLSBM.

Our aim is to investigate the consistency of both the MLE and the RMLE under three asymptotic regimes where we let either the number of nodes ($N$) or the number of types of edges ($M$) or both to grow. This setup is quite appropriate for modern day multi-layer networks,
where data collection increases both in terms of new entities as well
as new features or layers getting added to the database.
Consequently methods are being sought which would be consistent in such situations. Some consistency results for the MLE were obtained in \citet{hxa14} under the settings when $M$ grows, but $N$ and consequently $K$ remain fixed. Here we prove consistency results for the MLE in the more general asymptotic setting where $N$ can also grow (and $K$ grows with $N$). We then compare the MLE with the regularized estimator in terms of the asymptotic conditions required for consistency.
%Consequently methods are being sought which would be consistent in such situations. Although some results for the MLE were obtained in \citep{hxa14} under the settings when $M$ grows, but $N$ and consequently $K$ remains fixed, we still include it in our asymptotic analysis since we need results for the growing $N$ and $K$ case for asymptotic comparison with regularized estimator.
The different asymptotic setups we consider under the three regimes of growth in $N$ and $M$ are described below.
\begin{enumerate}
\item As both $M$ and $N$ grow, let $K=O(N^{1/2})$ and $L=\omega(MN(\log N)^{3+\delta})$ for some $\delta > 0$ for the MLE, while $K=O((MN)^{1/2-\epsilon})$ and $L=\omega(MN (\log N)^{3+\delta})$ with $\epsilon,\delta >0$ for the RMLE. For the RMLE, we further require that $M=O(N)$ so that $K$ does not exceed $N$.

\item As $N$ grows, $M$ either is fixed
or grows slower than $N$, i.e., either $M$ is $O(1)$, or $M\rightarrow\infty$
and $M=O(N)$. In this regime, let $K=O(N^{1/2})$,  $L=\omega(N(\log N)^{3+\delta})$ for some
$\delta > 0$  for the RMLE.

\item As both $N\rightarrow\infty$ and $M\rightarrow\infty$ with $M$ growing faster than $N$, i.e., $M=\omega(N)$, for RMLE we consider two related setups: (a) $K=O(\frac{N}{\log M \log N})$, $L=\omega(MN (\log N)^{1+\delta})$ for some $\delta>0$; and (b) $K=O(N^{1/2})$, $L$ is either $ \omega(M (\log M)^{2+\delta} (\log N)^{1+\delta}) $ for some $\delta>0$ if $(\log M)^{2+\delta} =O(N)$,
    %$\log M =O(N)$,
    or $\omega(MN(\log N)^{1+\delta})$ for some $\delta>0$ otherwise. In setting (a), we further require $\log M$ to grow slower than $N$ for the growth of $K$ to be meaningful. Also, in that setup if $\log M$ grows at the same rate as $(\log N)^{\beta}$ for some $\beta>0$, the number of communities grows almost as fast as the number of nodes except for the $\log $ terms and is ``highest dimensional'' in the sense of \citet{rqf12}.
\end{enumerate}

Note that the first regime assumes no relation between the growth rates of $N$ and $M$, while the next two regimes assume certain relations between the two growth rates.
%restricting their applicability.
So the last two regimes can be thought of as special cases of the first one in terms of the growth rates of $N$ and $M$. Naturally we expect some relaxation in the required growth conditions on $K$ and $L$ in the last two regimes. The asymptotic setups described above reflect this relaxation for the RMLE. However no such relaxation is possible for the MLE. Hence we will prove that MLE in MLSBM is consistent under the first asymptotic regime, whereas MLE in RMLSBM (i.e., the RMLE of MLSBM under the restrictions defined by Equation (\ref{piR}) is consistent under all three asymptotic regimes. The MLSBM, despite being intuitively the simplest extension, does not perform as well as the RMLSBM for community detection in multi-relational networks if the networks are sparse at an average or contain a large
number of communities.

\subsection{Preliminaries}
Since in this paper our primary interest is in modeling multi-layer
networks where layers are sparse on an average, we require the true MLSBM model probabilities
$\pi_{ql}^{(m)}$ to satisfy certain sparsity conditions. As \citet{zlz12} pointed out, if the block model probabilities remain
fixed as $N$ increases, then the network will be unrealistically
dense. In this connection it is worth noting that \citet{sn97} let the probabilities remain fixed and as a result the networks
considered there have linearly increasing average degree, while both
\citet{bc09} and \citet{cwa12} considered networks with
poly-logarithmically increasing average degree and hence gradually
decaying probabilities. Here to keep the network sparse, we scale down the
block model probabilities accordingly as $N$ increases.

We introduce a new notation
%(in a slight abuse of notation),
$L'$ to denote the quantity inside the asymptotic notation $\omega $ in the growth rate of $L$ under different asymptotic setups. As an example, consider the case when $L=\omega (MN (\log N)^{3+\delta})$, %$L'$ as defined above will then be
then $L'=MN (\log N)^{3+\delta}$. Hence $L'$ can be viewed as the minimum rate at which $L$ is required to grow under a particular asymptotic setup.
The blockmodel parameters are restricted to have an upper bound that
decreases with increasing $N$ except for a small finite set indexed
by the triplet $Q=\{q,l,m\}$ such that the expected number of edges in the
set $|E_{Q}|=o\left(\frac{L'}{\log(MN^{2})}\right)$. For the set $Q$ we can have $\frac{1}{MN^{2}}\leq \pi_{ql}^{(m)}\leq1-\frac{1}{MN^{2}}$.
For all $\{q,l,m\}\notin Q$, the parameters are restricted in the
following way
\begin{equation}
\pi_{ql}^{(m)}\in \left(\frac{1}{MN^{2}},\ C\frac{L'}{MN^{2} (\log M \log N)^{2+\delta}}\right), \label{eq:pirestrict}
\end{equation}
for some $\delta>0$ and some constant $C$,
%with $C$ being a constant
so that the upper bound is determined by the expected density of the network. The exact upper bound is determined by $L'$ and consequently, by the growth rate of $L$ and varies under the different asymptotic assumptions.

For any arbitrary partition $z$ of the entities in the graph, the
log likelihood of the set of $M$ adjacency matrices $A=\{A^{(1)},\ldots,A^{(M)}\}$
under the MLSBM with parameters
$\pi=\{\pi_{ql}^{(m)}\}$ is
\begin{equation}
l(A;z,\pi)=\sum_{m=1}^{M}\sum_{i<j}\{A_{ij}^{(m)}\mbox{log \ensuremath{\pi_{z_{i}z_{j}}^{(m)}}}+(1-A_{ij}^{(m)})\mbox{log \ensuremath{(1-\pi_{z_{i}z_{j}}^{(m)}})}\}.\label{eq:MLSBMllk}
\end{equation}
Note that for an undirected graph with no self-loops, both $A^{(m)}$
and $\pi^{(m)},$ $m=1,\ldots,M$, are symmetric matrices in
$\{0,1\}^{N\times N}$ and $[0,1]^{K\times K}$ respectively. The
Bernoulli parameters $\pi_{z_{i}z_{j}}^{(m)}$ depend both on the
class assignment $z$ and the type of relation $m$. For a fixed class
assignment $z$, let $N_{q}$ denote the number of nodes assigned
to class $q$, and $n_{ql}$ denote the maximum number of possible
edges between classes $q$ and $l$. So we have $n_{ql}=N_{q}N_{l}$ and
$n_{qq}={N_q \choose 2}$.
%$n_{qq}=\begin{pmatrix}N_{q}\\
%2
%\end{pmatrix}$.
For an arbitrary partition $z$, the
%maximum likelihood estimate
MLE of $\pi_{(z)}$ is
\begin{equation}
\hat{\pi}_{(z)ql}^{(m)}=\frac{1}{n_{ql}}\sum_{i<j}A_{ij}^{(m)}1\{z_{i}=q,z_{j}=l\}, \ m=1,\ldots,M,\ q,l=1,\ldots,K, \label{eq:pihat}
\end{equation}
where $1\{\cdot\}$ is the indicator function.
Note that for a fixed partition $z$, the denominator $n_{ql}$ in the MLE
 $\hat{\pi}_{(z)ql}^{(m)}$ is the same for all edge types $m$.

Now we define the expectation of $\hat{\pi}_{(z)}$ as $\bar{\pi}_{(z)}$
and that of $l(A;z,\pi)$ as $\bar{l}_{P}(z,\pi)$ under the independent
Bernoulli($P_{ij}^{(m)})$ model. Then we have
\begin{equation}
\bar{\pi}_{(z)ql}^{(m)}=\frac{1}{n_{ql}}\sum_{i<j}P_{ij}^{(m)}1\{z_{i}=q,z_{j}=l\},\  m=1,\ldots,M,\ q,l=1,\ldots,K, \label{eq:pibar}
\end{equation}
\begin{equation}
\bar{l}_{P}(z,\pi)=\sum_{m=1}^{M}\sum_{i<j}\{P_{ij}^{(m)}\mbox{log \ensuremath{\pi_{z_{i}z_{j}}^{(m)}}}+(1-P_{ij}^{(m)})\mbox{log \ensuremath{(1-\pi_{z_{i}z_{j}}^{(m)}})}\}. \label{eq:expllk}
\end{equation}
Clearly for a given $z$, $\hat{\pi}_{(z)}$ and $\bar{\pi}_{(z)}$ are
the maximizers of the functions $l(A;z,\pi)$ and $\bar{l}_{P}(z,\pi)$
respectively, and we let $l(A;z)$ and $\bar{l}_{P}(z)$ denote the
corresponding maximum values.

We extend Lemma 1 of \citet{cwa12} to multi-layer settings as follows:
%\begin{eqnarray}
%l(A;z)-\bar{l}_{P}(z) &=& \sum_{m} \sum_{i<j} \Bigg\{A_{ij}^{(m)} \log \left(\frac{\hat{\pi}_{z_i z_j}^{(m)}}{\bar{\pi}_{z_i z_j}^{(m)}}\right)+(1-A_{ij}^{(m)})\log \left(\frac{1-\hat{\pi}_{z_i z_j}^{(m)}}{1-\bar{\pi}_{z_i z_j}^{(m)}}\right)\Bigg\} \nonumber \\
%&& + X-E(X) \nonumber \\
%&=&\sum_{m} \sum_{q\leq l} n_{ql} D\left(\hat{\pi}_{(z)ql}^{(m)}\ ||\  \bar{\pi}_{(z)ql}^{(m)}\right)+X-E(X),\label{eq:llkdecomp}
%\end{eqnarray}
\begin{align}
l(A;z)-\bar{l}_{P}(z)  =& \sum_{m} \sum_{i<j} \Bigg\{A_{ij}^{(m)} \log \left(\frac{\hat{\pi}_{z_i z_j}^{(m)}}{\bar{\pi}_{z_i z_j}^{(m)}}\right)+(1-A_{ij}^{(m)})\log \left(\frac{1-\hat{\pi}_{z_i z_j}^{(m)}}{1-\bar{\pi}_{z_i z_j}^{(m)}}\right)\Bigg\}  + X-E(X) \nonumber \\
=& \sum_{m} \sum_{q\leq l} n_{ql} D(\hat{\pi}_{(z)ql}^{(m)}||\bar{\pi}_{(z)ql}^{(m)})+X-E(X),\label{eq:llkdecomp}
\end{align}
where
\begin{equation}
X=\sum_{m=1}^{M}\sum_{i<j}A_{ij}^{(m)}\log \left (\frac{\bar{\pi}_{z_{i}z_{j}}^{(m)}}{1-\bar{\pi}_{z_{i}z_{j}}^{(m)}}\right ).\label{eq:defx}
\end{equation}
Here $D(a||b)$ is the Kullback-Liebler divergence between two Bernoulli random variables with parameters $a$ and $b$ respectively. This equation decomposes the difference between the maximized likelihood and its expected value in terms of $\hat{\pi}_{(z)} $ and $\bar{\pi}_{(z)}$ for a given class assignment vector $z$.

Next we turn our attention to RMLSBM. As mentioned before, we consider RMLSBM as a restricted
version of MLSBM, and the MLE of RMLSBM can be viewed as a RMLE of
MLSBM under the restrictions.
Given a class assignment $z$, the RMLE $\hat{\pi}{}_{z_{i}z_{j}}^{(m)R}=\{\hat{\pi}_{(z)ql},\ \hat{\beta}_{(z)m}\}$
is the maximizer of $l^{R}(A;z,\pi^{R})$, the multi-layer block model
log likelihood within the restricted parameter space. Substituting
the estimated parameters in the likelihood function gives $l^{R}(A;z)$, the maximum of the likelihood function within the restricted
parameter space. However, no closed form solution exists for the RMLE.
%and numerical optimization methods are required to estimate the parameters.
Instead we have the following $M+{K(K+1)}/{2}$ estimating equations:
\begin{equation}
\frac{\partial}{\partial\beta_{m}}:=\sum_{i< j}\left(A_{ij}^{(m)}-\frac{\exp(\hat{\pi}_{z_{i}z_{j}}+\hat{\beta}_{m})}{1+\exp(\hat{\pi}_{z_{i}z_{j}}+\hat{\beta}_{m})}\right), \label{pibeta}
\end{equation}
\begin{equation}
\frac{\partial}{\partial\pi_{z_{i}z_{j}}}:=\sum_{i< j}\sum_{m}\left(A_{ij}^{(m)}-\frac{\exp(\hat{\pi}_{z_{i}z_{j}}+\hat{\beta}_{m})}{1+\exp(\hat{\pi}_{z_{i}z_{j}}+\hat{\beta}_{m})}\right).  \label{pibeta1}
\end{equation}
One of the equations is redundant since if we add the equations in (\ref{pibeta}), the resulting equation is identical to the sum of the equations in (\ref{pibeta1}).

Now we use the transformation defined by $\phi$ in Equation (\ref{eq:phi}). The likelihood with respect to the new parameters
can be represented as
\begin{equation}
l^{R}(A;z,\phi)=\sum_{m=1}^{M}\sum_{i<j}\{A_{ij}^{(m)}\mbox{log \ensuremath{\phi_{z_{i}z_{j}}^{(m)}}}+(1-A_{ij}^{(m)})\mbox{log \ensuremath{(1-\phi_{z_{i}z_{j}}^{(m)}})}\},\label{eq:resllk}
\end{equation}
and the estimating equations in (\ref{pibeta}) and (\ref{pibeta1}) can be written as
\begin{align}
\frac{1}{N(N+1)/2}\sum_{q\leq l} n_{ql} \hat{\phi}_{(z)ql}^{(m)} &=\frac{1}{N(N+1)/2}\sum_{q\leq l}\sum_{i<j}A_{ij}^{(m)}1\{z_{i}=q,z_{j}=l\}  \nonumber \\
 &=\frac{1}{N(N+1)/2}\sum_{i<j}A_{ij}^{(m)}, \quad m=1,\ldots,M, \label{eq:phihat}
\end{align}
\begin{equation}
\frac{1}{M}\sum_{m}\hat{\phi}_{(z)ql}^{(m)}=\frac{1}{Mn_{ql}}\sum_{m}\sum_{i<j}A_{ij}^{(m)}1\{z_{i}=q,z_{j}=l\},\ q\leq l \in \{1,\ldots ,K\}.\label{eq:phihat-1}
\end{equation}
Together the right hand sides of these equations are the complete and sufficient statistics for the model. Hence we have ${K(K+1)}/{2}+M-1$ independent equations which will together determine the MLE of ${K(K+1)}/{2}+M-1$ free parameters in the set $\pi^{R}_{(z)}$. Here it is understood that the estimation procedure ensures that the finiteness condition of $\pi_{ql}$ and $\beta_{m}$ are respected possibly by restricting $\pi_{ql},\beta_{m}\in (-C \log (MN^{2}),\ C \log (MN^2))$.
%replacing $\pm \infty$ by $\pm \log(CMN^2)$.
By the functional invariance property of the MLE, $\hat{\phi}_{(z)ql}^{(m)}=\frac{\exp(\hat{\pi}_{ql}+\hat{\beta}_{m})}{1+\exp(\hat{\pi}_{ql}+\hat{\beta}_{m})}$
is the MLE of $\phi_{(z)ql}^{(m)}$. Note that the minimum value any $\hat{\phi}_{(z)ql}^{(m)}$ can take due to the imposed boundedness constraint is $1/MN^2$. This value is sufficiently small so that none of the partial sums in the left hand side of Equations (\ref{eq:phihat}) and (\ref{eq:phihat-1}) exceeds $1$.

As before we define
expectations of $\hat{\phi}_{z}$ as $\bar{\phi}_{z}$ and that of
$l^{R}(A;z,\phi)$ as $\bar{l}_{P}^{R}(z,\phi)$ under the independent
Bernoulli($P_{ij}^{(m)})$ model. Then,
\begin{equation}
\bar{l}_{P}^{R}(z,\phi)=\sum_{m=1}^{M}\sum_{i<j}\{P_{ij}^{(m)} \log (\bar{\phi}_{z_{i}z_{j}}^{(m)})+(1-P_{ij}^{(m)})\log (1-\bar{\phi}_{z_{i}z_{j}}^{(m)})\}.
\end{equation}
For a given class assignment $z$, $\hat{\phi}_{z}$ and $\bar{\phi}_{z}$
are the maximizers of the functions $l^{R}(A;z,\phi)$ and $\bar{l}_{P}^{R}(z,\phi)$
respectively, and we let $l^{R}(A;z)$ and $\bar{l}_{P}^{R}(z)$ denote
the corresponding maximum values. The difference between the maximized values of the observed and expected likelihood can be decomposed in two parts similar to Equation (\ref{eq:llkdecomp}) as follows
\begin{equation}
 l^{R}(A;z)-\bar{l}_{P}^{R}(z) = \sum_{m} \sum_{q\leq l} n_{ql} D\left(\hat{\phi}_{(z)ql}^{(m)}\ ||\ \bar{\phi}_{(z)ql}^{(m)}\right)+X-E(X),\label{eq:resllkdecomp}
\end{equation}
where as before,
\begin{equation}
X=\sum_{m=1}^{M}\sum_{i<j}A_{ij}^{(m)}\log \left (\frac{\bar{\phi}_{z_{i}z_{j}}^{(m)}}{1-\bar{\phi}_{z_{i}z_{j}}^{(m)}}\right ). \label{eq:defxres}
\end{equation}
A proof of this result can be found in the Appendix. Since the maximum of unrestricted likelihood would be at least as large as the maximum of restricted likelihood, we have $l(A;z)\geq l^{R}(A;z)$
and $\bar{l}_{P}(z)\geq\bar{l}_{P}^{R}(z)$ for all $z$.

Now let $\bar{z}$ denote the true partition. Further let $\hat{z}$
and $\hat{z}^{R}$ denote the MLEs of $\bar{z}$
under the two models MLSBM and RMLSBM respectively, i.e.,
%defined as,
\begin{equation}
\hat{z}=\arg\max_{z}l(A,z).
\end{equation}
\begin{equation}
\hat{z}^{R}=\arg\max_{z}l^{R}(A,z).
\end{equation}

\subsection{Main results}
We give several theorems in this section as we develop towards our main result. These theorems provide
insights into the conditions required under the three asymptotic regimes
discussed in the beginning of Section \ref{sec:consistency}, which in turn provide comparison
between the asymptotic behavior of MLEs in the two models MLSBM and RMLSBM. All the proofs are given in the Appendix.

The first three theorems
%whose proofs follow closely Theorem 1 in \citep{cwa12},
bound the difference in the maximized log likelihood
and its expected value for both MLSBM and RMLSBM as defined in Equations (\ref{eq:llkdecomp}) and (\ref{eq:resllkdecomp}).
%Now we are ready to state and prove the theorems.

\begin{thm}
\label{thm:thetabound} Suppose a MLSBM
%multi-layer stochastic block model (MLSBM)
and a
%multi-layer restricted stochastic block model (RMLSBM)
RMLSBM,
both with $K$ classes and $M$ layers, are fitted to the graph with adjacency matrix $\{A_{ij}\}_{i<j}=\{A_{ij}^{(1)},\ldots,A_{ij}^{(M)}\}_{i<j},\ i,j=1,\ldots,N$,
where $A_{ij}^{(m)}$ are independent Bernoulli$(P_{ij}^{(m)})$
trials. For any class assignment $z$, suppose the estimate $\hat{\pi }_{(z)}=\{\hat{\pi}_{(z)ql}^{(m)};\ q,l\in\{1,\ldots,K\},\ m\in\{1,\ldots,M\}\}$
maximizes the multi-layer block model likelihood $l(A;z,\pi)$ and the estimate
$\hat{\pi}^{R}_{(z)}=\{(\hat{\pi}_{(z)ql},\hat{\beta}_{(z)m});\ q\leq l,\ q,l\in\{1,\dots,K\},\ m\in\{1,\ldots,M\}\}$ maximizes the likelihood from the model with the restricted parameter space defined by
$\Pi^{R}$. Let $\hat{\phi}_{(z)}=\{\hat{\phi}_{(z)ql}^{(m)};\ q,l\in\{1,\ldots,K\},\ m\in\{1,\ldots,M\}\}$ be defined from $\hat{\pi}^{R}_{(z)}$ according to Equation (\ref{eq:phi}). Then for any $\epsilon >0$,
\begin{align}
\label{thm1_result1}
&P\left(\underset{z}{\max}\sum_{q\leq l}n_{ql}\sum_{m}D\left(\hat{\pi}_{(z)ql}^{(m)}\ ||\ \bar{\pi}_{(z)ql}^{(m)}\right) \geq \epsilon\right) \\ \nonumber
& \leq \exp\left(N\log K+M(K^{2}+K) \log \left(\frac{N}{K}+1\right)-\epsilon\right),
\end{align}
\begin{align}
\label{thm1_result2}
&P\Bigg(\underset{z}{\max} \Bigg\{ \sum_{m}\frac{N(N+1)}{2}D\left(\frac{\sum_{q\leq l}n_{ql}\hat{\phi}_{(z)ql}^{(m)}}{N(N+1)/2}\ \Big{|}\Big{|}\ \frac{\sum_{q\leq l}n_{ql}\bar{\phi}_{(z)ql}^{(m)}}{N(N+1)/2}\right) \Bigg\} \geq \epsilon\Bigg) \\ \nonumber
 & \leq \exp \left(N \log K+(K^{2}+K)\log \left(\frac{NM^{{1}/{2}}}{K}+1\right)+M\log \left(\frac{N(N+1)}{2}+1\right)-\epsilon \right),
\end{align}
\begin{align}
\label{thm1_result3}
&P\Bigg(\underset{z}{\max} \Bigg\{\sum_{q \leq l}Mn_{ql}D\left(\frac{1}{M}\sum_{m}\hat{\phi}_{ql}^{(m)}\ \Big{|}\Big{|}\ \frac{1}{M}\sum_{m}\bar{\phi}{}_{ql}^{(m)}\right)\Bigg\}\geq \epsilon\Bigg)  \\ \nonumber
 & \leq \exp \left( N\log K+(K^{2}+K)\log\left(\frac{NM^{{1}/{2}}}{K}+1\right)+M\log \left(\frac{N(N+1)}{2}+1\right)-\epsilon \right).
\end{align}
\end{thm}

The first result (\ref{thm1_result1}) provides a bound for the first part of the right hand side of Equation (\ref{eq:llkdecomp}) for MLSBM. The results (\ref{thm1_result2}) and (\ref{thm1_result3}) provide a bound that will be used in Theorem \ref{th:consrMLSBM} to bound the first part of the corresponding likelihood decomposition for RMLSBM in Equation (\ref{eq:resllkdecomp}). In the proofs of the next two theorems, we first bound the second part of Equations (\ref{eq:llkdecomp}) and (\ref{eq:resllkdecomp}),
and then combine the results to provide a bound for the difference between the log likelihood and its expected value under any arbitrary partition $z$ for MLSBM and RMLSBM respectively.

\begin{thm}
%Assume the settings of Theorem \ref{thm:thetabound}, where
Suppose a MLSBM with $K$ classes and $M$ layers is fitted to the graph whose edges $A_{ij}^{(m)}$ are independent Bernoulli($P_{ij}^{(m)}$) trials.   If we further assume that (i) $\frac{1}{MN^{2}}\leq P_{ij}^{(m)}\leq 1-\frac{1}{MN^{2}}$ for
all $i<j$, (ii) $K=O(N^{1/2})$, and (iii) the total expected number
of edges of the entire multi-layer graph $L=\underset{m}{\sum}\underset{i<j}{\sum}E(A_{ij}^{(m)})$ is $\omega(MN(\log N)^{3+\delta})$ for some $\delta>0$ as both $M$ and $N$ grow, then
\[
\underset{z}{\max}|l(A;z)-\bar{l}_{P}(z)|=o_{P}(L).
\]
\label{th:consMLSBM}
\end{thm}

The result of this theorem
holds under the given conditions irrespective of the relationship between the growth rates of $M$ and $N$. We state the result under the first asymptotic regime mentioned at the beginning of Section \ref{sec:consistency}
%only one of the three asymptotic regimes
%conditions
%mentioned at the beginning of the section,
since we do not get any relaxation in the assumption regarding the total expected number of edges if we assume certain relations between the growth rates of $M$ and $N$.
%On the other hand, if we assume $K=O((MN)^{1/2-\epsilon})$ for some $\epsilon>0$, then we require the total expected number of edges to be $\omega(M^{2}N(\log N)^{3+\delta})$ which is unrealistically dense.

The next theorem states that the restricted likelihood in RMLSBM is also asymptotically well behaved under five independent sets of conditions
%as discussed at the beginning of the section.
corresponding to the three asymptotic regimes discussed at the beginning of Section \ref{sec:consistency}. The first two sets of conditions correspond to regime 1, the third set of conditions corresponds to regime 2, and the last two sets of conditions correspond to regime 3.

\begin{thm}
Assume that a RMLSBM with $K$ classes and $M$ layers is fitted to the graph whose edges $A_{ij}^{(m)}$ are independent Bernoulli($P_{ij}^{(m)}$) trials. If we further assume any of the following five sets of conditions with respect to the growth of the properties of the model under different asymptotic settings:

(i) both $M$  and $N$ grow, $K=O(N^{1/2})$, $\frac{1}{MN^{2}}\leq P_{ij}^{(m)}\leq C\frac{\log N}{N (\log M)^{2+\delta}}$ for
all $i<j$, where $C$ is a constant, and the total expected number of edges
of the entire multi-layer graph $L=\omega(MN(\log N)^{3+\delta})$
for some $\delta>0$;

(ii) both $M$  and $N$ grow but $M=O(N)$, $K=O((MN)^{1/2-\epsilon})$ for some $\epsilon>0$,  $\frac{1}{MN^{2}}\leq P_{ij}^{(m)}\leq C\frac{\log N}{N (\log M)^{2+\delta}}$ for
all $i<j$, where $C$ is a constant, and the total expected number of edges
of the entire multi-layer graph $L=\omega(MN (\log N)^{3+\delta})$ for some $\delta>0$;

(iii) $M$ is either a constant or grows
slower than $N$, i.e., $M=o(N)$, $K=O(N^{1/2})$, $\frac{1}{MN^{2}}\leq P_{ij}^{(m)}\leq C\frac{\log N}{MN (\log M)^{2+\delta}}$ for all $i<j$, where $C$ is a constant, and the total expected number of
edges of the entire multi-layer graph $L$ is $\omega(N (\log N)^{3+\delta})$
for some $\delta>0$;

(iv) $M$ grows and $N$ is either a constant or grows slower than $M$, i.e.,
$M=\omega(N)$, $K=O(\frac{N}{\log N \log M})$, $\frac{1}{MN^{2}}\leq P_{ij}^{(m)}\leq C\frac{1}{N \log N (\log M)^{2+\delta}}$ for all $i<j$, where $C$ is a constant, and the total expected number of edges of the entire multi-layer graph $L=\omega(MN(\log N)^{1+\delta})$ for some $\delta>0$;

(v) $M$ grows and $N$ is either a constant or grows slower than $M$, i.e.,
$M=\omega(N)$, $K=O(N^{1/2})$, $\frac{1}{MN^{2}}\leq P_{ij}^{(m)}\leq \min \left( C\frac{1}{N^{2} \log N}, C\frac{1}{N \log N (\log M)^{2+\delta}}\right)$ for all $i<j$, where $C$ is a constant, and the total expected number of edges of the entire multi-layer graph $L$ is larger than the the smaller of $M (\log M)^{2+\delta} (\log N)^{1+\delta}$ and $MN (\log N)^{1+\delta}$ for some $\delta>0$;

\noindent
then,
\[
\underset{z}{\max}|l^{R}(A;z)-\bar{l}_{P}^{R}(z)|=o_{P}(L).
\]
\label{th:consrMLSBM}
\end{thm}

It is clear from Theorem \ref{th:consMLSBM}
and Theorem \ref{th:consrMLSBM} that in RMLSBM, the bound on the likelihood can be established
both for relatively milder conditions on the expected total number
of edges and relatively faster growth conditions on the number of communities.
As we will see in Theorem 5 and the discussion following it, this
enables RMLSBM to be a more attractive model for community detection
either when the number of communities is large or when we have relatively
sparser graphs.

Now we are ready to state our main results which show that when the true data generating process is a $K$-class MLSBM, the fraction of nodes misclustered
by the MLEs
%(for MLSBM)
and the RMLEs
%(for RMLSBM)
converge to zero under
different asymptotic regimes. We define the number of ``misclustered'' nodes $N_{e}(\hat{z})$ as the number of incorrect class
assignments under $\hat{z}$, counted for every node whose true class
under $\bar{z}$ is not in the majority within its estimated class
under $\hat{z}$ \citep{cwa12}.

The previous results (Theorems \ref{thm:thetabound}, \ref{th:consMLSBM}, \ref{th:consrMLSBM}) hold for any $P_{ij}^{(m)}$ whenever they are bounded as described in the theorems. Now we assume further structure on the probabilities, namely a MLSBM. %We assume the true partition to be $\bar{z}$
Denote the true partition as $\bar{z}$,
and under the true partition, let the true block model parameter array be $\bar{\pi}$. Hence, under MLSBM we have
\[
P_{ij}^{(m)}=\bar{\pi}_{\bar{z}_{i} \bar{z}_{j}}^{(m)}.
\]
Consequently, $\bar{l}_{P}(\bar{z},\pi)$ from Equation (\ref{eq:expllk}) is maximized by the true model parameter $\bar{\pi}$, and we have the maximized expected likelihood as
\begin{equation}
\bar{l}_{P}(\bar{z})=\sum_{m=1}^{M}\sum_{q\leq l} n_{ql}\{ {\bar{\pi}_{ql}^{(m)}} \log {\bar{\pi}_{ql}^{(m)}}+(1-{\bar{\pi}_{ql}^{(m)}}) \log (1-{\bar{\pi}_{ql}^{(m)}})\}. \label{eq:truemax}
\end{equation}

On the other hand, the expected restricted likelihood is maximized by the parameter array  $\bar{\pi}^{R}$ under the restricted parameter space of RMLSBM. Note that this is different from the true model parameter array $\bar{\pi}$ due to the restrictions imposed on the parameter space. Using the transformation introduced in Equation (\ref{eq:phi}),
the maximized expected restricted likelihood is
\begin{align}
\bar{l}_{P}^{R}(\bar{z})
&=\sum_{m=1}^{M}\sum_{i<j} \{P_{ij}^{(m)} {\log {\bar{\phi}_{\bar{z}_{i} \bar{z}_{j}}^{(m)}}+(1-P_{ij}^{(m)}) \log (1-{\bar{\phi}_{\bar{z}_{i} \bar{z}_{j}}^{(m)}})}\}\nonumber \\
&=\sum_{m=1}^{M}\sum_{i<j} \{{\bar{\pi}_{\bar{z}_{i} \bar{z}_{j}}^{(m)}} {\log {\bar{\phi}_{\bar{z}_{i} \bar{z}_{j}}^{(m)}}+(1-{\bar{\pi}_{\bar{z}_{i} \bar{z}_{j}}^{(m)}}) \log (1-{\bar{\phi}_{\bar{z}_{i} \bar{z}_{j}}^{(m)}})}\} \nonumber\\
&=\sum_{m=1}^{M}\sum_{q\leq l} n_{ql} \{{\bar{\pi}_{ql}^{(m)}} {\log {\bar{\phi}_{ql}^{(m)}}+(1-{\bar{\pi}_{ql}^{(m)}}) \log (1-{\bar{\phi}_{ql}^{(m)}})}\}. \label{eq:truemaxres}
\end{align}

The next theorem relates the difference between observed and true likelihood with the fraction of misclustered nodes $N_{e}(\hat{z})$ and the expected total number of edges $L$ to establish a bound for the misclustering rate.

\begin{thm}
Suppose the data are generated according to a $K$-class MLSBM with membership vector $\bar{z}$ and parameter array $\bar{\pi}$, the conclusion of Theorem \ref{th:consMLSBM} holds, and the following conditions hold with respect to the model
sequence: for all blockmodel classes $q=1,\ldots,K$, class size $N_{q}$
grows as $s=\underset{q}{\min }\{N_{q}\}=\Omega({N}/{K})$, and over
all distinct class pairs $(q,l)$ and all classes $c\neq\{q,l\}$,
\begin{align}
\underset{q,l}{\min} \, \underset{m}{\min} \, \underset{c}{\max} &\left\{D\left(\bar{\pi}{}_{qc}^{(m)}\ \Big{|}\Big{|}\ \frac{\bar{\pi}{}_{qc}^{(m)}+\bar{\pi}{}_{lc}^{(m)}}{2}\right)+D\left(\bar{\pi}{}_{lc}^{(m)}\ \Big{|}\Big{|}\ \frac{\bar{\pi}{}_{qc}^{(m)}+\bar{\pi}{}_{lc}^{(m)}}{2}\right)\right\} \nonumber \\
&=\Omega\left(\frac{LK}{MN^2}\right), \label{eq:probcond}
\end{align}
then
\begin{equation}
N_{e}(\hat{z})=o_{P}(N).
\end{equation}
\label{thm4}
\end{thm}

Note that condition (\ref{eq:probcond}) is very similar to condition (ii) of Theorem 3 in \citet{cwa12} with the total number of edges for the single layer case being replaced by the average number of edges $L/M$ in each layer for the multi-graph. This ensures that any two rows in any of the layer matrices $\bar{\pi}^{(m)}$ of $\bar{\pi}$ differ in at least one entry by at least a constant times $\frac{LK}{MN^{2}}$. Also, when we take into account the asymptotic conditions required on the growth of $K$ and $L$ for the result of Theorem 2 to hold, i.e., $K=O(N^{1/2})$ and
$L=\omega(MN(\log N)^{3+\delta})$ with $M$ and $N$ both growing,
then we have $\frac{LK}{MN^{2}}=\omega\Big(\frac{(\log N)^{3+\delta}}{N^{1/2}}\Big)$.
%This condition is the same as that in \citep{cwa12}.
As argued in \citet{cwa12}, if $L$ is close to its least possible rate of growth, $\frac{LK}{MN^{2}}$ goes to $0$ for large $N$ and the condition is not too prohibitive.
For example, if $L= MN (\log N)^\beta$ with $\beta > 4$, then $(\log N)^\beta =o(N^{1/2})$, so $\frac{LK}{MN^{2}}$ goes to $0$ and the condition is not overly restrictive.

We state the corresponding conclusion for the restricted likelihood estimation (for RMLSBM) in the next theorem, i.e., the class membership assignment vector estimated through the maximum likelihood estimation in the restricted model RMLSBM is consistent under data generated from the MLSBM.

\begin{thm}
Suppose the data are generated according to a $K$-class MLSBM with membership vector $\bar{z}$ and parameter array $\bar{\pi}$, the conclusion of Lemma 3 holds, and the following conditions hold with respect to the model
sequence: for all blockmodel classes $q=1,\ldots,K$, class size $N_{q}$
grows as $s=\underset{q}{min}\{N_{q}\}=\Omega({N}/{K})$, and over
all distinct class pairs $(q,l)$ and all classes $c\neq\{q,l\}$,
\begin{equation}
\underset{q,l}{\min} \, \underset{m}{\min} \, \underset{c}{\max}\left\{D\left(\bar{\pi}{}_{qc}^{(m)}\ \Big{|}\Big{|}\ \frac{\bar{\pi}{}_{qc}^{(m)}+\bar{\pi}{}_{lc}^{(m)}}{2}\right)+D\left(\bar{\pi}{}_{lc}^{(m)}\ \Big{|}\Big{|}\ \frac{\bar{\pi}{}_{qc}^{(m)}+\bar{\pi}{}_{lc}^{(m)}}{2}\right)\right\}=\Omega(g),\label{eq:probcond1}
\end{equation}
then under any of the five sets of growth conditions in Theorem \ref{th:consrMLSBM}, we have
\begin{equation}
N_{e}(\hat{z}^{R})=o_{P}(h).
\end{equation}
Here $g$ in condition (\ref{eq:probcond1}) and the growth rate $h$ depend on the asymptotic conditions imposed on $K$ and
$L$. The growth rate $h$ can be determined from $g$ by the relationship
$h=\frac{KL}{MNg}$. In particular,
(i) when $K=O(N^{1/2})$, $L=\omega(MN  (\log N)^{3+\delta})$
with $M$ and $N$ both growing arbitrarily,
%( i.e no relation between the growth rates),
then we have $g=\frac{LK}{MN^{2}}=\omega\Big(\frac{(\log N)^{3+\delta}}{N^{1/2}}\Big)$
and $h=N$;
(ii) when $K=O((MN)^{1/2-\epsilon})$, $L=\omega(MN (\log N)^{3+\delta})$
with $M$ and $N$ both growing so that $M=O(N)$, then we have $g=\frac{LK}{MN^{2}}=\omega\Big((\frac{M}{N})^{1/2}\Big)$
and $h=N$;
(iii) when $K=O(N^{1/2})$, $L=\omega(N(\log N)^{3+\delta})$ and
$M=o(N)$, then we have $g=\frac{LK}{N^{2}}=\omega\Big(\frac{(\log N)^{3+\delta}}{N^{1/2}}\Big)$
and $h=N/M$;
(iv) when $K=O(N^{1-\epsilon}/ \log M)$, $L=\omega(MN (\log N)^{1+\delta}$
and $M=\omega(N)$, then we have $g=\frac{LK}{MN^{2}}=\omega\Big(\frac{1}{\log M}\Big)$
and $h=N$;
(v) when $K=O(N^{1/2})$, $L$ is $\omega(MN(\log N)^{1+\delta})$ if $N<(\log M)^{2+\delta}$ or $\omega(M(\log M)^{2+\delta} (\log N)^{1+\delta})$ if $N>(\log M)^{2+\delta}$
and $M=\omega(N)$, then we have $g=\frac{LK}{MN^{2}}=\omega\Big(\frac{(\log N)^{1+\delta}}{N^{1/2}}\Big)$ or $g=\frac{LK}{MN^{2}}=\omega\Big(\frac{(\log M)^{2+\delta} (\log N)^{1+\delta}}{N^{3/2}}\Big)$
and $h=N$.
\label{thm5}
\end{thm}

Note that in Theorem \ref{thm5}, we have used generic notations
$g$ and $h$ to denote functions of the network properties such as
$N$, $K$ and $L$. The functions $g$ and $h$ vary across asymptotic setups. This is so because the regularity condition (\ref{eq:probcond1}) on the difference among
the elements of block model probability matrices
%(\ref{eq:probcond})
should be as less prohibitive as possible.
Note that in our results, we have chosen $g$ in such a way that if $L$ is close to its least possible rate of growth, then $g$ asymptotically decays to 0 under the assumed asymptotic setup.
This ensures that our condition (\ref{eq:probcond1}) is not overly restrictive.
%we have chosen $g$ in such a way that it asymptotically decays to 0 under the assumed asymptotic setup.
It also enables us to understand and contrast
the asymptotic behavior of the RMLE from a unified point of view.

\subsection{Sparse networks}

The results of all previous theorems imply that for sparse multi-layer networks, consistency can be achieved with a large number of relatively sparser graphs as long as they together satisfy the edge density requirement. In the case when $M$ grows slower than $N$, in MLSBM we do not get any relaxation in the required growth condition on the total expected number of edges from all the graph layers combined, and it remains $\omega(MN (\log N)^{3+\delta})$ for $K=O(N^{1/2})$. However in RMLSBM we only require the total expected number of edges from all layers to be $\omega(N(\log N)^{3+\delta})$ for $K=O(N^{1/2})$ (Condition (iii) of Theorem \ref{th:consrMLSBM}). This implies that we only require the expected number of edges per layer to be $\omega({N(\log N)^{3+\delta}}/{M})$ on average. For perspective, if $M$ grows faster than $ (\log N)^{3+\delta}$, then the average number of edges per layer needs to grow only at $O(N)$, which is the sparse bounded degree regime. This case is extremely challenging for single layer networks. In comparison, the consistency of the MLE in MLSBM requires the average expected number of edges per layer to be $\omega(N(\log N)^{3+\delta})$ \citep{cwa12} and hence the average degree per layer must grow at least as $ (\log N)^{3+\delta}$ . Thus consistency can be achieved with a large number of
relatively sparse layers. This is particularly
important as most modern applications of community detection in multi-layer
graph fall under this asymptotic scenario.

\subsection{A Large number of communities}
Under MLSBM, consistent community detection is possible when the number of communities grows as $K=O(N^{1/2})$ and the total expected number of edges is $\omega(MN(\log N)^{3+\delta})$ as both $M$ and $N$ grow. However, if we assume $K=O((MN)^{1/2-\epsilon})$ for some $\epsilon>0$, then we require the total expected number of edges to be $\omega(M^{2}N(\log N)^{3+\delta})$ which is unrealistically dense. On the other hand, under RMLSBM consistent estimation is possible with comparable edge density even when the number of communities grows faster, either as $K=O((MN)^{1/2-\epsilon})$ when both $M$ and $N$ grow but $M=O(N)$, or as $K=O(\frac{N}{\log M \log N})$ when $N$ grows slower than $M$ (Conditions (ii) and (iv) of Theorem \ref{th:consrMLSBM}). Hence the restricted model is advantageous for community detection in networks with a large number of communities.

\section{Baseline procedures}

We define three intuitively simple baseline procedures for community detection in multi-layer networks. The first two are based on aggregating the layers of the graph and the third one is an ensemble of results from single layer community detection through majority voting.

The first aggregate procedure, which we call ``agg-mean" creates a binary network on the nodes by adding an edge between two nodes if they are connected in more than half of the layers. Hence an edge between two nodes, $A_{ij}^{agg-mean}$ is a Bernoulli random variable with probability
\begin{equation}
P_{ij}^{agg-mean}=P(\sum_{m} A_{ij}^{(m)} > M/2 ).
\label{agg-mean}
\end{equation}
However, this method of collapsing a multi-layer graph into a single layer graph is not very useful for the sparse graph regimes we are interested in, because the probability that $\sum_{m} A_{ij}^{(m)} >1$ asymptotically vanishes. This can be seen as follows: the random variable $\sum_{m} A_{ij}^{(m)}$ is a sum of $M$ Bernoulli random variables with different probabilities $P_{ij}^{(m)}$. Hence $\sum_{m} A_{ij}^{(m)}$ follows a Poisson-binomial distribution and
\begin{align*}
P(\sum_{m} A_{ij}^{(m)} > 1) &=1- \{ P(\sum_{m} A_{ij}^{(m)} =0) +P(\sum_{m} A_{ij}^{(m)} =1) \} \\
&= 1-\{ \prod_{m} (1-P_{ij}^{(m)}) +\sum_{m} P_{ij}^{(m)} \prod_{k \neq m}(1-P_{ij}^{(k)})\} \rightarrow 0,
\end{align*}
if $P_{ij}^{(m)} \rightarrow 0 $ as $N\rightarrow \infty$ with $M$ remaining fixed. Hence the new graph created by this procedure will have asymptotically few edges.

A more appropriate aggregate measure is to create a network by adding edges if $\sum_{m} A_{ij}^{(m)} > 0$. We call this procedure ``agg-sparse". Note that in this case the edge between two nodes $A_{ij}^{agg-sparse}$ is a Bernoulli random variable with probability
\begin{align}
P_{ij}^{agg-sparse} &=P(\sum_{m} A_{ij}^{(m)} > 0 )=1- P(\sum_{m} A_{ij}^{(m)} =0)=1-\prod_{m} (1-P_{ij}^{(m)}) \nonumber\\
 &\asymp 1-\exp(-\sum_{m}P_{ij}^{(m)}) \asymp \sum_{m} P_{ij}^{(m)},
\label{agg-sparse}
\end{align}
%\begin{align}
%P_{ij}^{agg-sparse} &=P(\sum_{m} A_{ij}^{(m)} > 0 )=1- P(\sum_{m} A_{ij}^{(m)} =0)=1-\prod_{m} %(1-P_{ij}^{(m)}) \nonumber\\
% &\rightarrow 1-\exp(-\sum_{m}P_{ij}^{(m)}) \rightarrow \sum_{m} P_{ij}^{(m)}.
%\label{agg-sparse}
%\end{align}
since $P_{ij}^{(m)} \rightarrow 0 $ as $N\rightarrow \infty$.
Clearly this network is also generated by a SBM with the same community assignment vector as the original multi-layer network.
%The success probabilities,
The probability of an edge,
given the block assignments, can also be written in terms of those of the original network as
$$P_{ij}^{agg-sparse}|(z_{i}=q,z_{j}=l) \approx \sum_{m} \pi_{ql}^{(m)}.$$
Hence from known results on single layer SBM, a maximum likelihood procedure will be able to recover the node assignments consistently \citep{cwa12}.
From now on ``aggregate SBM" will refer to this sparse model. We compare this baseline aggregate SBM with the multi-layer models, MLSBM and RMLSBM in terms of minimax rates \citep{zhang15,gao15} and consistency thresholds \citep{mossel14, abbe15,hajek14} in the next section.

The third baseline procedure is performing community assignment through a scheme by which a node is assigned to a cluster if it belongs to that cluster in majority of the cluster assignments through MLEs in the individual layers. The cluster labels obtained from different single layer MLEs are aligned with each other by solving the linear sum assignment problem.

\section{Minimax rates and sharp thresholds}
In this section we derive the minimax rates of misclassification error and sharp thresholds for consistency of community detection in MLSBM and the aggregate SBM. For this analysis, we further assume that all the layers are informative of the underlying community assignments even though the quality of that information in terms of ``signal to noise ratio" can vary, i.e.,  either all layers have more intra-community edges compared to inter-community edges or vice-versa. Formally, $\pi_{qq}^{(m)} \geq \pi_{ql}^{(m)}$ for all $q, l, m$, or $\pi_{qq}^{(m)} \leq \pi_{ql}^{(m)}$ for all $q, l, m$. To align notations and settings with \citet{zhang15}, we slightly modify the growth condition on class sizes of Theorem 4 and 5 as $N_{q} \in [\frac{N}{s K},\frac{s N}{ K}]$ with $s \geq 1 $ and redefine the parameter space of our undirected symmetric MLSBM with no self loops as
\begin{align}
\Theta^{ML} (N,K,M,\mathbf{a},\mathbf{b},\beta)=&\Bigg \{ (z,\{P_{ij}^{(m)}\}) : N_{q} \in \left[\frac{N}{s K},\frac{s N}{ K}\right], \forall q, P_{ij}^{(m)} \geq \frac{a^{(m)}}{N} \nonumber \\
& \text{if } z_{i}=z_{j} \text{ and } P_{ij}^{(m)} \leq \frac{b^{(m)}}{N}  \text{ if } z_{i} \neq z_{j}, \, \forall m \Bigg \},
\label{newMLSBM}
\end{align}
with $P, z, N_{q}, s, N, K, M$ as defined previously. Note that the parameters $a^{(m)}$ and $b^{(m)}$ represent the lowest intra-community probability and the highest inter-community probability for layer $m$ respectively. As per assumption, $a^{(m)} > b^{(m)} $ within a layer $m$, however there is no assumption among the relationships of the parameters across layers. We define $I^{(m)}$ as the Renyi divergence \citep{van14} of order 1/2 between two Bernoulli distributions $Bern(\frac{a^{(m)}}{N})$ and $Bern(\frac{b^{(m)}}{N})$, i.e.,
\begin{equation}
I^{(m)}=-2 \log \left( \sqrt{\frac{a^{(m)}}{N} \frac{b^{(m)}}{N}} +\sqrt{1-\frac{a^{(m)}}{N}} \sqrt{1-\frac{b^{(m)}}{N}} \right).
\end{equation}
Let $\bar{z}$ denote the true community labels of the MLSBM and $\hat{z}$ be an estimate of it. Then we define the mis-clustering rate of $\hat{z}$ with respect to $\bar{z}$ up to permutations as
$$r(\bar{z},\hat{z})=\inf_{\delta}d_{H}(\bar{z},\delta (\hat{z}))/N,$$
where $\delta(\cdot)$ is a permutation of the community labels and $d_{H}(\cdot)$ is the Hamming distance.
Then we have the following result for MLSBM (proved in the Appendix).

\begin{thm}
Under the assumption that $\frac{N\sum_{m}I^{(m)}}{K \log K} \rightarrow \infty $, then
\begin{equation}
\inf_{\hat{z}} \sup_{\Theta^{ML}} E[r(\bar{z},\hat{z})] = \begin{cases}
\exp(-(1+\epsilon_N)\frac{N\sum_{m}I^{(m)}}{2}), & K=2, \\
\exp(-(1+\epsilon_N)\frac{N\sum_{m}I^{(m)}}{s K}), & K\geq 3,
\end{cases}
\end{equation}
for any $s \in [1,\sqrt{5/3}]$ and some sequence $\epsilon_N=o(1)$. Moreover, if $\frac{N\sum_{m} I^{(m)}}{K}=O(1)$, then $\inf_{\hat{z}} \sup_{\Theta^{ML}} E [r(\bar{z},\hat{z})] \geq c$ for some constant $c$, i.e., at least a constant fraction of nodes are mis-clustered.
\label{MLSBMrisk}
\end{thm}
The above theorem implies that for MLSBM, minimax risk of error decays exponentially and if $\frac{N\sum_{m}I^{(m)}}{K \log K} \rightarrow \infty $, the rate goes to 0 asymptotically, i.e., exact recovery of community labels is possible. Moreover from the proof of Theorem \ref{MLSBMrisk} in the Appendix, there exists a procedure which achieves this rate. On the other hand if $\frac{N\sum_{m} I^{(m)}}{K}=O(1)$, then the minimax risk of error is lower bounded by a constant (see the part on lower bound in the proof in Appendix) implying that consistent recovery is not possible in such situations.

Since the model ``agg-sparse" is itself a single layer SBM and $ \sum_{m} P_{ij}^{(m)} \geq \sum_{m} \frac{a^{(m)}}{N}$ if $ z_{i}=z_{j} $ and $ \sum_{m} P_{ij}^{(m)} \leq \sum_{m} \frac{b^{(m)}}{N}$ if $ z_{i} \neq z_{j}$, then defining $I^{agg}$ as
\begin{equation}
I^{agg}=-2 \log \left( \sqrt{\frac{\sum_{m} a^{(m)}}{N} \frac{\sum_{m}b^{(m)}}{N}} +\sqrt{1-\frac{\sum_{m}a^{(m)}}{N}} \sqrt{1-\frac{\sum_{m}b^{(m)}}{N}} \right),
\end{equation}
we have the following result using Theorem 1.1 of \citet{zhang15}.

\begin{thm}
If $\frac{NI^{agg}}{K \log K} \rightarrow \infty $, then
\begin{equation}
\inf_{\hat{z}} \sup_{\Theta^{agg}} E[r(\bar{z},\hat{z})] = \begin{cases}
\exp(-(1+\epsilon_N)\frac{NI^{agg}}{2}), & K=2, \\
\exp(-(1+\epsilon_N)\frac{NI^{agg}}{s K}), & K\geq 3,
\end{cases}
\end{equation}
for any $s \in [1,\sqrt{5/3}]$ and some sequence $\epsilon_N=o(1)$. In addition, if $\frac{N I^{agg}}{K}=O(1)$, then $\inf_{\hat{z}} \sup_{\Theta^{agg}} E[r(\bar{z},\hat{z})] \geq c$ for some constant $c$, i.e., at least a constant fraction of nodes are mis-clustered.
\label{riskagg}
\end{thm}
The previous two theorems state results about the fundamental properties of the two models which allow us to compare the models without going into the specifics of the method used to compute the class assignments in practice.

Since the Renyi divergence $I^{(m)} \geq 0$ for all $m$, we have $\sum_{m}I^{(m)} \geq I^{(m)}$ for all $m$. Hence the minimax rate for MLSBM is lower than all individual single layer SBMs. Moreover, since Renyi divergence is convex, we have $\frac{1}{M}\sum_{m} I^{(m)} \geq \frac{1}{M} I^{agg}$ asymptotically. This can be shown using Jensen's inequality with the concave functions $\log(x)$ and $\sqrt{x}=\sqrt{\frac{b^{(m)}}{a^{(m)}}}$ (see Theorem 11 of \citet{van14} for a proof), and then noting that asymptotically $I^{(m)} \asymp \frac{(a^{(m)}-b^{(m)})^2}{a^{(m)}N}$ \citep{zhang15}. Hence the minimax rate of MLSBM is at most that of the aggregate graph. Note that equality in the above inequality is achieved if and only if all the $I^{(m)}$s are equal and $\frac{b^{(m)}}{a^{(m)}}$ is equal for all $m$. We recognize the quantities $\frac{b^{(m)}}{a^{(m)}}$ and $I^{(m)}$ as signal to noise ratios  in the $m$th layer. Hence the MLSBM has lower minimax rate compared to the aggregate SBM as long as the signal quality in different layers varies.

This result will be intuitively apparent if we note from the proof of the above theorems that, given the parameters are known or accurately estimated, the penalized maximum likelihood (ML) decision rule, which attains the minimax rate of error in MLSBM, weights the edges from different layers by $c^{(m)}$ before adding. The penalty terms also get weighted by $k^{(m)}$ before being added. The quantity $c^{(m)}=\log \frac{a^{(m)}(1-b^{(m)}/N)}{b^{(m)}(1-a^{(m)}/N)}$  can be thought of as a measure of the signal to noise ratio. Hence, layers with high signal to noise ratio, i.e., high quality information for the purpose of community detection, get more weight. In contrast, the penalized ML decision rule in aggregate graph SBM by construction adds layers without weighting. Hence intuitively the result on minimax rates makes sense, since if all layers contain the same amount of information, then it is immaterial if the decision rule weights the graphs by information content or not, but in all other cases giving more weight to the more informative layer pays off.

Moreover, while it is clear that MLSBM has lower minimax rate than individual layer SBMs, it is not true trivially for the aggregate graph. Since $I^{(m)}$ can be written in terms of signal to noise ratio as $I^{(m)} \asymp \frac{(a^{(m)}-b^{(m)})^2}{a^{(m)}N}$, consequently for $I^{agg}$ to be large, the sum of the probabilities $\sum_{m}a^{(m)}$ and $\sum_{m} b^{(m)}$ must be well separated. This is not always guaranteed as large $a^{(m)}$'s and $b^{(m)}$'s with relatively low difference can overshadow a large difference in smaller $a^{(m)}$'s and $b^{(m)}$'s while adding. We will take this point up again in the next section where we discuss sharp thresholds for consistency.

We note that the model RMLSBM is a MLSBM with a restricted parameter space $\Pi^R$. Hence Theorem \ref{MLSBMrisk} will give the minimax rate under the restricted parameter space with the divergence in the $m$th layer being $I^{(m)} \asymp \frac{(\phi_{a}^{(m)}-\phi_{b}^{(m)})^2}{\phi_{a}^{(m)}N}$, where $\phi$ is the transformation of the parameters in RMLSBM as defined before. In particular, we have $\text{logit}(\phi_{a}^{(m)})=a+\beta_{m}$. The rate for the aggregate SBM under RMLSBM can similarly be obtained using Theorem \ref{riskagg} with $I^{agg}$ being $I^{agg} \asymp \frac{(\sum_{m}\phi_{a}^{(m)}-\sum_{m}\phi_{b}^{(m)})^2}{\sum_{m}\phi_{a}^{(m)}N}$. This implies that (a) if RMLSBM is the true data generating model then it has lower minimax rate compared to each of the individual layers, and (b) by the earlier discussion it also has lower minimax rate compared to the aggregate SBM constructed from a RMLSBM graph, since neither $I^{(m)}$ nor the ratio $\frac{\phi_{a}^{(m)}}{\phi_{b}^{(m)}}=1+\frac{\exp(a-b)-1}{1+\exp(a+\beta_m)}$ is equal for all $m$.

\subsection{Sharp consistency thresholds}
We derive sharp thresholds for strong and weak consistency for community detection \citep{mossel14,abbe15} in MLSBM and the aggregate SBM under two scenarios: sparse graph with average degree per layer $o(\log n)$ and ultra-sparse graph with average degree per layer $o(1)$.

In the first case, let $a^{(m)}=\alpha_{1}^{(m)} \log N$ and $b^{(m)}=\alpha_{2}^{(m)} \log N$ with $\alpha_{1}^{(m)} \geq \alpha_{2}^{(m)}>0$ for all $m$. Then Corollary 4.1 of \citet{zhang15} gives that assuming $K=N^{o(1)}$, the sharp threshold for the existence of a strongly consistent estimator for the $m$th layer SBM is $\frac{\sqrt{\alpha_{1}^{(m)}}-\sqrt{\alpha_{2}^{(m)}}}{\sqrt{K}} >1$. Hence for the aggregate SBM this threshold is $\frac{\sqrt{\sum_{m} \alpha_{1}^{(m)}}-\sqrt{\sum_{m} \alpha_{2}^{(m)}}}{\sqrt{K}} >1$. Clearly, if the threshold is met in each of the layers, then it will be met in the aggregate SBM as well. However in a more realistic case where this threshold is not met in all the layers, whether the aggregate SBM will have a strongly consistent estimator or not will depend on whether the sum of probabilities meets the threshold of well separation or not, which in turn will depend on the relatively denser layers.
To see this, note that this threshold can be written
as $\frac{\sum_{m} \alpha_{1}^{(m)}-\sum_{m} \alpha_{2}^{(m)}}{\sqrt{\sum_{m} \alpha_{1}^{(m)}}+\sqrt{\sum_{m} \alpha_{2}^{(m)}}} >{\sqrt{K}}$.
%To see this, note from Theorem \ref{riskagg} that this threshold can be written as $\frac{\sum_{m} \alpha_{1}^{(m)}-\sum_{m} \alpha_{2}^{(m)}}{\sqrt{\sum_{m} \alpha_{1}^{(m)}}} >{\sqrt{K}}$.
For aggregate graph, the denominator of this quantity is dominated by the dense layers, and hence the difference in $a$ and $b$ must be large in dense layers for the aggregate to be consistent. In other words, strong signals in sparse layers will get ignored if the signal in dense layers are not strong.

On the other hand, for MLSBM, strong consistency is achieved if any of $\frac{NI^{(m)}}{K} \rightarrow \infty$ or their sum goes to infinity. This implies that the threshold is $\sum_{m} \frac{\sqrt{\alpha_{1}^{(m)}}-\sqrt{\alpha_{2}^{(m)}}}{\sqrt{K}} >1$, which is achieved if at least one of the layers achieves consistency threshold or the layers together achieve the threshold. By the argument before, this threshold consists of sum of normalized signal to noise ratios, hence all layers, dense or sparse, get equal weightage in determining the threshold.
The consistency threshold for RMLSBM using Theorem \ref{MLSBMrisk} is
$\sum_{m} \frac{\sqrt{\alpha_{1,\phi}^{(m)}}-\sqrt{\alpha_{2,\phi}^{(m)}}}{\sqrt{K}} >1$,
%The consistency threshold for RMLSBM using Theorem \ref{MLSBMrisk} is $\sum _{m}\frac{ \alpha_{1,\phi}^{(m)}- \alpha_{2,\phi}^{(m)}}{\sqrt{ \alpha_{1,\phi}^{(m)}}} >{\sqrt{K}}$, which implies $ \sum_{m} \frac{\sqrt{\alpha_{1,\phi}^{(m)}}-\sqrt{\alpha_{2,\phi}^{(m)}}}{\sqrt{K}} >1$,
where $\phi_{a}^{(m)}=\alpha_{1,\phi}^{(m)} \log N$ and $\phi_{b}^{(m)}=\alpha_{2,\phi}^{(m)} \log N$ with $\alpha_{1,\phi}^{(m)} \geq \alpha_{2,\phi}^{(m)}>0$ for all $m$. The corresponding threshold for aggregate SBM generated from a RMLSBM is $\frac{\sum \alpha_{1,\phi}^{(m)}- \sum \alpha_{2,\phi}^{(m)}}{\sqrt{ \sum \alpha_{1,\phi}^{(m)}}} >{\sqrt{K}}$. Here we note that the threshold for RMLSBM is also the sum of normalized signal to noise ratios. However since the parameter space is restricted, the difference between inter and intra community parameters are uniform across layers, and variations in the above sum only come from the normalizing factor due to the layer specific sparsity parameter.

Qualitatively, the minimax rate and consequently the threshold in MLSBM take into account variations in both signal quality and sparsity while adding contributions from different layers. RMLSBM tries to estimate the signal to noise ratio in each layer by two parameters, one global parameter which signifies the aggregate signal quality, and the other layer specific parameter which signifies sparsity. Hence although RMLSBM ignores the variation in signal quality, it attempts to reduce the undue influence of dense layers by taking into account the variation in sparsity. The aggregate SBM, on the other hand, does not take into account either the signal quality or the sparsity, and hence is heavily influence by dense layers irrespective of signal quality. Hence both RMLSBM and aggregate SBM would perform well if all the layers have similar signal strength and similar density. If the layers do not have similar density but the signal strength across layers can somewhat be well approximated by an average signal strength, RMLSBM will still be able to detect it through the noise and perform well. Clearly, RMLSBM and aggregate graph will not perform well if both signal strength and sparsity of layers vary widely, and we need to resort to MLSBM in such cases.

In the bounded degree case, while consistent recovery is not possible  in each of the layers since the graph is not fully connected (only detection is possible), a consistent recovery is still possible in the multi-layer models.
The condition for consistent recovery in MLSBM with $a^{(m)}=o(1)$ and $b^{(m)}=o(1)$ is
$\sum_{m}\frac{a^{(m)}-b ^{(m)}}{(\sqrt{a^{(m)}}+\sqrt{b^{(m)})K}} \rightarrow \infty$.
%The condition for consistent recovery in MLSBM with $a^{(m)}=o(1)$ and $b^{(m)}=o(1)$ is $\sum_{m}\frac{a^{(m)}-b ^{(m)}}{\sqrt{ a^{(m)}K}} \rightarrow \infty$.
Note that the condition for detection or weak recovery defined as finding a partition correlated with the true community structure for two communities is $\frac{a-b}{\sqrt{a+b}} >2$ \citep{mossel12,mossel13}.

\section{Estimation using mixture model approach}

Simultaneous maximum likelihood estimation of parameters and class
assignments in the stochastic blockmodel is a difficult problem
%as noted by several earlier papers
\citep{ns01, cwa12, rqf12}. The same difficulties remain in
the MLSBM
%multi-layer stochastic blockmodel
and its restricted version.
Consequently, to obtain an estimation algorithm here, we view the MLSBM
%multi-layer stochastic blockmodel
as a mixture model with discrete
latent variables $Z$.
%However, when the model is viewed as a mixture model with discrete latent variables,
In this case, $Z_{i}$ is a missing
random variable that follows a multinomial distribution with $K$
parameters: $Z_{i}\sim Mult(1,\alpha=(\alpha_{1},\alpha_{2},\ldots,\alpha_{K}))$.
%In particular,
We follow the framework laid
out by \citet{dpr08} to simultaneously estimate the conditional blockmodel
parameters and the class assignments with variational EM technique.
The derivations for MLSBM are straightforward extensions of the corresponding
formula in \citet{dpr08} and are omitted in this paper while the update rules for RMLSBM have been derived in the Appendix
%We instead provide
The update steps for MLSBM and RMLSBM are also provided in the Appendix under Algorithm 1 and Algorithm 2 respectively.

\begin{figure}[hp]
\centering{}
\begin{subfigure}{0.35\textwidth}
\centering{}
\includegraphics[width=.95\linewidth]{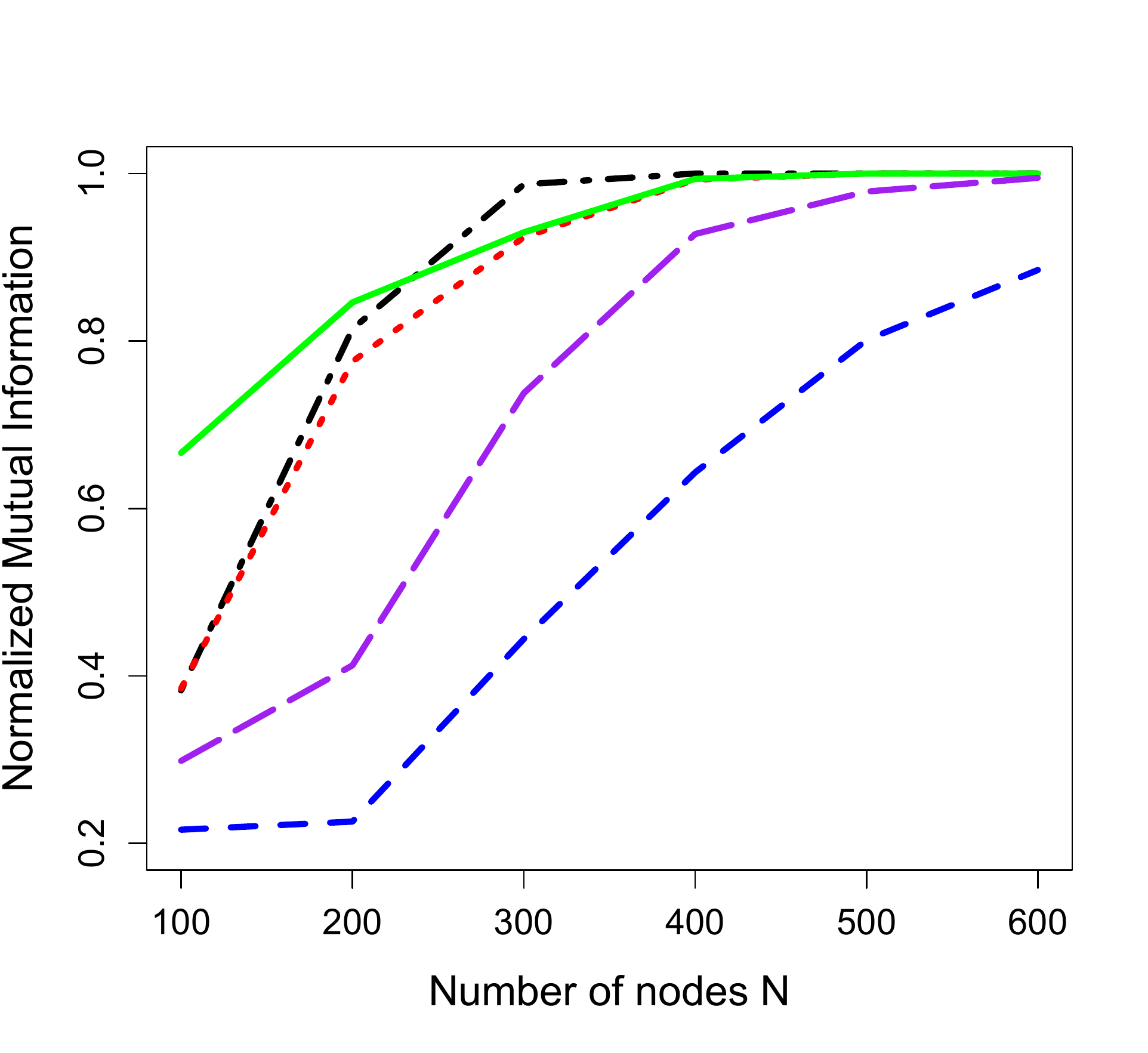}
\end{subfigure}%
\begin{subfigure}{0.35\textwidth}
\centering{}
\includegraphics[width=.95\linewidth]{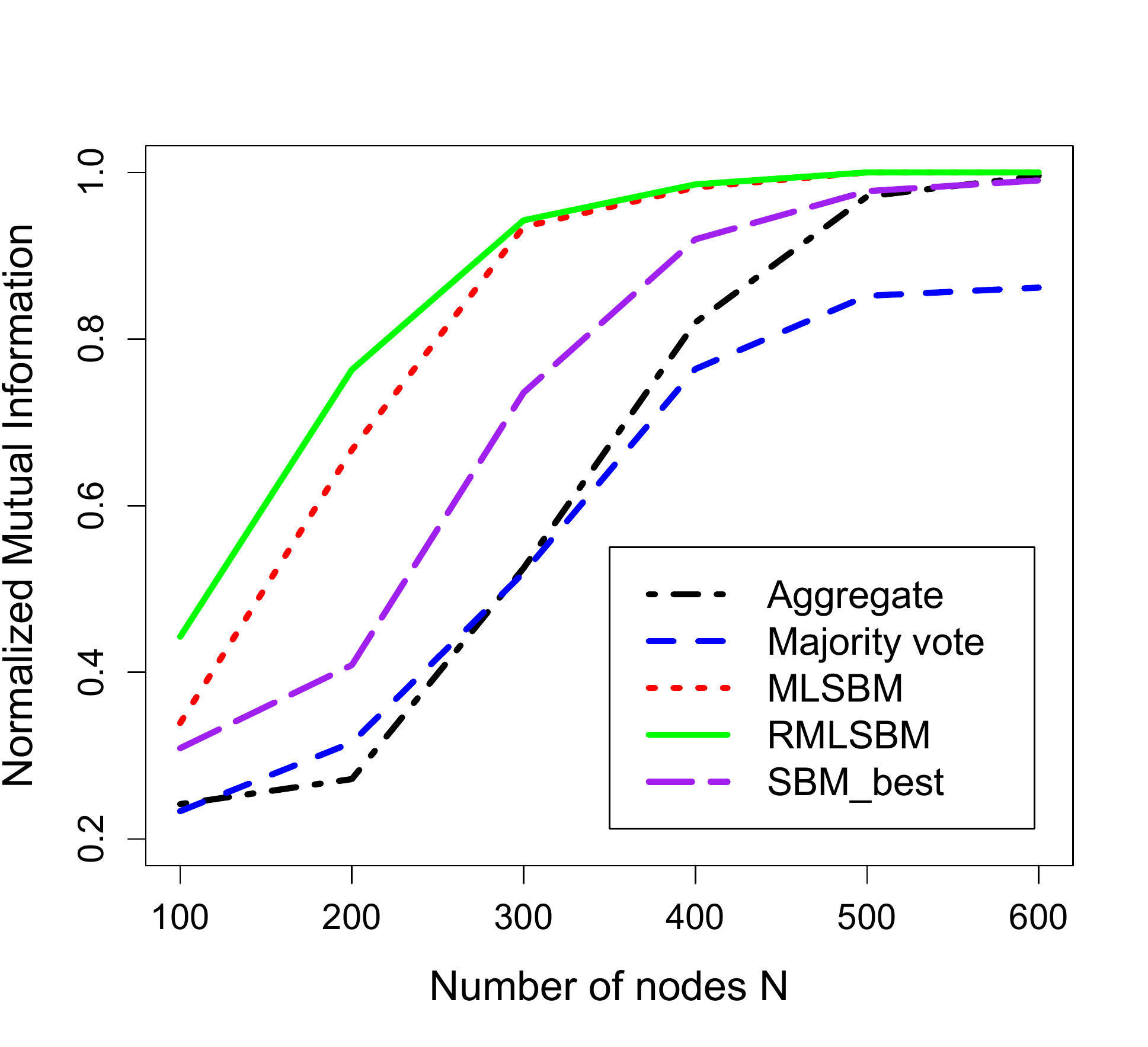}
\end{subfigure}
\begin{center} (a) \hspace{160pt} (b) \end{center}
\begin{subfigure}{0.35\textwidth}
\centering{}
\includegraphics[width=.95\linewidth]{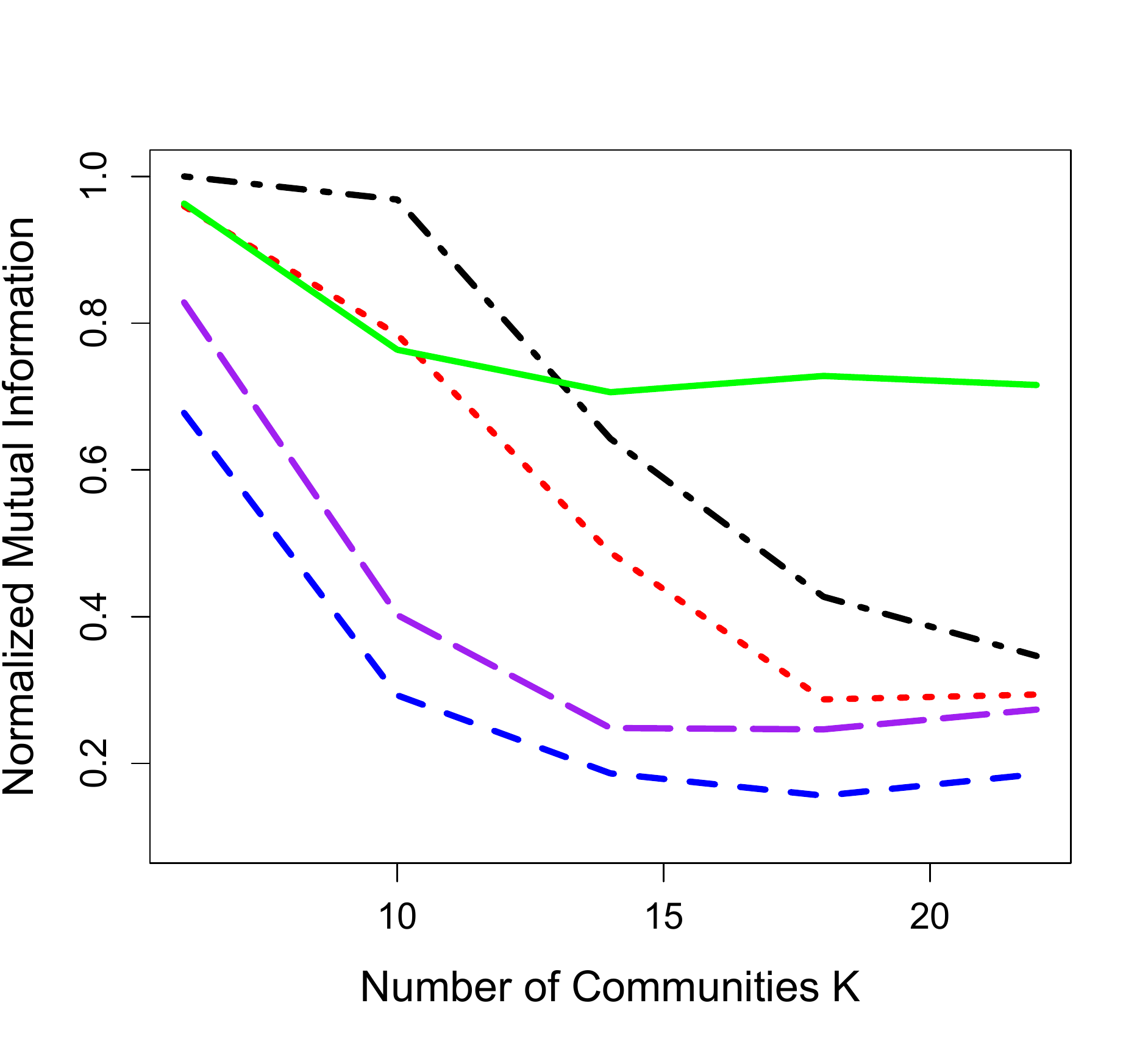} \label{fig:inckfixn}
\end{subfigure}%
\begin{subfigure}{0.35\textwidth}
\centering{}
\includegraphics[width=.95\linewidth]{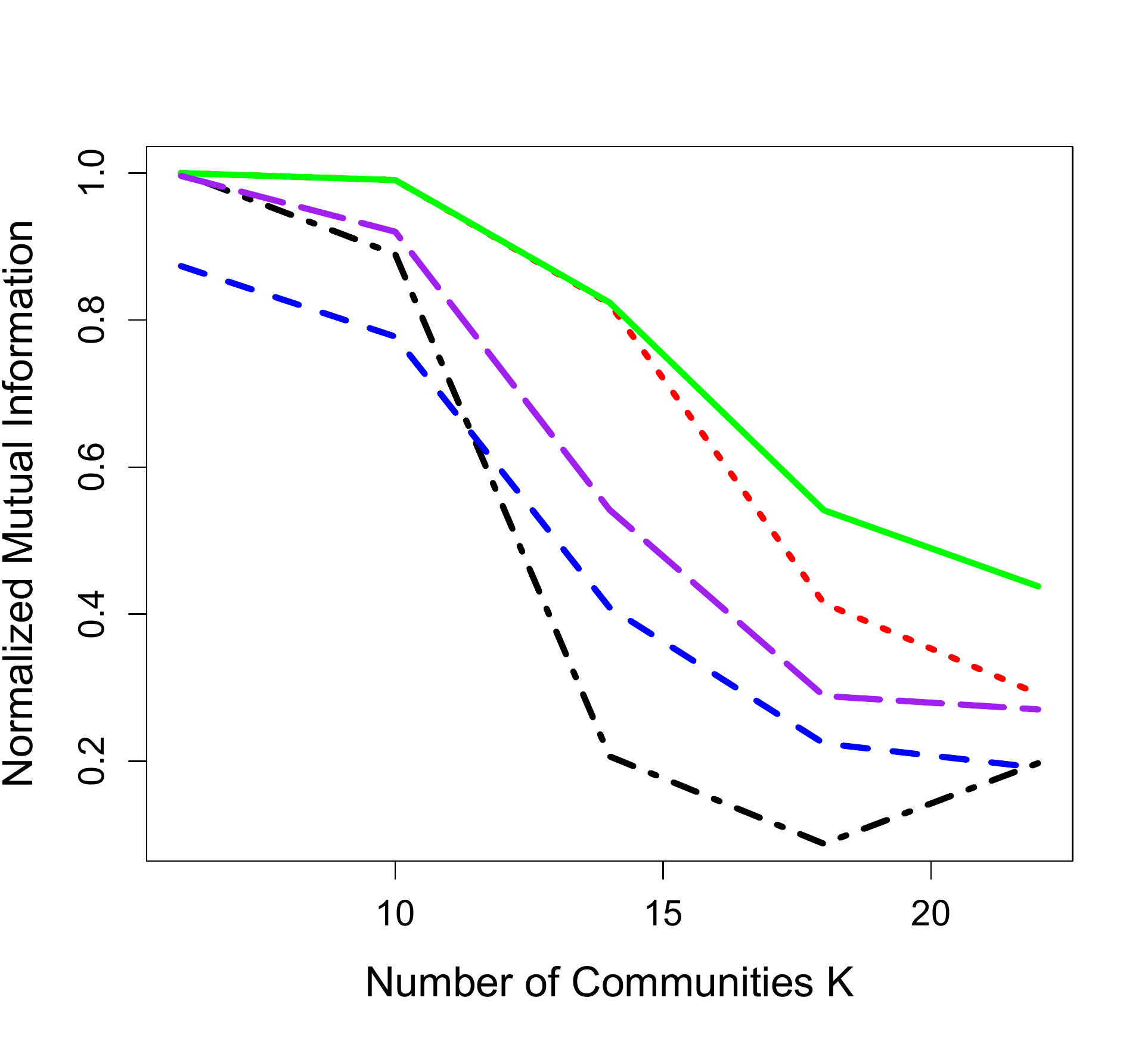} \label{fig:inckfixn}
\end{subfigure}
\begin{center} (c) \hspace{160pt} (d) \end{center}
\begin{subfigure}{0.35\textwidth}
\centering{}
\includegraphics[width=.95\linewidth]{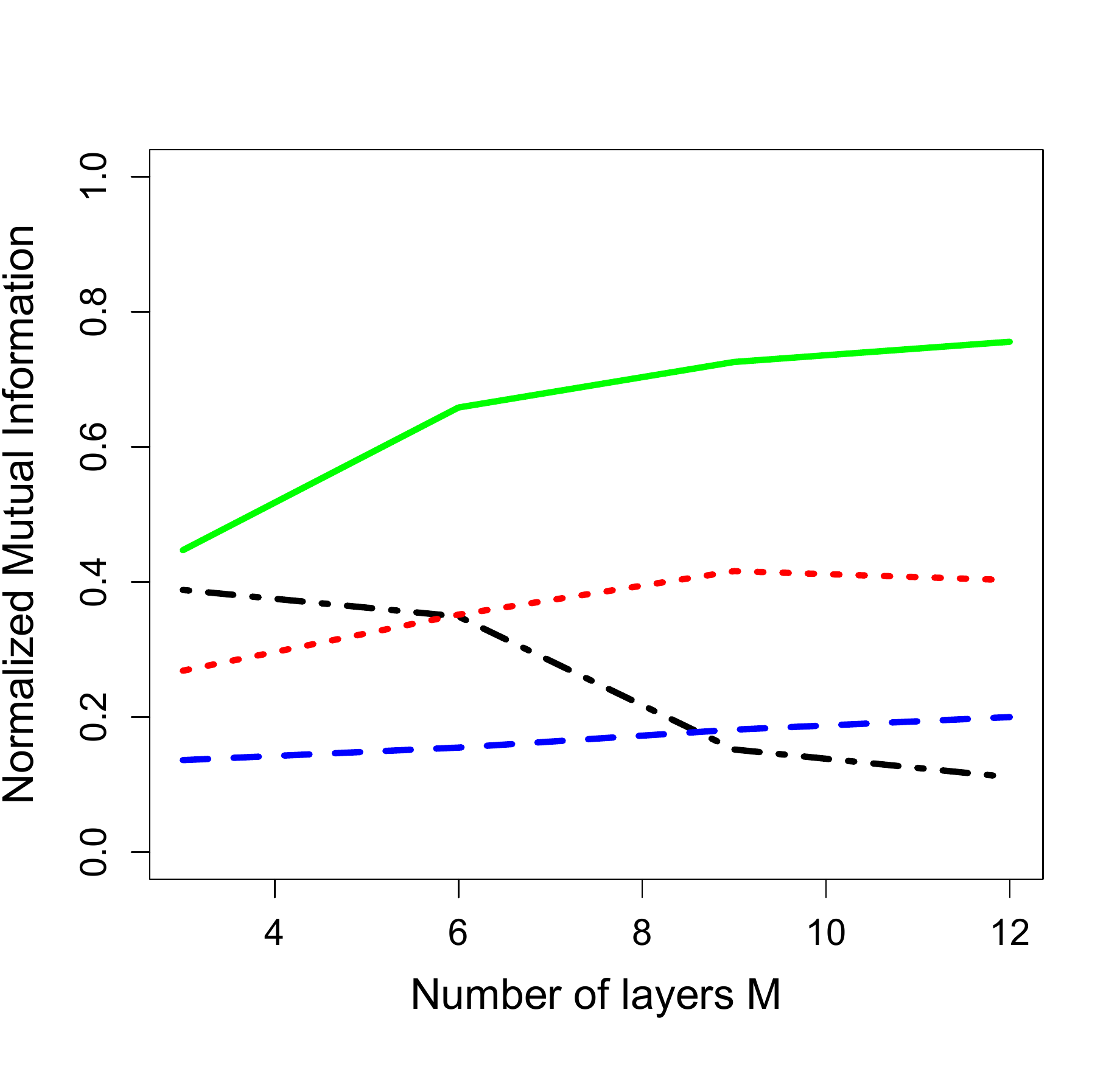}
\end{subfigure}%
\begin{subfigure}{0.35\textwidth}
\centering{}
\includegraphics[width=.95\linewidth]{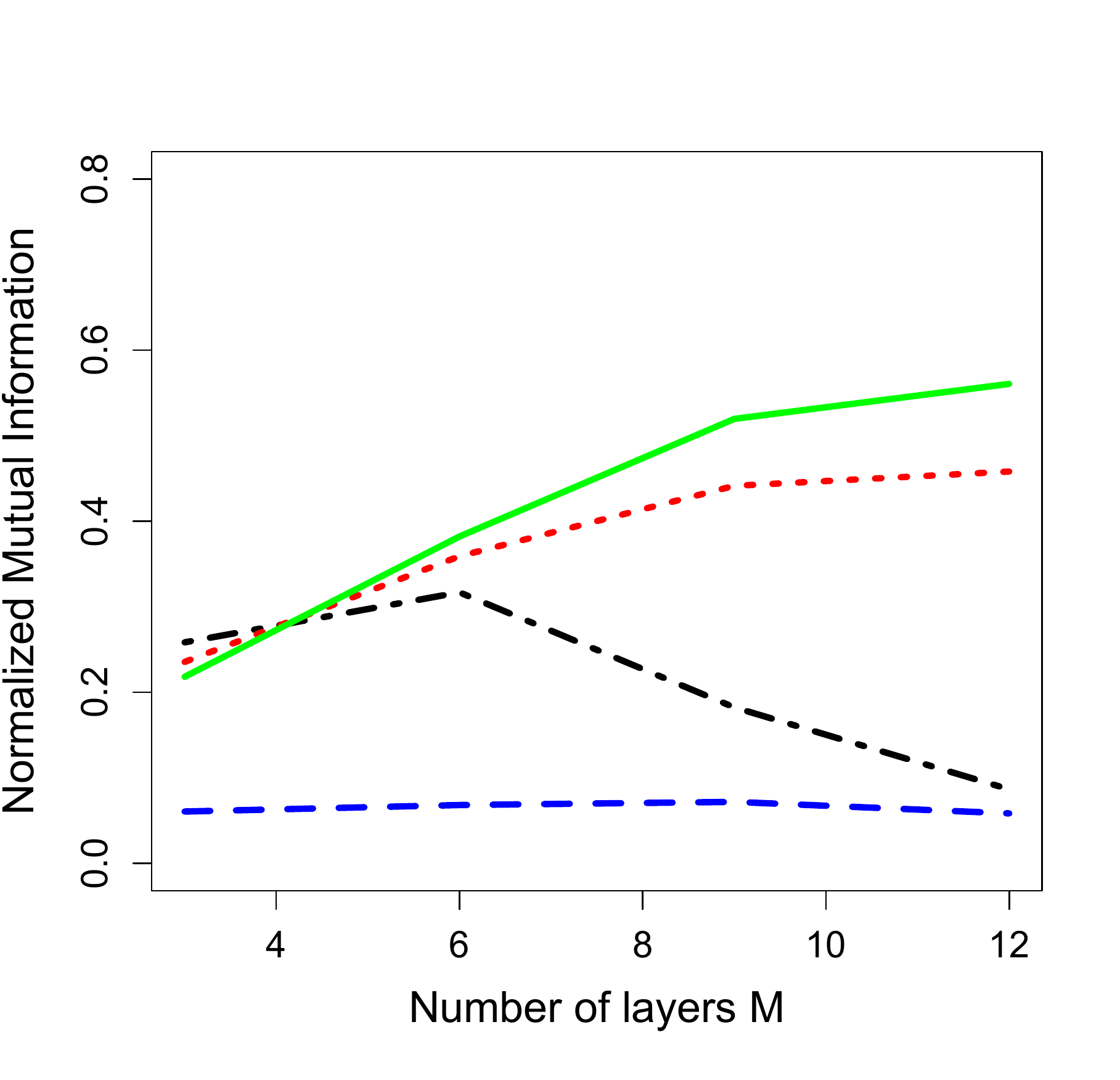}
\end{subfigure}
\begin{center} (e) \hspace{160pt} (f) \end{center}
\caption{\small Comparison of the performance of various methods for three simulation settings under two scenarios: all layers are sparse and have strong SNR (left column: (a)(c)(e)), and the layers are mixed in terms of sparsity and SNR (right column: (b)(d)(f)).
%all strong (left column: a,c,e) and mixed (right column: b,d,f),
(a, b) fixed $K=10$ and $M=5$ while $N$ increases
from 100 to 600; (c, d) fixed $N=400$ and $M=5$ while $K$ increases from 6 to 22;
(e, f) fixed $N=300$ and $K=15$ while $M$ increases from 3 to 12. The legend in Figure (b) is common to all figures. SBM\_best indicates the result from the best performing MLE in the single layer SBMs.}
\label{simulation_figure}
\end{figure}

%\begin{multicols}{2}
\section{Simulation results}
In this section we numerically test the asymptotic results and compare the performance of the methods through a simulation study.
%on data simulated from the MLSBM.
%multi-layer stochastic block model.
We generate data from the more general model, MLSBM. We then compare the relative performance of the two multi-layer methods (MLE and RMLE) between themselves as well as with single layer methods and baseline methods such as majority voting and MLE in aggregate SBM. The comparison is done under
various settings on the number of nodes $N$, the number of communities $K$,
the number of types of relations $M$, and the expected total number of edges $L$.

Since the true class labels of the nodes are known in simulated data,
we compare the class assignments from different methods with the true
labels. We use correct clustering rate (CCR) and normalized mutual information (NMI) as measures of
similarity between partitions. The CCR counts the fraction of nodes whose cluster assignment matches the true class label (as determined by the true class label of the majority of nodes in that cluster).
The higher the CCR, the better the performance of the clustering method.
The NMI is an information theoretic measure of the mutual dependence or similarity of two random variables. The NMI takes values in the range of $0$ to
$1$, with $0$ indicating random cluster assignment with respect to the true class labels,
and $1$ indicating perfect match between the
true and assigned clusters.
%The former case
If NMI is 0, it means even though the
cluster assignment was not completely random and done according to
some algorithm, the solution presents no information regarding the
true class labels. Since the results in terms of CCR are very similar to that of NMI, we omit those results here to save space.

In all the simulation studies we repeat the experiments $50$ times
and take the average of our measures across them. We first generate the node labels independently from a multinomial distribution with probabilities
$P(Z_{i}=k)=\alpha_{k}$. Then we generate the data using the node labels and $M$ different connectivity matrices, all of which give
larger probability to connections within groups in comparison
to the connections between groups. However, we vary the ``signal to noise ratio" (SNR) from layer to layer by varying the ratio of the diagonal and off diagonal elements of the parameter matrix.

We consider two scenarios: (i) all layers are sparse and have strong SNR, (ii) the layers are mixed in terms of sparsity and signal strength in the following way: two layers are sparse and have strong signal, two layers are dense and have weak signal, and one layer is dense with strong signal. While the first scenario is a rather idealistic scenario where all layers are ``similar" in the sense that they are sparse and strongly informative about the underlying community structure, the second scenario (also considered in \citet{pai13}) is more realistic in applications. For the first scenario, the SNR is kept at 3-4 and sparsity is varied slightly from layer to layer in such a way that variational EM algorithm for community detection on each of the layer individually gives very similar performance. The connectivity matrix parameters are then sampled from a uniform distribution within a small range so as to maintain SNR requirement while having different values for each of the entries of the matrix. For the second scenario, the informative strong signal layers have a SNR of $3$ while the non-informative weak signal layers have a SNR only marginally greater than 1. We again sample the actual values of the parameters from a uniform distribution within a small range.

The initial guess for the variational algorithm in both MLE and RMLE is obtained by a two step procedure. On a randomly selected layer we first run spectral clustering to generate an initial guess and then we use this to run a variational EM algorithm on that layer. We use the class assignment and fitted SBM parameters from that layer as our initial guess for the MLSBM parameters. In our simulation results described below, the final solution of class assignments for both the MLE and the RMLE mostly turns out to be an improved estimate of the true class assignments irrespective of which layer we choose to initialize the method.

\subsection{Fixed $K$ and $M$ while $N$ increases}
%{Fixed number of communities}
In this simulation, we take $M=5$ types of
edges or network layers, each with a separate connectivity matrix
inducing a different network according to the schemes described above. We keep the number of communities $K$ fixed at $10$ and vary the number of nodes $N$ from $100$ to $600$. The aim of this study is to compare the two multi-layer methods with the single layer methods and baseline methods in terms of the number of nodes required to achieve a consistent estimation of community assignment with moderately low number of communities. Figures \ref{simulation_figure}(a) and (b) display the results from this study for the two scenarios respectively. Clearly the MLE in MLSBM and RMLSBM reach NMI of close to 1 faster than the single layer ones as well as majority voting as the number of nodes increases. The algorithm in aggregate layer performs similarly to that in MLSBM and RMLSBM for the first (all strong signal) scenario (Figure \ref{simulation_figure}(a)), however it performs poorly for the second (mixed signals) scenario (Figure \ref{simulation_figure}(b)). This shows that aggregating edges across layers works fine if the information quality is similar across layers, but it is not robust if the information content changes across layers. The accuracy of majority voting behaves similarly to the single layer ones. Moreover, for a small number of nodes, the MLE in RMLSBM performs better than all the other methods considered in both scenarios.

\subsection{Fixed $N$ and $M$ while $K$ increases}

In this simulation, we test the performance of the multi-layer methods against the single layer and baseline methods with increasing number of communities. We fix the number of nodes $N$ and the number of layers $M$ at 400 and 5 respectively, while we let $K$ increase from 6 to 22 in steps of 4. The results from this simulation study are displayed in Figures \ref{simulation_figure}(c) and (d). Whereas the accuracy of community detection in all the single layer methods and the majority voting decreases rapidly with increasing number of communities, the multi-layer methods explored here, especially
the RMLSBM, perform well even with a large number of communities. Between RMLSBM and MLSBM, RMLSBM clearly outperforms MLSBM as the number of communities grows. This simulation also serves as a test of robustness of RMLSBM for small number of communities. We notice that in both scenarios, RMLSBM behaves similarly to MLSBM and does not break down for small number of communities. In the all-strong scenario, the MLE in aggregate SBM outperforms both MLSBM and RMLSBM  for small communities, but similar to MLSBM, its accuracy also quickly drops as $K$ increases (Figure \ref{simulation_figure}(c)). In the mixed signal scenario, the MLE in aggregate SBM performs much worse compared not only to MLSBM and RMLSBM, but also to majority voting and the best performing MLE among the individual layers.
To put things into perspective,
%in terms of NMI
for the all-strong scenario,                                                                                                                                                                                                                                                                                                                                                                while the NMI for MLSBM, aggregate SBM, majority voting and the single layer SBMs reduce below 0.5, it settles to a value close to 0.8 for RMLSBM as the number of communities increases to 20.

\subsection{Fixed $N$ and $K$ while $M$ increases}
%{Increasing $M$ while $K$ and $N$ are fixed}
In this simulation, we keep the number of nodes $N$ and the number of communities $K$ fixed at 300 and 15 respectively, while we increase the number of layers $M$ gradually from 3 to 12.  For this simulation, each layer of the multi-layer network was generated from a $K$-class SBM with a simple connectivity matrix given by $P_{K \times K} =\lambda I_{K} + \epsilon 1_{K\times K} -\epsilon I_K$. In the first scenario, the parameters are $\epsilon=0.10 + U(-0.02,\ 0.02)$ and $\lambda =3\epsilon$, while in the second scenario, the parameters are  $\epsilon=0.09 + U(-0.03,\ 0.03)$ and $\lambda =U(1.5,\ 3) \epsilon$. Here $U(a,b)$ is a random number generated from the uniform distribution between $a$ and $b$. Note that in the first scenario, all layers are sparse and have strong signals, while in the second scenario, we let both sparsity and signal strength vary across the layers. This second scenario would be a good test of the robustness of different multi-layer methods.

We compare the performance of MLE in MLSBM and RMLSBM with majority voting and aggregate SBM in terms of the accuracy of community detection in Figures \ref{simulation_figure}(e) and (f). The curves for majority votes in both figures remain almost flat with increasing number of layers, indicating that the accuracy of community detection does not improve with more layers. The MLE of aggregate SBM performs well initially, but its accuracy quickly falls with increasing number of layers as the model assumption that $\sum_{m}A_{ij}^{(m)}>1$ with vanishing probability breaks down. For MLSBM, the accuracy increases initially, however the improvement quickly slows down and both the curves in Figures \ref{simulation_figure}(e) and (f) flatten with increasing layers. This is because the number of parameters to be estimated also keeps on increasing fast with increasing number of layers, which contributes to less efficiency. For RMLE, the accuracy of community detection generally increases with increasing number of layers and is almost always higher than all other methods.
%and becomes stable when the number of layers is more than 10.

%Figure 2 displays the results of all the simulations.
The three studies clearly point out the advantages of the multi-layer methods over the single layer ones and the baseline ones, as well as the relative advantage of RMLSBM over MLSBM within the scope of the simulations.

\section{Twitter UK politics dataset}

In this section we test our methods on a real dataset on interactions
between British Members of Parliament (MPs) in the social networking
site Twitter curated by \citet{gp13}. Although the original dataset consists of 419 nodes, we only considered the largest subset that is connected across all layers for our analysis. Hence our multi-layer network consists of 381 nodes. The  different
layers of network we have correspond to three direct relations:
``mentions'', ``follows'' and ``retweets'', and three derived relations:
``mentioned by the same person (co-mentions)'', ``followed by the
same person (co-follows)'', and ``retweeted by the same person (co-retweets)''.
All relations are assumed to be binary by assigning one if the relation
is true for at least one case (e.g., if at least one person follows
both MP $i$ and MP $j$, then the relation ``co-follows'' between
the two MPs is true). All the relations individually can be represented
as graphs. For the graphs with direct relations, ``mentions'', ``follows'' and
``retweets'', a directed edge from node $i$ to node $j$ implies
that MP $i$ mentioned, followed or retweeted respectively MP $j$
at least once in his/her tweets. We converted all directed edges into undirected edges for this analysis. Average degrees of nodes in different
network layers are presented in Table \ref{tab:degree}. Note that among
the direct layers, ``follows'' is relatively dense compared to ``mentions''
and ``retweets'', while the derived networks are overall much
denser compared to the direct ones.

\begin{table}[h]
\protect\caption{Average degrees of nodes in different network layers for Twitter UK politics data}
\centering{}%
\begin{tabular}{cccccc}
\hline
Mentions & Follows & Retweets & Co-Mentions & Co-Follows & Co-Retweets\tabularnewline
\hline
58.48 & 98.34 & 31.88 & 361.51 & 297.21 & 147.56\tabularnewline
\hline
\end{tabular} \label{tab:degree}
\end{table}

\begin{table}[h]
\protect\caption{The NMI and CCR for Twitter UK politics data}
%based on the best of 10 starting values.}
%\medskip{}
\begin{centering}
\begin{tabular}{c|cccccc}
\hline
Measure & Mentions & Follows & Retweets & Co-Mentions & Co-Follows & Co-Retweets \tabularnewline
\hline
NMI & 0.4522 & 0.5992 & 0.4610 & 0.3449 & 0.2520 & 0.4009\tabularnewline
CCR & 0.8182 & 0.9022 & 0.7926 & 0.7565 & 0.7053 & 0.8136\tabularnewline
\hline
\end{tabular}
\begin{center} (a) Individual network layers \end{center}
\end{centering}

%\smallskip
\begin{centering}
{\footnotesize
\begin{tabular}{c|cccc|cccc}
\hline
& \multicolumn{4}{c}{NMI} & \multicolumn{4}{|c}{CCR}\tabularnewline
\hline
 &Majority & Aggregate SBM & MLSBM & RMLSBM & Majority & Aggregate SBM & MLSBM & RMLSBM\tabularnewline
Direct & 0.5213 & 0.5819 & 0.6764 & 0.6821& 0.8477 & 0.8871 & 0.9527 & 0.9553 \tabularnewline
All & 0.3825 & 0.3326 & 0.5428 & 0.6250 & 0.7217 & 0.7506 & 0.8393 & 0.9107\tabularnewline
\hline
\end{tabular}
}
\end{centering}
\begin{center} (b) Combined network layers \end{center} \label{tab:twitter}
\end{table}

The goal here is to cluster the MPs into communities based on the
information about their twitter activities. The ground truth communities
are known to be consisting of five communities corresponding to the political
affiliations of the MPs: 152 Conservative, 178 Labour, 39 Liberal Democrat, 5 SNP and 7 Other MPs.
The clustering quality is assessed through
%Normalized Mutual Information
NMI and CCR as before.

Part (a) of Table \ref{tab:twitter} reports the performance of the algorithm for the
six individual layers considered. Note that the performance of the
derived networks is  worse compared to the direct ones
despite being denser. Clearly the signal in favor of the ground
truth is stronger in the ``direct networks'' compared to the ``derived
networks''. The performance
of majority vote, MLEs in aggregate SBM, MLSBM and RMLSBM
%the saturated and
%the restricted multi-layer models
on multi-layer
networks constructed from the three direct layers
and all layers together are given in part (b) of Table \ref{tab:twitter}. In
both cases the multi-layer methods outperform the baseline methods,
and between the two multi-layer methods,  RMLE outperforms MLE. From
the results for direct networks, we note that
%as shown in the simulations,
the performance of multi-layer methods is not affected by inclusion of
relatively sparse layers (``mentions", ``retweets") and multi-layer methods perform better than the densest layer (``follows"),
as long as all the signal strength is high. However the performance
deteriorates as the signal quality becomes bad with the inclusion
of poor performing derived networks.
%The restricted model
RMLSBM is more
robust towards such layers with poor signal compared to MLSBM. The MLE in aggregate SBM performs poorly in the full network due to the number of layers in that network being too large.
%the saturated model.

\section{Discussions}

In this paper we extended the stochastic block model to the multi-layer
settings with two related models, MLSBM and its restricted version RMLSBM. The community assignments through maximum likelihood estimation in both models
are consistent under data generated from the more general model MLSBM with suitable conditions on the growth rate of the number of communities, the number of types of layers,
and the total number of edges of the multi-layer graph. We also derived minimax rates of error and sharp thresholds for consistency of community detection in MLSBM, RMLSBM and a baseline model, the SBM obtained by aggregating the layers. We compared the proposed methods with the MLEs in single layer networks as well as two baseline methods, MLE in the aggregate SBM and majority voting, through results on asymptotic consistency and simulation.

We demonstrate advantages of the MLE in RMLSBM over the MLEs from single-layer SBMs as well as the majority voting and the MLE in MLSBM, both in the asymptotic consistency analysis and the simulation studies, when either the number of communities is large or the graph layers are relatively sparse. This includes the case when the individual layers have bounded average degree, which is an extremely challenging case for single layer networks. We would like to emphasize that handling the bounded degree case would not be possible with the usual MLSBM extension. Both the baseline methods suffer from deficiencies that limit their abilities to detect communities in multi-layer networks effectively. While the aggregation of graphs performs poorly if the community structure information contained in different layers are heterogeneous, the majority voting fails to infer community structure correctly from a large number of layers with week signals. The observations of this paper are in line with previous work in regression settings where a parsimonious model with similar accuracy is preferred over a model with a large number of parameters. The RMLSBM approximates the MLSBM quite well with fewer parameters for most multi-layer networks which are sparse or have a large number of communities. Hence in such cases the RMLSBM outperforms the MLSBM.

%section{Appendix}
\begin{center}
APPENDIX A
\end{center}

\section*{Derivation of variational inference for RMLSBM}
We  derive the update rules for RMLSBM. Note that for the restricted model, the complete
data log likelihood is given by
\begin{eqnarray*}
l(A,Z) &=& l(A|Z)+l(Z) \\
&=& \sum_{i}\sum_{q}Z_{iq}\alpha_{q}+\frac{1}{2}\sum_{i\neq j}\sum_{q,l}\sum_{m}Z_{iq}Z_{jl}\{A_{ij}^{(m)}(\hat{\pi}_{ql}+\hat{\beta}_{m})\\
&& -\log(1+\exp(\hat{\pi}_{ql}+\hat{\beta}_{m})\}.
\end{eqnarray*}
The likelihood of the observed data can be obtained by summing the
complete data likelihood over all possible values of the unobserved
missing class assignment labels $Z$. However, note that the number of
all possible assignments grows exponentially as $K^{N}$, and
the sum quickly becomes computationally intractable even for moderate
$N$. Hence instead we use the EM algorithm for mixture models, where
the unobserved class assignments are treated as missing values. However
one needs to compute the conditional distribution of the missing values
(class assignments here) given the observed data, i.e., $P(Z|A)$. Unfortunately,
as argued by \citet{dpr08}, $P(Z|A)$ is itself intractable, since
the probability of the latent class assignments of a node depends
not only on the observed edges connected to that node, but also on
the connectivity pattern of the whole network.

The variational
approximation concentrates the search for optimal class assignments
to a smaller set by assuming that the class assignments follow a multinomial
distribution with parameters known as variational parameters. It aims
at maximizing an expression containing the log likelihood and the
negative of the Kullback-Liebler (KL) divergence between the true
probability distribution of $P(Z|A)$ and its variational approximation $R_{A}(\cdot)$.
If the approximation to the distribution coincides with the distribution,
then the KL divergence is zero and the variational approximation is
the same as the regular EM. So the new objective function to be optimized
as a lower bound of $l(A)$ is
\[
J(R_{A})=\log l(A)-KL[R_{A}(\cdot),\ P(\cdot|A)].
\]
Here we constraint $R_{A}$ to have
the following form of the product of multinomial densities
\[
R_{A}(Z)=\prod_{i}\prod_{q}\tau_{iq}^{Z_{iq}}.
\]
The variational distribution $R_{A}(Z)$ has the interpretation of
being an approximation of $P(Z|A)$.
%\end{multicols}

\begin{figure}[h]
%\begin{subfigure}{0.48\textwidth}
\begin{algorithm}[H]
\While{either convergence criterion on parameters not met
 or $t<t_{max}$}{
\tcp {E-step: Compute variational estimates $\tau=\{\tau_{iq}\}$}

\While{either convergence criteria on $\tau$ are not met or $s<s_{max}$}{

\For{$i\leftarrow\{1,2,\ldots,N\}$}{

\For{$q\leftarrow\{1,2,\ldots,K\}$}{
$\hat{\tau}_{iq}^{(s+1)}=\exp[\hat{\alpha}_{q}^{(t)}\underset{i<j}{\sum}\underset{l}{\sum}\underset{m}{\sum}\hat{\tau}_{jl}^{(s)}\{A_{ij}^{(m)}\hat{\pi}_{qlm}^{(t)}+(1-A_{ij}^{(m)})(1-\hat{\pi}_{qlm}^{(t)})\}]$

$s=s+1$

}}}

$\hat{\tau}_{iq}^{(t+1)}=\hat{\tau}_{iq}^{(t+1)}/\overset{K}{\underset{q=1}{\sum}}\hat{\tau}_{iq}^{(t+1)}$

\tcp {M-step: Estimate the parameters}

\For{$q\leftarrow1$ to $K$}{

$\hat{\alpha}_{q}^{(t+1)}=\frac{1}{N}\overset{N}{\underset{i=1}{\sum}}\hat{\tau}_{iq}^{(t+1)}$

\For{$m\leftarrow1$ to $M$}{

\For{$l\leftarrow1$ to $K$}{

$\hat{\pi}_{qlm}^{(t+1)}=\frac{\underset{i<j}{\sum}\hat{\tau}_{iq}^{(t+1)}\hat{\tau}_{jl}^{(t+1)}A_{ij}^{(m)}}{\underset{i<j}{\sum}\hat{\tau}_{iq}^{(t+1)}\hat{\tau}_{jl}^{(t+1)}}$

} } }

$t=t+1$

}
\caption{Variational EM algorithm for MLSBM}
\end{algorithm}
%\end{subfigure}%
%\hspace*{0.6cm}
%\begin{subfigure}{0.48\textwidth}
\end{figure}

\begin{figure}[h]
\begin{algorithm}[H]

\While{either convergence criteria on parameters are not met or $t<t_{max}$}{

\tcp{E-Step: Compute variational estimates $\tau=\{\tau_{iq}\}$}

\While{either convergence criteria on $\tau$ are not met or $s<s_{max}$}{

\For{$i\leftarrow\{1,2,\ldots,N\}$}{

\For{$q\leftarrow\{1,2,\ldots,K\}$}{

$\hat{\tau}_{iq}^{(s+1)}=\exp[\hat{\alpha}_{q}^{(t)}\underset{i<j}{\sum}\underset{l}{\sum}\underset{m}{\sum}\hat{\tau}_{jl}^{(s)}\{A_{ij}^{(m)}(\hat{\pi}_{ql}^{(t)}+\hat{\beta}_{m}^{(t)})-\log(1+\exp(\hat{\pi}_{ql}^{(t)}+\hat{\beta}_{m}^{(t)}))\}]$

$s=s+1$

}}}

\tcp{Normalize the variational estimates so that they sum to 1 for each
$i$}

$\hat{\tau}_{iq}^{(t+1)}=\hat{\tau}_{iq}^{(t+1)}/\overset{K}{\underset{q=1}{\sum}}\hat{\tau}_{iq}^{(t+1)}$

\tcp{M-step: Estimate the parameters}

\For{$q\leftarrow1$ to $K$}{

 $\hat{\alpha}_{q}^{(t+1)}=\frac{1}{N}\overset{N}{\underset{i=1}{\sum}}\hat{\tau}_{iq}^{(t+1)}$

}

\tcp{ Use BFGS optimization method to find the parameters}

$(\hat{\pi}^{(t+1)},\hat{\beta}^{(t+1)})=\underset{\pi,\beta}{\arg\max}J(\pi,\beta)$

$t=t+1$

}
\caption{Variational EM algorithm for RMLSBM}
\end{algorithm}
%\end{subfigure}

\end{figure}

%\begin{multicols}{2}

%\end{multicols}

In the E step of the following
variational EM algorithm, we compute the variational approximation
estimates of the probabilities of class assignments for each node.
Given the model parameters $\alpha,$ $\pi$, $\beta$, the variational
parameters $\tau$ can be computed by minimizing the function
\begin{align}
J(R_{A})&=\sum_{i}\sum_{q}\tau_{iq}\log(\alpha_{q})+\frac{1}{2}\sum_{i\neq j}\sum_{q,l}\sum_{m}Z_{iq}Z_{jl}\{A_{ij}^{(m)}(\hat{\pi}_{ql}+\hat{\beta}_{m}) \\
&-\log(1+\exp(\hat{\pi}_{ql}+\hat{\beta}_{m})\}-\sum_{i}\sum_{q}\tau_{iq}\log(\tau_{iq})\nonumber
\end{align}
with the constraint that $\sum_{q}\tau_{iq}=1$ for all $i.$ The
solution for the $(t+1)$th EM step can be readily obtained as
\begin{equation*}
\hat{\tau}_{iq}^{(t+1)}=\exp\Big[\hat{\alpha}_{q}^{(t)}\underset{i<j}{\sum}\underset{l}{\sum}\underset{m}{\sum}\hat{\tau}_{jl}^{(t)}\{A_{ij}^{(m)}(\hat{\pi}_{ql}^{(t)}+\hat{\beta}_{m}^{(t)}) \log(1+\exp(\hat{\pi}_{ql}^{(t)}+\hat{\beta}_{m}^{(t)}))\}\Big].
\end{equation*}

In the M step we estimate the parameters of the model by maximizing
the approximate likelihood. Since we do not have a closed form solution
for the parameters $\pi$ and $\beta$, we use a gradient descent algorithm
(BFGS optimization algorithm) to simultaneously optimize the objective
function with respect to all the parameters. The gradients of the
objective function with respect to $\pi$ and $\beta$ are
\begin{equation}
\frac{\partial}{\partial \beta_{m}^{(t)}}:=\sum_{i\neq j}\sum_{q,l}\hat{\tau}_{iq}^{(t)}\hat{\tau}_{jl}^{(t)}\left(A_{ij}^{(m)}-\frac{\exp(\hat{\pi}_{ql}^{(t)}+\hat{\beta}_{m}^{(t)})}{1+\exp(\hat{\pi}_{ql}^{(t)}+\hat{\beta}_{m}^{(t)})}\right),
\end{equation}
\begin{equation}
\frac{\partial}{\partial \pi_{ql}^{(t)}}:=\sum_{i\neq j}\sum_{m}\hat{\tau}_{iq}^{(t)}\hat{\tau}_{jl}^{(t)}\left(A_{ij}^{(m)}-\frac{\exp(\hat{\pi}_{ql}^{(t)}+\hat{\beta}_{m}^{(t)})}{1+\exp(\hat{\pi}_{ql}^{(t)}+\hat{\beta}_{m}^{(t)})}\right).
\end{equation}
The two algorithms corresponding to the two models are described in Algorithm 1 and Algorithm 2 respectively.

\section*{Proofs of consistency results}

\subsection*{Proof of Equation (3.16)}

\begin{align}
& l^{R}(A;z)-\bar{l}_{P}^{R}(z) \nonumber \\
=& \sum_{m} \sum_{i<j} \Bigg \{A_{ij}^{(m)} \log \left(\frac{\hat{\phi}_{z_i z_j}^{(m)}}{\bar{\phi}_{z_i z_j}^{(m)}}\right)+(1-A_{ij}^{(m)})\log \left(\frac{1-\hat{\phi}_{z_i z_j}^{(m)}}{1-\bar{\phi}_{z_i z_j}^{(m)}}\right)\Bigg \} + X-E(X) \nonumber \\
=& \sum_{m} \sum_{i<j} \Bigg \{A_{ij}^{(m)} (\hat{\pi}_{ql} + \hat{\beta} _{m}-\bar{\pi}_{ql} -\bar{\beta}_{m}) - \log \left(\frac{1+\exp (\hat{\pi}_{ql} + \hat{\beta} _{m}}{ 1+ \exp (\bar{\pi}_{ql} + \bar{\beta} _{m})}\right) \Bigg \} + X-E(X) \nonumber \\
=& \sum_{q \leq l} (\hat{\pi}_{ql} -\bar{\pi}_{ql}) \sum_{m} \sum_{i<j} A_{ij}^{(m)} 1\{z_i=q,z_j =l\} + \sum_{m} ( \hat{\beta} _{m}-\bar{\beta}_{m}) \sum_{i<j} A_{ij}^{(m)} \nonumber \\
&-\sum_{m} \sum_{q \leq l}n_{ql} \log \left(\frac{1+\exp (\hat{\pi}_{ql} + \hat{\beta} _{m}}{ 1+ \exp (\bar{\pi}_{ql} + \bar{\beta} _{m})}\right)  + X-E(X)  \nonumber \\
=& \sum_{q \leq l} (\hat{\pi}_{ql} -\bar{\pi}_{ql}) n_{ql} \sum_{m} \hat{\phi}_{(z)ql}^{(m)}+ \sum_{m} ( \hat{\beta} _{m}-\bar{\beta}_{m}) \sum_{q \leq l} n_{ql} \hat{\phi}_{(z)ql}^{(m)}  \nonumber \\
&-\sum_{m} \sum_{q \leq l}n_{ql} \log \left(\frac{1+\exp (\hat{\pi}_{ql} + \hat{\beta} _{m})}{ 1+ \exp (\bar{\pi}_{ql} + \bar{\beta} _{m})}\right)  + X-E(X) \nonumber \\
=& \sum_{m} \sum_{q \leq l} n_{ql} \Bigg \{ \hat{\phi}_{(z)ql}^{(m)} \log \left(\frac{\hat{\phi}_{(z)ql}^{(m)}}{\bar{\phi}_{(z)ql}^{(m)}}\right)+(1-\hat{\phi}_{(z)ql}^{m})\log \left(\frac{1-\hat{\phi}_{(z)ql}^{(m)}}{1-\bar{\phi}_{(z)ql}^{(m)}}\right)\Bigg \} + X-E(X) \nonumber \\
=& \sum_{m} \sum_{q\leq l} n_{ql} D\left(\hat{\phi}_{(z)ql}^{(m)}\ ||\ \bar{\phi}_{(z)ql}^{(m)}\right)+X-E(X),
\end{align}

\subsection*{Proofs of main results}
Before we describe the proves of Theorems 1 and 2, we need the following lemma.

\begin{lem}
\label{thm:sizeofset}
For a fixed $z$, let $\hat{\pi }_{(z)}=\{\hat{\pi}_{(z)ql}^{(m)};\ q,l\in\{1,\ldots,K\},\ m\in\{1,\ldots,M\}\}$ denote the MLE of the parameters of MLSBM, and let $\hat{\pi}^{R}_{(z)}=\{(\hat{\pi}_{(z)ql},\hat{\beta}_{(z)m});\ q\leq l,\ q,l\in\{1,\dots,K\},\ m\in\{1,\ldots,M\}\}$ be the MLE of the parameters of RMLSBM. Then for any $z$, we have the size of the set of all possible
values that $\hat{\pi}_{(z)}$ can take as
\[
|\hat{\Pi }_{(z)}| \leq \left(\frac{N}{K}+1\right)^{MK(K+1)},
\]
and that
 $\hat{\pi}^{R}_{(z)}$
can take as
\[
|\hat{\Pi}^{R}_{(z)}| \leq \left(M^{1/2}\left(\frac{N}{K}+1\right)\right)^{K^{2}+K}\left(\frac{N(N+1)}{2}+1\right)^{M},
\]
where $\hat{\Pi}_{(z)}$ and $\hat{\Pi}^{R}_{(z)}$ denote the range of $\hat{\pi}_{(z)}$ and $\hat{\pi}^{R}_{(z)}$ respectively for a fixed $z$.
\end{lem}

\begin{proof}

We first determine the size of the set of all possible values that the MLE of the parameter array $\pi$ can take in the MLSBM. Notice
that from Equation (3.5) the estimate $\hat{\pi}^{(m)}$ of the parameter matrix for
any layer $m$ can take any of the $\prod_{q\leq l}(n_{ql}+1)$
values, since its $K(K+1)/2$ upper diagonal components ($\hat{\pi}_{ql}^{(m)}, \, q\leq l, \, q,l\in\{1,\ldots,K\}$)
can take any of the $n_{ql}+1$ values in the set $\{0,1/n_{ql},\ldots ,1\}$ independently.
Hence, $|\hat{\Pi}|=\underset{m}{\prod}\underset{q\leq l}{\prod}(n_{ql}+1)$.
However this is subject to the constraint that $\underset{q\leq l}{\sum}n_{ql}=\begin{pmatrix}N\\
2
\end{pmatrix}$. This implies that $|\hat{\Pi}|$ is a product of $\begin{pmatrix}K+1\\
2
\end{pmatrix}$ positive terms whose sum is fixed. So $|\hat{\Pi}|$ is maximized
when the terms are all equal, i.e., $n_{ql}=\begin{pmatrix}N\\
2
\end{pmatrix}\Big/\begin{pmatrix}K+1\\
2
\end{pmatrix}$ uniformly across all $m$. Hence we have the following inequality
\begin{align*}
|\hat{\Pi}| &\leq\left(\begin{pmatrix}N\\
2
\end{pmatrix}\Big/\begin{pmatrix}K+1\\
2
\end{pmatrix}+1\right)^{MK(K+1)/2}
\\
&<\left(\frac{N^{2}}{K^{2}}+1\right)^{MK(K+1)/2}
<\left(\frac{N}{K}+1\right)^{MK(K+1)}.
\end{align*}

Now we turn our attention to the set of values the MLE of the parameter array in RMLSBM can take.
Note that Equations (3.13) and (3.14) together represent ${K(K+1)}/{2}+M$ equations involving partial
sums of the MLEs of the ${K(K+1)}/{2}+M$ elements in the parameter array $\pi^{R}$ (although the equations are written in terms of the transformation $\phi$ for convenience, they actually represent the same equations as Equations (3.10) and (3.11). The right hand side of the equations together are the sufficient statistics under the RMLSBM. Note that due to the identifiablility constraint, we have only ${K(K+1)}/{2}+M-1$ free parameters. On the other hand, one of the equations in the set of equations is also redundant, since adding together the first $M$ equations represented by Equation (3.13) and adding the remaining $K(K+1)/2$ equations represented by Equation (3.14) yield the same equation and hence there is one linear dependence. This set
of equations determines the MLE of ${\pi}^{R}$. Hence the size of the set of all distinct solutions $\hat{\pi}^{R}$
is at most the number of possible sets of system of equations. To determine the later, we notice that the right hand side
of each of the first set of $M$ equations can take ${N(N+1)}/{2}+1$
values from the set $\{0,2/[N(N+1)], \ldots, 1\}$, while the right hand
side of each of the next set of ${K(K+1)}/{2}$ equations can
take $Mn_{ql}+1$ values from the set $\{0,1/(Mn_{ql}),\ldots,1\}$. So the size
of the set of possible values the estimated parameter array $\hat{\pi}^{R}$ can take is
\[
|\hat{\Pi}^{R}| \leq \prod_{q\leq l}(Mn_{ql}+1)\prod_{m=1}^{M}\left(\frac{N(N+1)}{2}+1\right).
\]
The first term is maximized as before when all the $n_{ql}$'s are equal,
i.e., $n_{ql}=\begin{pmatrix}N\\
2
\end{pmatrix}\Big/\begin{pmatrix}K+1\\
2
\end{pmatrix}$. The second term is a fixed quantity. So we have
\begin{align*}
|\hat{\Pi}^{R}| &\leq\left(M\begin{pmatrix}N\\
2
\end{pmatrix}\Big/\begin{pmatrix}K+1\\
2
\end{pmatrix}+1\right)^{K(K+1)/2}\left(\frac{N(N+1)}{2}+1\right)^{M}
\\
& \leq\left(M\frac{N^{2}}{K^{2}}+1\right)^{K(K+1)/2}\left(\frac{N(N+1)}{2}+1\right)^{M}\\
& \leq\left(M^{1/2}\frac{N}{K}+1\right)^{K(K+1)}\left(\frac{N(N+1)}{2}+1\right)^{M}.
\end{align*}
Lastly notice that the transformation defined by Equation (3.1) is an onto function but not necessarily one-to-one, so one or more parameter arrays $\pi^{R}$ map to one $\phi$. Hence for every estimate $\hat{\phi}$ there exists a corresponding estimate array $\hat{\pi}^{R}$. Therefore we have
\[
|\hat{\Phi}|\leq |\hat{\Pi}^{R}|\leq\left(M^{1/2}\frac{N}{K}+1\right)^{K(K+1)}\left(\frac{N(N+1)}{2}+1\right)^{M}.
\]

\end{proof}

For brevity of notation henceforth we remove the subscript $(z)$ from $\pi_{(z)}$, $\pi^{R}_{(z)}$ and $\phi_{(z)}$, denoting the set of parameters of MLSBM, RMLSBM and the transformation of the set of parameters of RMLSBM respectively for a fixed $z$. We also remove the subscript $(z)$ from $\hat{\Pi}_{(z)}$ and $\hat{\Pi}^R_{(z)}$.

\subsubsection*{Proof of Theorem 1}
The proof for the unrestricted case follows  the structure of the proof of Theorem 1 in \citet{cwa12}.
Following the arguments in the aforementioned paper, we first notice that for a fixed $z$, each estimate ${\hat{\pi}}_{ql}^{(m)}$ is a sum of $n_{ql}$ independent Bernoulli random variables with mean ${\bar{\pi}}_{ql}^{(m)}$. Hence the probability that ${\hat{\pi}}_{ql}^{(m)}=\nu $, where $\nu \in \{0,1/n_{ql},\ldots ,1\}$ can be bounded as
\[
P(\hat{\pi}_{ql}^{(m)}=\nu)\leq \exp \left(-n_{ql}D(\nu\ ||\ \bar{\pi}{}_{ql}^{(m)})\right),
\]
and by the independence of $A_{ij}^{(m)}$, the bound on the probability of any realization $\hat{\pi} $ is
\[
P(\hat{\pi})\leq \exp \left(-\sum_{q\leq l}n_{ql}\sum_{m}D(\hat{\pi}_{ql}^{(m)}\ ||\ \bar{\pi}{}_{ql}^{(m)})\right).
\]

Recall $\hat{\Pi}$ denotes the set of values the estimate array
$\hat{\pi}$ can take for a fixed class assignment $z$. In Lemma \ref{thm:sizeofset}, we have bounded the size of this set as $|\hat{\Pi}|\leq \left(\frac{N}{K}+1\right)^{MK(K+1)}$.
Now we consider the event that
$\sum_{q\leq l}n_{ql}\sum_{m}D(\hat{\pi}_{ql}^{(m)}\ ||\ \bar{\pi}{}_{ql}^{(m)})$
is at least as large as some $\epsilon>0$, and derive an upper bound
for its probability of occurrence:
\begin{align*}
P(\hat{\Pi}_{\epsilon}) &=P\left(\hat{\pi}_{}\in\hat{\Pi};\ \sum_{q\leq l}n_{ql}\sum_{m}D(\hat{\pi}_{ql}^{(m)}\ ||\ \bar{\pi}{}_{ql}^{(m)})\geq\epsilon\right)=\underset{\hat{\pi}\in\hat{\Pi}_{\epsilon}}{\sum}P(\hat{\pi})\\
&\leq\underset{\hat{\pi}\in\hat{\Pi}_{\epsilon}}{\sum} \exp \left(-\sum_{q\leq l}n_{ql}\sum_{m}D(\hat{\pi}_{ql}^{(m)}\ ||\ \bar{\pi}{}_{ql}^{(m)})\right)\leq\underset{\hat{\pi}\in\hat{\Pi}_{\epsilon}}{\sum} \exp(-\epsilon)\\
& =|\hat{\Pi}_{\epsilon}| \exp(-\epsilon)\leq|\hat{\Pi}| \exp(-\epsilon)\leq\left(\frac{N}{K}+1\right)^{MK(K+1)} \exp(-\epsilon)
\end{align*}
Hence for all $\epsilon >0 $, we have over all $K^{N} $ possible class assignments $z$,
\[
P\left(\underset{z}{\max}\sum_{q\leq l}n_{ql}\sum_{m}D(\hat{\pi}_{ql}^{(m)}\ ||\ \bar{\pi}{}_{ql}^{(m)}) \geq \epsilon\right) \leq  P\left(\bigcup_z \left\{\sum_{q\leq l}n_{ql}\sum_{m}D(\hat{\pi}_{ql}^{(m)}\ ||\ \bar{\pi}_{ql}^{(m)}) \geq \epsilon\right\}\right)
\]
\[
\leq K^{N} \exp\left(MK(K+1) \log \left(\frac{N}{K}+1\right)-\epsilon\right)
\leq \exp \left(N\mbox{log}K+M(K^{2}+K) \log \left(\frac{N}{K}+1\right)-\epsilon\right).
\]

The proof for the restricted case, although follows the same structure as before, is more involved as we need to deal with estimating equations instead of closed form solutions. Note that for a fixed $z$, the left hand side of each of the $M$ estimating equations in (3.13) is $\frac{1}{N(N+1)/2} \sum_{q\leq l} n_{ql}{\hat{\phi}}_{ql}^{(m)}$, which is a sum of ${N(N+1)}/{2}$ independent Bernoulli random variables with mean $\frac{1}{N(N+1)/2}\sum_{q\leq l}n_{ql}{\bar{\phi}}_{ql}^{(m)}$ respectively. Hence the probability that\\
$\frac{1}{N(N+1)/2}\sum_{q\leq l} n_{ql}{\hat{\phi}}_{ql}^{(m)}=\nu_{m}$, where $\nu_{m} \in \{0,2/[N(N+1)], \ldots, 1\}$ can be bounded as
\[
P\left(\frac{\sum_{q\leq l} n_{ql} \hat{\phi}_{ql}^{(m)}}{N(N+1)/2}=\nu_{m}\right)
\leq \exp\left(-\frac{N(N+1)}{2} D\left(\nu_{m}\ \Big|\Big|\ \frac{\sum_{q\leq l}n_{ql}\bar{\phi}{}_{ql}^{(m)}}{N(N+1)/2}\right)\right),
\]
for $m\in \{1,\ldots, M\}$.

Similarly the left hand side of each of the ${K(K+1)}/{2} $ estimating equations in (3.14) is $\frac{1}{M} \sum_{m} {\hat{\phi}}_{ql}^{(m)}$, which is a sum of
$Mn_{ql}$ independent Bernoulli random variables with mean $\frac{1}{M} \sum_{m} {\bar{\phi}}_{ql}^{(m)}$. Hence the probability that
$\frac{1}{M}\sum_{m} {\hat{\phi}}_{ql}^{(m)}=\nu_{ql}$, where $\nu_{ql} \in \{0,1/(Mn_{ql}),\ldots,1\}$  can be bounded as
\[
P\left(\frac{1}{M}\sum_{m}\hat{\phi}_{ql}^{(m)}=\nu_{ql}\right)\leq \exp\left(-Mn_{ql}D\left(\nu_{ql}\ \Big|\Big|\ \frac{1}{M} \sum_{m}\bar{\phi}{}_{ql}^{(m)}\right)\right),
\]
for $q\leq l$, $q,l \in \{1,\ldots, K\}$.

Now since these ${K(K+1)}/{2}+M$ estimating equations together determine the MLE $\hat{\pi}^{R}$ of RMLSBM, the probability of any realization of $\hat{\pi}^{R}$ is bounded by the joint probability of the occurrence of the estimating equations. Note that although the equations within the
two sets (3.13) and (3.14) are independent of each other, the two sets of equations are not independent of each other. Hence because of the inequalities that $P(A \cap B) \leq P(A) $ and $P(A \cap B) \leq P(B) $, we have
\begin{align}
\nonumber
P(\hat{\pi}^{R}) &\leq \prod_{m} P\left(\frac{1}{N(N+1)/2}\sum_{q\leq l}n_{ql}\hat{\phi}_{ql}^{(m)}\right)  \\ 
& \leq \exp\left(-\sum_{m}\frac{N(N+1)}{2}D\left(\frac{\sum_{q\leq l}n_{ql}\hat{\phi}_{ql}^{(m)}}{N(N+1)/2}\ \Big|\Big|\ \frac{\sum_{q\leq l}n_{ql}\bar{\phi}{}_{ql}^{(m)}}{N(N+1)/2}\right)\right), 
\label{E1}
\end{align}
and
\begin{align}
\nonumber
P(\hat{\pi}^{R}) &\leq  \prod_{q\leq l} P\left(\frac{1}{M}\sum_{m}\hat{\phi}_{ql}^{(m)}\right) \\ 
& \leq \exp\left(-\sum_{q \leq l}Mn_{ql}D\left(\frac{1}{M}\sum_{m}\hat{\phi}_{ql}^{(m)}\ \Big|\Big|\ \frac{1}{M}\sum_{m}\bar{\phi}{}_{ql}^{(m)}\right) \right). 
\label{E2}
\end{align}

For brevity, we call the right hand sides of
%large expressions inside $\exp $ of
Equations (\ref{E1}) and (\ref{E2}) as $\exp(-E_1)$ and $\exp(-E_2)$ respectively.
From Lemma 1, we have the size of set of all possible values $\hat{\pi}^{R}$ can take
\[
|\hat{\Pi}^{R}|\leq\left(M^{1/2}\frac{N}{K}+1\right)^{K(K+1)}\left(\frac{N(N+1)}{2}+1\right)^{M}.
\]
Now we consider the event that $E_i$
is at least as large as some $\epsilon>0$ for $i=1,2$ respectively.
\begin{align*}
&P(\hat{\Pi}_{\epsilon}^{R})=P(\hat{\pi}^{R}\in\hat{\Pi}^{R};E_i \geq \epsilon)=\underset{\hat{\pi}^{R}\in\hat{\Pi}_{\epsilon}^{R}}{\sum}P(\hat{\pi}^{R})\leq\underset{\hat{\pi}^{R}\in\hat{\Pi}_{\epsilon}^{R}}{\sum} \exp(-E_i) \\
&\leq|\hat{\Pi}^{R}|\exp(-\epsilon)\leq\left(M^{1/2}\frac{N}{K}+1\right)^{K(K+1)}\left(\frac{N(N+1)}{2}+1\right)^{M} \exp(-\epsilon).
\end{align*}
Hence for all $\epsilon >0 $, we have over all $K^{N} $ possible class assignments $z$,
\begin{align*}
&P\Bigg(\underset{z}{\max} \Bigg\{ \sum_{m}\frac{N(N+1)}{2}D\left(\frac{\sum_{q\leq l}n_{ql}\hat{\phi}_{ql}^{(m)}}{N(N+1)/2}\ \Big|\Big|\ \frac{\sum_{q\leq l}n_{ql}\bar{\phi}_{ql}^{(m)}}{N(N+1)/2}\right) \Bigg\} \geq \epsilon\Bigg) \\
 & \leq \exp \left(N\mbox{log}K+(K^{2}+K)\log\left(M^{1/2}\frac{N}{K}+1\right)+M\log \left(\frac{N(N+1)}{2}+1\right)-\epsilon \right),
\end{align*}
and
\begin{align*}
&P\Bigg(\underset{z}{\max} \Bigg\{\sum_{q \leq l}Mn_{ql}D\left(\frac{1}{M}\sum_{m}\hat{\phi}_{ql}^{(m)}\ \Big|\Big|\ \frac{1}{M}\sum_{m}\bar{\phi}{}_{ql}^{(m)}\right)\Bigg\}\geq \epsilon\Bigg)  \\
 & \leq \exp \left( N\log K+(K^{2}+K)\log\left(M^{1/2}\frac{N}{K}+1\right)+M\log \left(\frac{N(N+1)}{2}+1\right)-\epsilon \right).
\end{align*}

\subsubsection*{Proof of Theorem 2}
First we note that $X$, as defined in Equation (3.9),
is a sum of bounded independent random variables, because
each element $X_{ij}^{(m)}$ in the sum is bounded by $C=2 \log (\sqrt{M}N)$ in
absolute value. So we can use a Bernstein type inequality for
sums of bounded independent random variables \citep{cl06} to obtain
\begin{align*}
P(|X-E(X)|>\epsilon) &\leq\exp\left(-\frac{\epsilon^{2}}{2\underset{m}{\sum}\underset{i<j}{\sum}E[X_{ij}^{(m)2}]+\frac{2}{3}\epsilon C}\right)\\
&\leq\exp\left(-\frac{\epsilon^{2}}{8L\log^{2}(\sqrt{M} N)+\frac{4}{3}\epsilon\log(\sqrt{M} N)}\right),
\end{align*}
since $\underset{m}{\sum}\underset{i<j}{\sum}E[X_{ij}^{(m)2}]=\underset{m}{\sum}\underset{i<j}{\sum}P_{ij}^{(m)}\log^{2}(\bar{\pi}_{ql}^{(m)}/(1-\bar{\pi}_{ql}^{(m)}))<4L\log^{2} (\sqrt{M} N)$.
Combining this inequality with the result in Theorem 1,
we have over all possible $K^{N}$ class assignments $z$,
\begin{align*}
&\underset{z}{\max}P(|l(A;z)-\bar{l}_{P}(z)|>2\epsilon L)\\
&\leq\underset{z}{\max}\left(P\left(\sum_{q \leq l}n_{ql}\sum_{m}D(\hat{\pi}_{ql}^{(m)}\ ||\ \bar{\pi}_{ql}^{(m)})>\epsilon L\right)+P\left(|X-E(X)|>\epsilon L\right)\right)\\
&\leq\exp\left(N\mbox{log}K+M(K^{2}+K)\mbox{log}\left(\frac{N}{K}+1\right)-\epsilon L\right)\\
&+\exp\left(N\log K-\frac{\epsilon^{2}L}{8\log^{2}(\sqrt{M} N)+\frac{4}{3} \epsilon \log(\sqrt{M} N)}\right),
\end{align*}
which goes to zero asymptotically as $N$ grows under the growth conditions
mentioned on $K$ and $L$. So we have
\[
\underset{z}{\mbox{max}}|l(A;z)-\bar{l}_{P}(z)|=o_{P}(L).
\]

\subsubsection*{Proof of Theorem 3}

The proof for the RMLSBM will be a slight modification of the earlier
proof for MLSBM. As before we need to bound the two terms in the decomposition of the difference between maximized likelihood and its expected value defined in Equation (3.16). For that we write the first part in the right hand side of (3.16), which we call $E_3$ here for brevity, in terms of the quantities we have already bounded in Theorem 1. We begin by noticing that, since the Kullback-Liebler divergence $D(a||b)$ is convex,
%in $(a||b)$,
we can use a reverse of Jensen's inequality \citep{s09,bdp01} to write
\[
\sum_{q\leq l}n_{ql}D\left(\hat{\phi}_{ql}^{(m)}\ ||\ \bar{\phi}_{ql}^{(m)}\right)
\leq \frac{N(N+1)}{2} D\left(\frac{\sum_{q\leq l}n_{ql}\hat{\phi}_{ql}^{(m)}}{N(N+1)/2}\ \Big|\Big|\ \frac{\sum_{q\leq l}n_{ql}\bar{\phi}{}_{ql}^{(m)}}{N(N+1)/2}\right)+\log (MN^2),
\]
and
\[
\sum_{m}n_{ql}D\left(\hat{\phi}_{ql}^{(m)}\ ||\ \bar{\phi}_{ql}^{(m)}\right)
\leq Mn_{ql} D\left(\frac{1}{M}\sum_{m}\hat{\phi}_{ql}^{(m)}\ \Big|\Big|\ \frac{1}{M}\sum_{m}\bar{\phi}{}_{ql}^{(m)}\right)+\log (MN^2).
\]
To derive the inequality, we used $- \log ({\hat{\phi}_{ql}^{(m)}}/{\bar{\phi}_{ql}^{(m)}})$ as our convex function of ${\hat{\phi}_{ql}^{(m)}}/{\bar{\phi}_{ql}^{(m)}}$ on the interval $[{1}/{(MN^2)}, 1-{1}/{(MN^2)}]$ to obtain a reverse of the ``log-sum inequality".
Summing the two inequalities over $m$ and $q,l$ respectively, we have
\[
E_3\leq 2 \sum_{m} \frac{N(N+1)}{2} D\left(\frac{\sum_{q\leq l}n_{ql}\hat{\phi}_{ql}^{(m)}}{N(N+1)/2}\ \Big|\Big|\ \frac{\sum_{q\leq l}n_{ql}\bar{\phi}{}_{ql}^{(m)}}{N(N+1)/2}\right) +o(M (\log(\sqrt{M}N))^{1+\delta}),
\]
and
\[
E_3\leq 2 \sum_{q\leq l} Mn_{ql} D\left(\frac{1}{M}\sum_{m}\hat{\phi}_{ql}^{(m)}\ \Big|\Big|\ \frac{1}{M}\sum_{m}\bar{\phi}{}_{ql}^{(m)}\right) +o(K^{2} (\log(\sqrt{M}N))^{1+\delta}).
\]
Hence $E_3$ is bounded by the minimum of the above two upper bounds.
Since the first part in the right hand side of the above two inequalities is bounded by the same quantity,
%in both cases,
we will take the inequality for which the second part is smaller. Under the conditions on the growth of $L$ in the theorem, the minimum of the two second parts is $o(L)$. Consequently,
\begin{align*}
&\underset{z}{\max }P(|l^{R}(A;z)-\bar{l}_{P}^{R}(z)| >2\epsilon L) \\ & \leq \exp \left( N \log K+(K^{2}+K)\log \left(M^{1/2}\frac{N}{K}+1\right)  + M \log \left(\frac{N(N+1)}{2}+1\right) -\epsilon L \right) \\& + \exp\left(N\log K-\frac{\epsilon^{2}L}{8\log ^{2} (\sqrt{M} N)+\frac{4}{3} \epsilon \log N}\right),
\end{align*}
so under the growth conditions mentioned under different
asymptotic settings,
\[
\underset{z}{\max}|l^{R}(A;z)-\bar{l}_{P}^{R}(z)|=o_{P}(L).
\]

\subsubsection*{Proof of Theorem 4}

For MLSBM, if the conclusion $\underset{z}{\mbox{max}}|l(A;z)-\bar{l}_{P}(z)|=o_{P}(L)$
of Theorem 2 holds, the data are generated according to a $K$-class
blockmodel with membership vector $\bar{z}$ and probability matrix
$\bar{\pi}$, and the maximum-likelihood $K$-class blockmodel class assignment estimator is $\hat{z}$, then it is easy to see
\begin{align}
 &\bar{l}_{P}(\bar{z})-\bar{l}_{P}(\hat{z}) \leq  \bar{l}_{P}(\bar{z})-\bar{l}_{P}(\hat{z}) +l(A,\hat{z})-l(A,\bar{z}) \label{eq:misclus} \\
 &\leq |\bar{l}_{P}(\bar{z})-l(A,\bar{z}) | + |\bar{l}_{P}(\hat{z})- l(A,\hat{z})|=o_{P}(L). \nonumber
\end{align}
Note that the terms $\bar{l}_{P}(\bar{z})-\bar{l}_{P}(\hat{z})$ and $l(A,\hat{z})-l(A,\bar{z})$ are positive quantities as mentioned earlier. 

The rest of the proof requires the concepts of partition and refinement as laid out
in \citet{cwa12}. We briefly review the concepts here and apply them to MLSBM and its regularized version RMLSBM. Let $[N]$
denote the set of integers $\{1,2,\ldots,N\}$. Any multi-layer blockmodel
induces a partition of the $M$ upper triangular probability matrices. Formally we define a partition of
$\{P_{ij}^{(m)}\}_{i<j}$ into $U$ subsets $\{S_{1},\ldots,S_{U}\}$
by the following mapping
\[
\Theta:(i,j)_{i\in[N],\ j\in[N],\ i<j}\rightarrow[U].
\]
Note that the partitions induced on all $M$ probability matrices
are the same, since the partition is a function only of the indices
and not of the type of edges. There exists a bijection between the set $[U]$ and the upper triangular part of the parameter matrices of MLSBM, so we can write $\pi_{\Theta(i,j)}=\pi_{z_iz_j}$.

In MLSBM, for a general partition, we define $S_{u}=\{(i,j):\ \Theta(i,j)=u,\ i<j\}$ and  $\bar{\pi}_{u}=|S_{u}|^{-1}\underset{m}{\sum} \underset{\Theta(i,j)=u,i<j}{\sum}P_{ij}^{(m)}$, so that we can define the log likelihood under this partition as
\[
\bar{l}_{P}^{*}(\Theta)=\sum_{m=1}^{M}\sum_{i<j}\{P_{ij}^{(m)}\mbox{log \ensuremath{\bar{\pi}_{\Theta(i,j)}^{(m)}}}+(1-P_{ij}^{(m)})\mbox{log \ensuremath{(1-\bar{\pi}_{\Theta(i,j)}^{(m)}})}\}.
\]
It is easy to see that $\bar{l}_{P}^{*}(\Theta^{z})=\bar{l}_{P}(z)$, where $\Theta^{z}$ is the
partition corresponding to block model assignment
$z$. A refinement $\Theta'$ of partition $\Theta$ further subdivides
the partitions in $\Theta$ into subgroups or sub-partitions so that
$\Theta^{'}(i_{1},j_{1})_{i_{1}<j_{1}}=\Theta^{'}(i_{2},j_{2})_{i_{2}<j_{2}}$
$\Rightarrow\Theta(i_{1},j_{1})_{i_{1}<j_{1}}=\Theta(i_{2},j_{2})_{i_{2}<j_{2}}$. From Lemma A2 of \citet{cwa12}, it can be easily obtained
\[
\bar{l}_{P}^{*}(\Theta)\leq\bar{l}_{P}^{*}(\Theta').
\]

One such refinement is constructed in the following way \citep{cwa12}. We consider a $K$ class MLSBM with membership vector $\bar{z}$ and let $\Theta^{z}$ denote a partition of $\{P_{ij}^{(m)}\}_{i<j}$ for any $z$. Now, for a given membership class under $z$, partition the
corresponding set of nodes into subclasses according to the true class
assignment $\bar{z}$ of each node. Then remove one node from each
of the two largest subclasses so obtained, and group them together
as a pair; continue this pairing process until no more than one nonempty
subclass remains. If pair ($i,j)$ is chosen from the above procedure,
then $z_{i}=z_{j}$ and $\bar{z}_{i}\neq\bar{z}_{j}$. Define $C_{1}$
as the number of $(i,j)$ pairs selected by the above method. Since
at least one of $i$ or $j$ is misclustered, we have $N_{e}(z)/2\leq C_{1}\leq N_{e}(z)$.

Next, for each $C_{1}$ pairs find all other distinct indices $k$
for which condition (3.26) of the theorem is satisfied.
Let $C_{2}$ denote the total number of distinct triples that can
be formed in this manner. For each of the $C_{2}$ such triples $(i,j,k)$, we
remove $P_{ik}$ and $P_{jk}$ from their previous subset assignment under $\Theta^{z}$ and place them in a new distinct two element subset. This partition so created
is a refinement of the original partition $\Theta^{z}$, and we call this refined partition $\Theta^{'z}$. The condition (3.26)
of the theorem implies that for each pair of classes $(q,l)$, there
exists at least one class $c$ that satisfies,
\begin{equation}
D\left(\bar{\pi}{}_{qc}^{(m)}\ \Big|\Big|\ \frac{\bar{\pi}{}_{qc}^{(m)}+\bar{\pi}{}_{lc}^{(m)}}{2}\right)+D\left(\bar{\pi}{}_{lc}^{(m)}\ \Big|\Big|\ \frac{\bar{\pi}_{qc}^{(m)}+\bar{\pi}_{lc}^{(m)}}{2}\right)\geq \frac{LK}{MN^2}.\label{eq:probuse}
\end{equation}
Consequently for any of the $C_{1}$ pairs of nodes under the true partition, we obtain triples at least
as large as the cardinality of the smallest class. Hence $C_{2}$ is at
least as large as $C_{1}s$, where $s$ the size of the smallest class.
Now as per assumption, $s=\Omega({N}/{K})$. Hence we can bound
the difference in the likelihood:
\begin{align*}
\bar{l}_{P}(\bar{z})-\bar{l}_{P}^{*}(\Theta^{'z}) &=\sum_{m}\sum_{i<j}D\left(P_{ij}^{(m)}\ ||\ \pi_{\Theta^{'z}(i,j)}^{(m)}\right)=C_{2}M\Omega\left(\frac{LK}{MN^2}\right) \\
&=C_{1}M\Omega\left(\frac{N}{K}\frac{LK}{MN^2}\right) =\frac{N_{e}(z)}{2}\Omega(L)\frac{MNKL}{KLMN^2}=\frac{N_{e}(z)}{N}\Omega(L).
\end{align*}

Since the above procedure is valid for any class assignment vector $z$, we can apply it for the maximum likelihood estimate $\hat{z}$ as well. Note that $\hat{z}$ induces
partition $\Theta^{\hat{z}}$ of the probability matrices $\{P_{ij}^{(m)}\}_{i<j,\ m=\{1,\ldots,M\}}$
and its refinement $\Theta^{'\hat{z}}$ increases the likelihood,
i.e., $\bar{l}_P^{*}(\Theta^{\hat{z}})\leq\bar{l}_P^{*}(\Theta^{'\hat{z}})$.
Also we have $\bar{l}_P^{*}(\Theta^{\hat{z}})=\bar{l}_{P}(\hat{z})$. Consequently
we have,
\[
\bar{l}_{P}(\bar{z})-\bar{l}_{P}(\hat{z})\geq\bar{l}_{P}(\bar{z})-\bar{l}_{P}^{*}(\Theta^{'\hat{z}}) =\frac{N_{e}(\hat{z})}{N}\Omega(L).
\]
Combining this with the result from Equation (3.25), we have
\[
N_{e}(\hat{z})=o_{P}(N).
\]

\subsubsection*{Proof of Theorem 5}
Before we proceed with the proof we need two lemmas. The first lemma bounds the difference between the maximized expected likelihoods from the unrestricted and the restricted models under the true partition. The second lemma uses this result along with the result of Theorem 3 to bound the difference between the maximized expected likelihood for the restricted model under the RMLE and the maximized expected likelihood for the unrestricted model under the true partition.

\begin{lem}
Under the true partition $\bar{z}$, if any of the five sets of conditions in Theorem 3 on the growth of multi-layer blockmodel parameters holds, then  $\bar{l}_{P}(\bar{z})-\bar{l}_{P}^{R}(\bar{z})=o_{P}(L)$, where $L$ is the expected number of edges in the multi-layer graph under the corresponding set of conditions.
\label{lem2}
\end{lem}

\begin{proof}
For large $N$, subtracting Equation (3.24) from Equation (3.23) we have
\begin{align*}
&\bar{l}_{P}(\bar{z})-\bar{l}_{P}^{R}(\bar{z}) \\
=& \sum_{q \leq l}n_{ql}\sum_{m}D(\bar{\pi}_{ql}^{(m)}||\bar{\phi}_{ql}^{(m)})
\\
\leq& |E_Q| \log (MN^{2})+\left(\frac{MN(N+1)}{2}-|E_Q|\right)C_1\frac{L'}{MN^{2}(\log M)^{1+\delta} (\log N)^{2+\delta}} \\
& \log \left(  \frac{C_1L'/(MN^{2} (\log M)^{1+\delta} (\log N)^{2+\delta})}{1/MN^2} \right)
\end{align*}
\begin{align*}
=& o_P(L') +\frac{C_1L'}{(\log M)^{1+\delta}(\log N)^{2+\delta}} \log \left( \frac{C_1L'}{(\log M)^{1+\delta}(\log N)^{2+\delta}} \right) \\
=& o_P(L')+o_P(L')\log \left( \frac{C_1L'}{(\log M)^{1+\delta}(\log N)^{2+\delta}} \right) \Big/[(\log M)^{1+\delta} (\log N)^{1+\delta}] \\
=& o_P(L')+o_P(L')R \\
=& o_P(L),
\end{align*}
where $C_1$ is a constant and $R=\log \left( \frac{C_1L'}{(\log M)^{1+\delta}(\log N)^{2+\delta}} \right) \Big/[(\log M \log N)^{1+\delta}]$. The inequality in step 2 comes from the upper bound on $D(p||q)$ which can be derived as follows. Without loss of generality, we can assume that $p>q$ and $D(p||q)\leq p \log \frac{p}{q} \leq p_{\max} \log \frac{p_{\max}}{q_{\min}}$. Next we replace $p_{\max}$ and $q_{\min}$ by the assumption on the lower and upper bounds of the restricted block model probabilities given in Equation (3.3).

Now to complete the proof, we only need to verify that under the five sets of conditions in Theorem 3,
%the right most term,
the term $R$ in the right hand side of the above derivation is $o(1)$. Under the first two sets of conditions, $L'=MN (\log N)^{3+\delta}$ and consequently $R=\frac{ \log(MN \log N /(\log M)^{1+\delta})}{(\log M \log N)^{1+\delta}}=o(1)$. Under the third set of conditions, $L'=N (\log N)^{3+\delta}$ and hence $R=\frac{ \log(N\log N/(\log M)^{1+\delta})}{(\log M \log N)^{1+\delta}}=o(1)$. Finally under the last two sets of conditions, if $L'=MN (\log N)^{1+\delta}$ then $R=\frac{ \log(MN/ (\log M)^{1+\delta})}{(\log M \log N)^{1+\delta}}=o(1)$, and if $L'=M (\log M)^{2+\delta} (\log N)^{1+\delta}$ then $R=\frac{ \log(M (\log M)^{1+\delta})}{(\log M \log N)^{1+\delta}}=o(1)$.

\end{proof}

\begin{lem}
Under the true partition $\bar{z}$ and the RMLE of the partition $\hat{z}^{R}$ (i.e., the MLE in the restricted model RMLSBM), we have $\bar{l}_{P}(\bar{z})-\bar{l}_{P}^{R}(\hat{z}^{R})=o_{P}(L)$
whenever the conclusion of Theorem 3 holds.
\label{thm:expllkdiff}
\end{lem}

\begin{proof}
Note that $\bar{l}_{P}(\hat{z}^{R})\geq\bar{l}_{P}^{R}(\hat{z}^{R})$
since the maximum of the unrestricted likelihood $\bar{l}_{P}(z)$ is uniformly larger than or equal to the maximum of the restricted likelihood $\bar{l}_{P}^{R}(z)$
for all $z$. Moreover, $\bar{z}$ maximizes $\bar{l}_{P}(\cdot)$ and
hence $\bar{l}_{P}(\bar{z})-\bar{l}_{P}^{R}(\hat{z}^{R})\geq0$.
Notice that $l^{R}(A,\hat{z}^R)-l^{R}(A,\bar{z})$ is positive since the observed restricted likelihood is maximized at $\hat{z}^{R}$. So we have
%\[\bar{l}_{P}(\bar{z})-\bar{l}_{P}^{R}(\hat{z}^{R})\leq\bar{l}_{P}(\bar{z})-\bar{l}_{P}^{R}(\hat{z}^{R})+l^{R}(A,\hat{z}^{R})-l^{R}(A,\bar{z}).\]
%Here we have added an  an additional positive term to $(\bar{l}_{P}(\bar{z})-\bar{l}_{P}^{R}(\hat{z}^{R})$. The added term $(l^{R}(A,\hat{z}^R)-l^{R}(A,z))$ is positive since the observed restricted likelihood is maximized at $\hat{z}^{R}$. Now we have,
\begin{align*}
\bar{l}_{P}(\bar{z})-\bar{l}_{P}^{R}(\hat{z}^{R}) &\leq\bar{l}_{P}(\bar{z})-\bar{l}_{P}^{R}(\hat{z}^{R})+l^{R}(A,\hat{z}^{R})-l^{R}(A,\bar{z})\\
%&\bar{l}_{P}(\bar{z})-\bar{l}_{P}^{R}(\hat{z}^{R})
&\leq|\bar{l}_{P}(\bar{z})-l^{R}(A,\bar{z})|+|\bar{l}_{P}^{R}(\hat{z}^{R})-l^{R}(A,\hat{z}^{R})| \\
&\leq|\bar{l}_{P}(\bar{z})-\bar{l}_{P}^{R}(\bar{z})|+|\bar{l}_{P}^{R}(\bar{z})-l^{R}(A,\bar{z})|+|\bar{l}_{P}^{R}(\hat{z}^{R})-l^{R}(A;\hat{z}^{R})|\\
&=o_{P}(L),
\end{align*}
by Lemma 2 and Theorem 3.

\end{proof}

Now we are ready to show that the class membership assignment vector estimated through the maximum likelihood estimation in the restricted model RMLSBM is consistent under data generated from the MLSBM.
We define regularized partition $\Theta^{R}$ of the matrices of probabilities between nodes $P_{ij}^{(m)}$,  computed according to the restricted model RMLSBM and its refinement $\Theta^{'R}$ in exactly the same way. We further define the corresponding restricted log likelihood associated with this partition $\Theta^{R}$ as $\bar{l}_{P}^{*R}(\Theta^R)$. For convenience we again resort to the transformation defined by Equation (3.1)
\[
\bar{l}_{P}^{*R}(\Theta^R)=\sum_{m=1}^{M}\sum_{i<j}\{P_{ij}^{(m)}\log \bar{\phi}_{\Theta^{R}(i,j)}^{(m)}+(1-P_{ij}^{(m)}) \log (1-\bar{\phi}_{\Theta^{R}(i,j)}^{(m)})\}.
\]
For any membership assignment $z^{R}$ from the RMLSBM, let  $\bar{l}_{P}^{*R}(\Theta_{z^{R}}^{R})$ be the corresponding partition of $P_{ij}^{(m)}$. It follows from this definition that $\bar{l}_{P}^{*R}(\Theta_{z^{R}}^{R})=\bar{l}_{P}^{R}(z^R)$. Hence we have
\[
\bar{l_{P}}(\bar{z})-\bar{l}_{P}^{*R}(\Theta_{z^{R}}^{R})=\sum_{m}\sum_{i<j}D\left(P_{ij}^{(m)}\ ||\ \bar{\phi}_{\Theta_{z^{R}}^{'R}(i,j)}^{(m)}\right)=C_{2}M\Omega(g)=C_{1}M\Omega\left(\frac{N}{K}g\right)
\]
\[
=\frac{N_{e}(z^{R})}{2}\Omega(L)\frac{MN}{KL}g=\frac{N_{e}(z^{R})}{h}\Omega(L).
\]
Now we specialize to $\hat{z}^{R}$. Since $\Theta^{'R}$ is a refinement of $\Theta^{R}$, it increases the restricted likelihood, i.e., $\bar{l}_{P}^{*R}(\Theta_{\hat{z}^{R}}^{'R}) \geq \bar{l}_{P}^{*R}(\Theta_{\hat{z}^{R}}^{R})$. Using this and the fact that $\bar{l}_{P}^{*R}(\Theta_{\hat{z}^{R}}^R)=\bar{l}_{P}^{R}(\hat{z}^R)$, we have
\[
\bar{l}_{P}(\bar{z})-\bar{l}_{P}^{R}(\hat{z}^{R})\geq\bar{l}_{P}(\bar{z})-\bar{l}_{P}^{*R}(\Theta_{\hat{z}^{R}}^{'R})=\frac{N_{e}(\hat{z}^{R})}{h}\Omega(L).
\]
The left hand side is $o(L)$ by Lemma 3, and hence,
\[
N_{e}(\hat{z}^{R})=o_{P}(h).
\]

\section*{Proofs of minimax and threshold results}

\subsection*{Proof of Theorem 6}
For brevity we mention here only the results and proofs that differ from the proof contained in \citet{zhang15}  and refer the reader to the aforementioned paper for a complete description of the techniques involved. We define the homogeneous/symmetric multi layer stochastic blockmodel as the MLSBM with the parameter space $\Theta^{ML}_1$ that has all intra-block connection probabilities equal to each other as well as all inter-block connection probabilities equal to each other for each layer. As before, we assume no relation among the connection probabilities of one layer with that of another layer. The parameter space can be written as
\begin{align}
\Theta^{ML}_{1} (z,N,K,M,\mathbf{a},\mathbf{b})=&\Bigg \{ (z,\{P_{ij}^{(m)}\}) \in \Theta^{ML} : P_{ij}^{(m)} = \frac{a^{(m)}}{N} \nonumber\\
& \text{if } z_{i}=z_{j} \text{ and } P_{ij}^{(m)} = \frac{b^{(m)}}{N} \, \text{ if } z_{i} \neq z_{j}, \;\forall m \Bigg \}.
\label{ML1}
\end{align}
Note that this model space is homogeneous and uniquely determined by $z$, i.e., given the community assignments $z$, the block model parameters are uniquely determined. This model space is  also closed under permutations, in the sense that the model obtained through permuting the class labels also belong to $\Theta^{ML}_{1}$. We further define a submodel of this where the block sizes are all (almost) same as
\begin{equation}
\Theta^{ML}_{0} (z,N,K,M,\mathbf{a},\mathbf{b})=\Bigg \{ (z,\{P_{ij}^{(m)}\})  \in \Theta^{ML}_{1}(z,N,K,M,\mathbf{a},\mathbf{b}) : N_{q}=(1+o(1))\frac{N}{K}, \; \forall q \Bigg \},
\label{ML0}
\end{equation}
and yet another submodel space of $\Theta^{ML}_{0}$ where the communities are of only 3 sizes, $\left \lfloor{\frac{N}{K}}\right \rfloor$, $\left \lfloor{\frac{N}{K}}\right \rfloor-1$ and $\left \lfloor{\frac{N}{K}}\right \rfloor+1$. This submodel space, denoted as $\Theta^{ML}_{L}$ is the least favorable case for community detection in terms of the size of communities (See Section 5.1 of \citet{zhang15}). The parameter space can be written as
\begin{align}
&\Theta^{ML}_{L} (z,N,K,M,\mathbf{a},\mathbf{b},\mathbf{S})=\Bigg \{ (z,\{P_{ij}^{(m)}\}) \in \Theta^{ML}_{0}(z,N,K,M,\mathbf{a},\mathbf{b}) : \Big |q : N_{q}=\left \lfloor{\frac{N}{K}}\right \rfloor \Big |=S_1 , \nonumber \\
& \Big |q : N_{q}=\left \lfloor{\frac{N}{K}}\right \rfloor+1 \Big|=S_2, \,
\Big |q : N_{q}= \left\lfloor{\frac{N}{K}}\right \rfloor-1 \Big|=S_3, \quad S_1+S_2+S_3=K \Bigg \}.
\label{MLL}
\end{align}
The submodel spaces $\Theta^{ML}_{0} $ and $\Theta^{ML}_{L} $ are also homogeneous and closed under permutation. Let $\hat{z}$ be the class assignment obtained from some procedure under consideration. We break the proof up into two parts, the first one proves a lower bound for the minimax risk and the second one shows that there exists an algorithm which attains the lower bound.

\subsubsection*{Lower bound}
It was argued in Section 5.1 of \citet{zhang15} that $\Theta^{ML}_{1}$ is the least favorable subspace of $\Theta^{ML}$ using the property of being closed under permutation. Hence, a lower bound on the minimax rates established on $\Theta^{ML}_{1}$ will also be a good lower bound for the larger parameter space $\Theta^{ML}$. Since the supremum over a larger space is always greater than the supremum over any of its subspaces, the lower bound on $\Theta^{ML}_{1}$ is a lower bound for the larger space trivially, but being a least favorable subspace makes it match the rate. Throughout this section (proof of lower bound) we assume $K \geq 3$. The proof for the case $K=2$ follows from \citet{zhang15} with the same modifications described below for the $K \geq 3$ case.

We start with a  couple of lemmas. The next lemma due to \citet{zhang15} shows that for any homogeneous parameter space which is closed under permutation (e.g., $\Theta^{ML}_1$ and all its submodels defined above), the minimum global Bayesian risk of $\hat{z}$ under the uniform prior is the same as the minimum of the local Bayesian risk for the first node. The local Bayesian risk for one node needs to be computed under an appropriate local loss function. \citet{zhang15} defined such a local loss function as the average over all possible permutations of $\hat{z}$ that minimizes the distance from the true class assignment. Let $S_z (\hat{z})=\{ \hat{z}^{'}= \delta(\hat{z}) : d_H (z, \hat{z}^{'})= \inf_{\delta} d_H (z, \delta(\hat{z})) \}$. Then the local loss function is defined as
\begin{equation}
r (z_i,\hat{z}_i ) = \frac{1}{|S_z (\hat{z})|} \sum_{\hat{z}' \in S_z (\hat{z})} d_H (z_i, \hat{z}'_i).
\label{localloss}
\end{equation}

\begin{lem} (Lemma 2.1 of \citet{zhang15}) Let $\Lambda$ be any homogeneous parameter space which is closed under permutation and $\tau$ be a uniform prior over the elements of $\Lambda$. Defining the global Bayesian risk as $B_{\tau}(\hat{z})=\frac{1}{|\Lambda|} \sum_{z\in \Lambda} E[r(z,\hat{z})]$ and local Bayesian risk for the first node (under the local loss function) as  $B_{\tau}(\hat{z}_1)=\frac{1}{|\Lambda|} \sum_{z\in \Lambda} E[r(z_1,\hat{z}_1)]$, we have
$$ \inf_{\hat{z}} B_{\tau}(\hat{z}) = \inf_{\hat{z}} B_{\tau}(\hat{z}_1).$$
\end{lem}
 
 Now we have the following lemma on the Bayesian local risk for the first node in the parameter space $\Theta^{ML}_{L}$ under an uniform prior.

\begin{lem}
Let $\hat{z}$ be an estimated class assignment from some procedure in the block model defined by (\ref{MLL}). Let $\tau$ be a uniform prior over all elements in $\Theta^{ML}_{L}$. For the first node, the local Bayesian risk, $B_{\tau}(\hat{z}_{1})=\frac{1}{|\Theta^{ML}_{L}|} \sum_{z\in \Theta^{ML}_{L}} E[r(z_1,\hat{z}_1)]$ is lower bounded as
\begin{equation}
B_{\tau}(\hat{z}_{1}) \geq \epsilon P \left(\sum_{m}c^{(m)}\sum_{i=1}^{\left \lfloor{\frac{N}{K}}\right \rfloor} X^{(m)}_{i} \geq \sum_{m}c^{(m)} \sum_{i=1}^{\left \lfloor{\frac{N}{K}}\right \rfloor} Y^{(m)}_{i} \right),
\label{lbeq}
\end{equation}
where $\epsilon>0$ is a constant, $c^{(m)}=\log \left(\frac{a^{(m)}(1-\frac{b^{(m)}}{N})}{b^{(m)}(1-\frac{a^{(m)}}{N})}\right)$, and $X^{(m)}_i \sim Bern(\frac{b^{(m)}}{N})$ and $Y^{(m)}_i \sim Bern(\frac{a^{(m)}}{N})$ are independent random variables for all $i=\{1,\ldots,\left \lfloor{\frac{N}{K}}\right \rfloor \}$. Moreover if $\frac{N \sum I^{(m)}}{K} \rightarrow \infty$, then the right hand side of Equation (\ref{lbeq}) is greater than or equal to
$$\exp(-(1+o(1)) N\sum_{m}I^{(m)}/K),$$
while if $\frac{N \sum I^{(m)}}{K} =O(1)$, then the right hand side of Equation (\ref{lbeq}) is $O(1)$.
\label{lemmamll}
\end{lem}

\begin{proof}
We follow the proof of Lemma 5.1 in Section 6.2 of \citet{zhang15}. Define $\Theta^{ML}_{L_1}$ as a subset of the parameter space of $\Theta^{ML}_{L}$ such that the class to which the first node belongs to is always of size $\left \lfloor{\frac{N}{K}}\right \rfloor+1$,  i.e., $\Theta^{ML}_{L_1}= \{ (z,P_{ij}^{(m)}) \in \Theta^{ML}_{L} : N_{z_1}=\left \lfloor{\frac{N}{K}}\right \rfloor+1 \}$. Letting $x_2 =(\left \lfloor{\frac{N}{K}}\right \rfloor+1) S_2$ it was shown in Section 6.2 of \citet{zhang15} that the ratio of the cardinality of the set $\Theta^{ML}_{L_1}$ to that of $\Theta^{ML}_{L}$ is a constant, i.e., $|\Theta^{ML}_{L_1}| / |\Theta^{ML}_{L}|=x_2/N \geq \epsilon$ for some $\epsilon>0$.
Consequently,
$$B_{\tau}(\hat{z}_{1}) \geq \frac{1}{|\Theta^{ML}_{L}|} \sum_{z\in \Theta^{ML}_{L_1}} E[r(z_1,\hat{z}_1)] \geq \frac{\epsilon}{|\Theta^{ML}_{L_1}|} \sum_{z\in \Theta^{ML}_{L_1}} E[r(z_1,\hat{z}_1)]. $$ 

For each $z' \in \Theta^{ML}_{L_1}$, we define $k'(z')=z'_1$ as the class to which the first node belongs to. Let $k(z')$ be the set of indices of the communities of size $\left \lfloor{\frac{N}{K}}\right \rfloor$. Since the first community is of size $\left \lfloor{\frac{N}{K}}\right \rfloor+1$, $k'(z')$ does not belong to $k(z')$. Now we define a new assignment $z(z')$ based on $z'$ as follows
\begin{equation}
z(z')_1=\begin{cases}
\min \{k \in k(z') : k>k'(z') \} & \text{ if } \max k(z') > k'(z')  \\
\min k(z') & \text{ if } \max k(z') < k'(z'),
\end{cases}
\end{equation}
and $z(z')_i=z'_i$ for all $i\geq 2$. Clearly $z(z') \in \Theta^{ML}_{L_1}$, differs from $z'$ only in the first node and by definition has a distance 1 from it. Moreover for any two distinct class assignments $z',z'' \in \Theta^{ML}_{L_1}$, $z' \neq z''$, the new assignments based on them $z(z')$ and $z(z'')$ are also different \citep{zhang15}. This implies that $\Theta^{ML}_{L_1} =\{z(z') : z' \in \Theta^{ML}_{L_1} \}$. Consequently,

$$B_{\tau}(\hat{z}_{1}) \geq \frac{\epsilon}{2|\Theta^{ML}_{L_1}|} \sum_{z'\in \Theta^{ML}_{L_1}} (E[r(z'_1,\hat{z}_1)]+E[r(z(z')_1,\hat{z}_1)]). $$  

Next we will derive a lower bound for the Bayes risk, $\inf_{\hat{z}}B_{\tau}(\hat{z}_{1})$. Conditional on $z'$ or $z(z')$, the distribution of $A$ in MLSBM involves a collection of $M$ adjacency matrices. We define two sets $J_0$ and $J_1$ as follows
\begin{align*}
J_0&=\{ i \in \{1,\ldots, N\} \backslash \{1\} : z'_i=z'_1\}, \\
J_1&=\{ i \in \{1,\ldots, N\} \backslash \{1\} : z'_i=z(z')_1\}.
\end{align*}
Hence, 
\begin{equation}
P(A|z')=\prod_{m} \Bigg \{ \prod_{i\in J_0} \left(\frac{a^{(m)}}{N}\right)^{A^{(m)}_{1i}} \left(1-\frac{a^{(m)}}{N}\right)^{1-A^{(m)}_{1i}} \prod_{i\in J_1} \left(\frac{b^{(m)}}{N}\right)^{A^{(m)}_{1i}} \left(1-\frac{b^{(m)}}{N}\right)^{1-A^{(m)}_{1i}}  \Bigg \} f(A^C),
\end{equation}
and
\begin{equation}
P(A|z(z'))=\prod_{m} \Bigg \{ \prod_{i\in J_1} \left(\frac{a^{(m)}}{N}\right)^{A^{(m)}_{1i}} \left(1-\frac{a^{(m)}}{N}\right)^{1-A^{(m)}_{1i}} \prod_{i\in J_0} \left(\frac{b^{(m)}}{N}\right)^{A^{(m)}_{1i}} \left(1-\frac{b^{(m)}}{N}\right)^{1-A^{(m)}_{1i}}  \Bigg \} f(A^C),
\end{equation}
where the function $f(A^C)$ is a function  involving connections from node 1 to nodes not in $J_0 \cup J_1$ and all connections not involving node 1. Let $\hat{z}^{B}$ attains the infimum of the local Bayes risk. Since $d_H(z', z(z'))=1$, the loss with respect to the local loss function defined in Equation (\ref{localloss}) is $r(z'_1,\hat{z}^{B}_1)=d_H (z'_1,\hat{z}^{B}_1)$ which is a 0-1 loss. Then $\hat{z}^{B}_1$ is the Bayes estimator with respect to the local 0-1 loss function and consequently $\hat{z}^{B}_1$ would be the mode of the posterior distribution, i.e.,
\begin{equation}
\hat{z}^{B}_1=\begin{cases}
z'_{1}, & \text{ if } \sum_{m} \sum_{i\in J_0}  c^{(m)} A^{(m)}_{1i} \geq  \sum_{m} \sum_{i\in J_1}  c^{(m)} A^{(m)}_{1i} \\
z(z')_1, & \text{ if }  \sum_{m} \sum_{i\in J_0} c^{(m)} A^{(m)}_{1i} < \sum_{m}  \sum_{i\in J_1} c^{(m)}  A^{(m)}_{1i}.
\end{cases}
\end{equation}
Hence we have,
\begin{equation}
\inf_{\hat{z}}B_{\tau}(\hat{z}_1) \geq \epsilon P\left(\sum_{m} c^{(m)} \sum_{i=1}^{\left \lfloor{\frac{N}{K}}\right \rfloor} X^{(m)}_{i} \geq \sum_{m} c^{(m)} \sum_{i=1}^{\left \lfloor{\frac{N}{K}}\right \rfloor} Y^{(m)}_{i} \right).
\end{equation}

To derive the probability in the above lower bound, let $Z_i=\sum_m Z_{i}^{(m)} :=\sum_{m} c^{(m)} (X^{(m)}_{i}-Y^{(m)}_{i})$. Hence the moment generating function (MGF) of $Z_i$ is, 
\begin{align*}
M_{Z_i}(t)&=\prod_{m} M_{Z_i^{(m)}}(t) =\prod_{m} E(e^{tc^{(m)}X_{i}}) E(e^{-tc^{(m)}Y_{i}}) \\
&=\prod_{m} \left( e^{tc^{(m)}} \frac{b^{(m)}}{N} +1 -\frac{b^{(m)}}{N} \right) \left( e^{-tc^{(m)}} \frac{a^{(m)}}{N} +1 -\frac{a^{(m)}}{N} \right).
\end{align*}
The MGF, $M_{Z_i}(t)$ is minimized at $t^{*}=\frac{1}{2}$ and the minimum value is 
\begin{equation}
M_{Z_i}(t^{*})=\prod_{m} M_{Z_i^{(m)}}(t^{*})=\prod_{m} \left(\sqrt{\frac{a^{(m)}}{N}\frac{b^{(m)}}{N}} +\sqrt{(1-\frac{a^{(m)}}{N})(1-\frac{b^{(m)}}{N})}\right)^2.
\label{mgfI}
\end{equation} 
This implies $-\log (M_{Z_i}(t^{*})) = \sum_{m} I^{(m)}$. Denoting $S_{N'}=\sum_{i=1}^{N'} \sum_{m} Z_{i}^{(m)}$ for $N'=\left \lfloor{ \frac{N}{K}} \right \rfloor$, we obtain for any $\delta>0$,
\begin{equation*}
P(S_{N'} \geq 0) \geq \sum_{N'\delta > S_{N'} \geq 0} \prod_{i=1}^{N'} \prod_{m=1}^{M} p(z_{i}^{(m)})
 \geq \frac{(M_{Z_i}(t^{*}))^{N'}}{\exp(N't^{*}\delta)} \sum_{N'\delta > S_{N'} \geq 0} \prod_{i=1}^{N'} \prod_{m=1}^{M} \frac{\exp(t^{*}z^{(m)}_{i})p(z^{(m)}_{i})}{M_{Z_i^{(m)}}(t^{*})}.
\end{equation*}
Now denoting $q_{m}(w)=\frac{\exp(t^{*}w)p(w)}{M_{Z_i^{(m)}}(t^{*})}$ for all $m$, we have
\begin{equation*}
P(S_{N'} \geq 0)  \geq \exp(-N'\sum_{m} I^{(m)}) \exp(-N't^{*}\delta) \sum_{N'\delta > S_{N'} \geq 0} \prod_{i=1}^{N'} \prod_{m=1}^{M} q_{m}(z^{(m)}_{i}).
\end{equation*}
We note that $q_{m}(w)$ is a probability mass function for all $m \in \{1, \ldots M \}$. Let $\{W^{(m)}_{i} \}, \, i \in \{1, \ldots, N' \}$, be i.i.d random variables with probability mass function $q_{m}(w)$. Then we have,
\begin{equation}
P(S_{N'} \geq 0)  \geq \exp(-N'\sum_{m} I^{(m)}) \exp(-N't^{*}\delta) P(\delta > \frac{1}{N'} \sum_{i=1}^{N'} (\sum_{m} W_i^{(m)}) \geq 0 ).
\label{lcompute}
\end{equation}
Clearly $W_{i}^{(m)}=c^{(m)} (X_i^{(m)} -Y_i^{(m)})$ can take 3 values, $ \pm c^{(m)}$ and $0$. The first two values correspond to the cases when $X_i^{(m)}=1, \, Y_i^{(m)}=0$ and $Y_i^{(m)}=1, \, X_i^{(m)}=0$ respectively. We compute the first probability as $q_{m}(W_{i}^{(m)}=c^{(m)})=\exp(c^{(m)}/2) \left(\frac{b^{(m)}}{N} \right) \left( 1-\frac{a^{(m)}}{N} \right) / M_{Z_{i}^{(m)}}(1/2)$. The second one follows similarly. Hence we have, 
\begin{equation*}
W_{i}^{(m)}=\begin{cases}
c^{(m)} & \text{w.p } \sqrt{\frac{a^{(m)}}{N}\frac{b^{(m)}}{N}(1-\frac{a^{(m)}}{N})(1-\frac{b^{(m)}}{N})} /M_{Z_{i}^{(m)}}(1/2)  \\
-c^{(m)} & \text{w.p } \sqrt{\frac{a^{(m)}}{N}\frac{b^{(m)}}{N}(1-\frac{a^{(m)}}{N})(1-\frac{b^{(m)}}{N})} /M_{Z_{i}^{(m)}}(1/2) \\
0 & \text{w.p } 1-P(W_{i}^{(m)}=c^{(m)})-P(W_{i}^{(m)}=-c^{(m)}).
\end{cases}
\end{equation*}
Hence $E(W_{i}^{(m)})=0$ and $Var(W_{i}^{(m)})=2 (c^{(m)})^2\sqrt{\frac{a^{(m)}}{N}\frac{b^{(m)}}{N}(1-\frac{a^{(m)}}{N})(1-\frac{b^{(m)}}{N})} /M_{Z_{i}^{(m)}}(1/2)$. Hence denoting $\sum_{m} W_i^{(m)}$ as $W_i$ we have, $E(\frac{1}{N'} \sum_{i=1}^{N'} W_i)=0$. Also by independence we the have variance of $\frac{1}{N'} \sum_{i=1}^{N'} W_i$ as $V=\sum_{m} Var(W_{i}^{(m)})/N'=\sum_{m} V^{(m)}$ where $V^{(m)}=Var(W_{i}^{(m)})/N'$. 

We now prove that $\sum_{m}I^{(m)}/ \sqrt{V} \rightarrow \infty$. First we consider the case when $a^{(m)} \asymp b^{(m)}$. Then we have $I^{(m)} \asymp \frac{1}{N}\frac{(a^{(m)}-b^{(m)})^2}{a^{(m)}} $ \citep{zhang15}. On the other hand replacing $N'$ by $ N/K$ we have  $V^{(m)} \asymp \frac{(c^{(m)})^2a^{(m)}}{N}/(N/K) \asymp \frac{(a^{(m)}-b^{(m)})^2K}{a^{(m)}N^2} $  since $c^{(m)} \asymp \frac{a^{(m)}-b^{(m)}}{a^{(m)}}$ and $M_{Z_{i}^{(m)}}(t^{*})= \exp (-I^{(m)}) = O(1)$. Consequently, $\sqrt{V} \asymp \frac{\sqrt{K}}{N} \sqrt{\sum_{m} \frac{(a^{(m)}-b^{(m)})^2}{a^{(m)}} }$. Clearly $\sum_{m}I^{(m)}/ \sqrt{V} \asymp \frac{1}{\sqrt{K}}\sqrt{\sum_{m} \frac{(a^{(m)}-b^{(m)})^2}{a^{(m)}}} \asymp \sqrt{\frac{N \sum_{m} I^{(m)}}{K}}\rightarrow \infty$. Next consider the other case $b^{(m)}=o(a^{(m)})$. Then $\sum_{m}I^{(m)} \asymp \frac{\sum_{m} a^{(m)}}{N}$ and $c^{(m)} \asymp \log (a^{(m)}/b^{(m)})$. Consequently, $V^{(m)} \asymp \frac{a^{(m)}}{N} \left(\log \left(\frac{a^{(m)}}{b^{(m)}}\right)\right)^2\sqrt{\frac{b^{(m)}}{a^{(m)}}}/\frac{N}{K} $. Hence $\sqrt{V} = o(\sqrt{\sum_{m} a^{(m)}K}/N)$. This implies $\sum_{m}I^{(m)}/ \sqrt{V}= \omega(\sqrt{\sum_{m} a^{(m)}/K})$. Since $\sum_{m} a^{(m)}/K \asymp N\sum_{m}I^{(m)}/K \rightarrow \infty$, we have $\sum_{m}I^{(m)}/ \sqrt{V} \rightarrow \infty$.

Then we choose $\delta=(\sum_{m} I^{(m)} \sqrt{\sum_{m} V^{(m)}})^{1/2}$ so that $\delta =o(\sum_{m} I^{(m)})$ and $\sqrt{V}=\sqrt{(\sum_{m} V^{(m)})}=o(\delta)$. Since the ratio of $\delta$ to the square root of  variance goes to infinity as $N$ goes to infinity by the central limit theorem we have, $P(\delta > \frac{1}{N'} \sum_{i=1}^{N'} \sum_{m} W_i^{(m)} \geq 0 ) \rightarrow 1/2$.  Consequently from Equation (\ref{lcompute}),
\begin{align*}
& P(S_{N'} >0) \geq \exp(-(1+o(1))N'\sum_{m}I^{(m)}) \\
\Rightarrow & P\left(\sum_{m} c^{(m)} \sum_{i=1}^{\left \lfloor{\frac{N}{K}}\right \rfloor} X^{(m)}_{i} \geq \sum_{m} c^{(m)} \sum_{i=1}^{\left \lfloor{\frac{N}{K}}\right \rfloor} Y^{(m)}_{i} \right) \geq \exp(-(1+o(1)) \frac{N\sum_{m}I^{(m)}}{K}),
\end{align*}
provided $N\sum_{m}I^{(m)}/K \rightarrow \infty $. The last inequality is obtained by replacing $N'$ by $\left \lfloor{\frac{N}{K}}\right \rfloor$.
If however, $N\sum_{m}I^{(m)}/K =O(1) $, we can choose a $\delta$ so that $N\delta/K$ is also a constant. Then considering the cases $a^{(m)} \asymp b^{(m)}$ and $b^{(m)}=o(a^{(m)})$ separately, from the earlier argument we have $\sum_{m}I^{(m)}/ \sqrt{V} \asymp \sqrt{N\sum_{m}I^{(m)}/K} =O(1)$ in both cases . So we have, $\frac{\delta}{\sqrt{V}} \asymp \frac{K}{N\sqrt{V}} \asymp \frac{K}{N\sum_{m} I^{(m)}}=O(1)$. Hence all the terms in the right hand side of Equation (\ref{lcompute}) are $O(1)$ and consequently, $P(S_{N'} >0)$ is $O(1)$.

\end{proof}

Now we combine the results of these two lemmas to prove a lower bound on $\Theta^{ML}_{0}$.
\begin{lem}
Under the assumption that $\frac{N\sum_{m}I^{(m)}}{K} \rightarrow \infty $,
\begin{equation}
\inf_{\hat{z}} \sup_{z \in \Theta^{ML}_{0}} E[r(z,\hat{z})] \geq 
\exp \left(-(1+\epsilon_N)\frac{N\sum_{m}I^{(m)}}{K}\right)
\end{equation}
for some sequence $\epsilon_N=o(1)$. Moreover, if $\frac{N\sum_{m} I^{(m)}}{K}=O(1)$, then $\inf_{\hat{z}} \sup_{\Theta^{ML}_{0}} E[r(z,\hat{z})] \geq c$ for some constant $c>0$.
\label{lemma3}
\end{lem}

\begin{proof}
Since $\Theta^{ML}_{L} \subset \Theta^{ML}_{0}$, the minimax risk of $\Theta^{ML}_{0}$ is lower bounded by the minimax risk of $\Theta^{ML}_{L}$. Due to the fact that Bayes risk lower bounds
the global risk, we also have $\inf_{\hat{z}} \sup_{z \in \Theta^{ML}_{L}} E[r(z,\hat{z})] \geq \inf_{\hat{z}} \sup_{z \in \Theta^{ML}_{L}} B_{\tau}(\hat{z})$. Hence we have from Lemma \ref{lemmamll},
$$\inf_{\hat{z}} \sup_{z \in \Theta^{ML}_{0}} E[r(z,\hat{z})] \geq \inf_{\hat{z}} \sup_{z \in \Theta^{ML}_{L}} E[r(z,\hat{z})] \geq \inf_{\hat{z}} \sup_{z \in \Theta^{ML}_{L}} B_{\tau}(\hat{z}) = \inf_{\hat{z}} \sup_{z \in \Theta^{ML}_{L}} B_{\tau}(\hat{z}_1) .$$
\end{proof}

Now we need to obtain the minimax lower bound for the larger parameter space $\Theta^{ML}$ in the next lemma which concludes the proof for lower bound.
\begin{lem}
(Lower bound) Under the assumption that $\frac{N\sum_{m}I^{(m)}}{K} \rightarrow \infty $,
\begin{equation}
\inf_{\hat{z}} \sup_{\Theta^{ML}} E[r(z,\hat{z})] \geq 
\begin{cases}
\exp \left(-(1+\epsilon_N)\frac{N\sum_{m}I^{(m)}}{2}\right) & K=2 \\
\exp \left(-(1+\epsilon_N)\frac{N\sum_{m}I^{(m)}}{sK} \right) & K \geq 3
\end{cases}
\end{equation}
for some sequence $\epsilon_N=o(1)$ and some $s>0$. Moreover, if $\frac{N\sum_{m} I^{(m)}}{K}=O(1)$, then $\inf_{\hat{z}} \sup_{\Theta^{ML}_{0}} E[ r(z,\hat{z})] \geq c$ for some constant $c>0$.
\label{risklower}
\end{lem}

\begin{proof}
By the argument of \citet{zhang15}, for $K=2$, $\Theta^{ML}_0$ is the least favorable case for $\Theta^{ML}$. Hence we can keep the same lower bound for $\Theta^{ML}$ (obviously the lower bound holds since  $\Theta^{ML}_0$ is a subspace of $\Theta^{ML}$). However for $K \geq 3$, this is not the case and we can improve the lower bound. The least favorable case consists of the case where at least a constant proportion of communities are of the size $\frac{N}{sK}$. Define $\Theta^{ML}_{L}$ to contain all $z \in \Theta^{ML}$ such that a constant proportion of communities have size ${\left \lfloor{\frac{N}{K}} \right \rfloor}$, and another constant proportion of communities have size  ${\left \lceil{\frac{N}{K}} \right \rceil}$ and all other communities are much larger in size. Then using identical arguments as Lemmas 4 and 5 we have,

\begin{align*}
\inf_{\hat{z}} \sup_{z \in \Theta^{ML}} E[r(\bar{z},\hat{z})] & \geq \inf_{\hat{z}} \sup_{z \in \Theta^{ML}_{L}} B_{\tau}(\hat{z}_1) \\
& \geq  \epsilon P(\sum_{m}c^{(m)}\sum_{i=1}^{\left \lfloor{\frac{N}{sK}}\right \rfloor} X^{(m)}_{i} \geq \sum_{m}c^{(m)} \sum_{i=1}^{\left \lfloor{\frac{N}{sK}}\right \rfloor} Y^{(m)}_{i} ) \\
& \geq \exp(-(1+\epsilon_N) \frac{N\sum_{m}I^{(m)}}{sK}).
\end{align*}

Combining these two cases we have the result for the entire parameter space $\Theta^{ML}$.

\end{proof}

\subsubsection*{Upper bound}
To prove the upper bound, we develop a penalized likelihood type algorithm similar to \citet{zhang15} and show that its risk is upper bounded by the lower bound obtained in the previous step. We note that in the homogeneous MLSBM case ($\Theta^{ML}_0$ and $\Theta^{ML}_1$), i.e., when all the intra-community connection probabilities are $a^{(m)}/N$ and all the inter-community connection probabilities are $b^{(m)}/N$ for layer $m$, the log likelihood function  is
\begin{align*}
l(z;A)=&\sum_{m} \Bigg \{\log(\frac{a^{(m)}}{N})\sum_{i<j} A_{ij}^{(m)} 1\{z_i=z_j\} + \log(1-\frac{a^{(m)}}{N})\sum_{i<j} (1-A_{ij}^{(m)}) 1\{z_i=z_j\} \\
& + \log(\frac{b^{(m)}}{N})\sum_{i<j} A_{ij}^{(m)} 1\{z_i \neq z_j\} + \log(1-\frac{b^{(m)}}{N})\sum_{i<j} (1-A_{ij}^{(m)}) 1\{z_i \neq z_j\} \Bigg \}.
\end{align*}
The maximum likelihood estimator $\hat{z}^{MLE}$ is given by,
\begin{equation}
\hat{z}^{MLE} =\arg \max_{z} T(z),
\end{equation}
where $T(z)$ is given by
\begin{align}
T(z) &= \sum_{m} \bigg \{\log \left(\frac{a^{(m)}(1-b^{(m)}/N)}{b^{(m)}(1-a^{(m)}/N)} \right) \sum_{i<j} A_{ij}^{(m)} 1\{z_i=z_j\} - \log \left(\frac{1-b^{(m)}/N}{1-a^{(m)}/N}\right) 1\{z_i=z_j\} \Bigg \} \nonumber \\
 &=\sum_{m} \{ c^{(m)} A_{ij}^{(m)} 1\{z_i=z_j\} - k^{(m)}1\{z_i=z_j\} \},
 \label{penalmle}
\end{align}
with $c^{(m)}>0$ is defined in Lemma \ref{lemmamll} and $k^{(m)}
=\log \left(\frac{1-b^{(m)}/N}{1-a^{(m)}/N}\right)$.
However in general the parameter space will not be homogeneous. Under the more general parameter space $\Theta^{ML}$, we still define an identical form of the penalized likelihood estimator as $\hat{z}^{MLE}$.
Let $\bar{z}$ be the true class assignment and $\hat{z} \in \Theta^{ML}_{0}$ be an arbitrary class assignment satisfying $r(\bar{z},\hat{z})=R/N$, where $0 <R< N$ is a positive integer. Then note that
\begin{align}
T(\hat{z})-T(\bar{z}) &= (\sum_{m}  c^{(m)} A_{ij}^{(m)} 1\{\hat{z}_i=\hat{z}_j\} -\sum_{m}  c^{(m)} A_{ij}^{(m)} 1\{\bar{z}_i =\bar{z}_j\} ) \nonumber\\
& \quad - (\sum_{m} k^{(m)}1\{\hat{z}_i=\hat{z}_j\}-\sum_{m} k^{(m)}1\{\bar{z}_i=\bar{z}_j\}) \nonumber \\
&= (\sum_{m}  c^{(m)} A_{ij}^{(m)} 1\{(i,j) \in \gamma(\hat{z},\bar{z})\} -\sum_{m}  c^{(m)} A_{ij}^{(m)} 1\{(i,j) \in \alpha(\hat{z},\bar{z})\} ) \nonumber \\
& \quad  - \sum_{m} k^{(m)} (|(i,j) \in \gamma(\hat{z},\bar{z})|-|(i,j) \in \alpha(\hat{z},\bar{z})|), \label{Tzdiff}
\end{align}
where $\alpha(\hat{z},\bar{z})=\{(i,j) : i<j, \bar{z}_i=\bar{z}_j,\hat{z}_i \neq \hat{z}_j\}$ and $\gamma(\hat{z},\bar{z})=\{(i,j) : i<j, \bar{z}_i \neq \bar{z}_j,\hat{z}_i = \hat{z}_j\}$. Henceforth we will use shorthands $\alpha$ and $\gamma$ respectively to denote the sets.

Let $P_{R}=P (\hat{z} \in \Theta^{ML}_{0} :r(\bar{z},\hat{z})=R/N, T(\hat{z}) \geq T(\bar{z}))$. We want to bound $P_m$ which is the probability that an arbitrary class assignment $\hat{z}$ which does not agree with the truth $\bar{z}$ in exactly $R$ places (after permutations) can maximize $T(z)$, i.e., $P(T(\hat{z}) \geq T(\bar{z}))$. We start with the following lemma.
\begin{lem}
Let $\hat{z}$ be an arbitrary class assignment satisfying $r(\bar{z},\hat{z})=R/N$, where $0<R<N$ is a positive integer. Then there exists a sequence $\epsilon \rightarrow 0$, independent of $\hat{z}$, such that 
\begin{equation*}
P(T(\hat{z}) \geq T(\bar{z})) \leq \begin{cases}
\exp\left(-\frac{(1-\epsilon)NR\sum_{m}I^{(m)}}{K}+R^2\sum_{m}I^{(m)}\right), & \text{ if } R\leq \frac{N}{2K}, \\
\exp\left(-\frac{2(1-\epsilon)NR\sum_{m}I^{(m)}}{9K}\right), & \text{ if } R> \frac{N}{2K}.
\end{cases}
\end{equation*}
\label{Tzgap}
\end{lem}

\begin{proof}

Let $U^{(m)}=\{U_l^{(m)} \sim Bern(p^{(m)}_l)\}$, $V^{(m)}=\{V_l^{(m)} \sim Bern(q^{(m)}_l)\}$, $X^{(m)}=\{X_l^{(m)} \sim Bern(q^{(m)})\}$ and $Y^{(m)}=\{Y_l^{(m)} \sim Bern(p^{(m)})\}$ are sets of independent Bernoulli random variables for arbitrary $l$. Further let $\min p^{(m)}_l \geq p^{(m)} $ and $\max q^{(m)}_l \leq q^{(m)}$. Then we can define two sets of random variables $\{A^{(m)}_l \sim Bern(\frac{p^{(m)}}{p^{(m)}_l})\}$ and $\{B^{(m)}_l \sim Bern(\frac{q^{(m)}_l}{q^{(m)}})\}$ independent of $U$ and $V$. Now we define i.i.d copies $\{X^{(m)'}\}$ of $ \{X^{(m)} \}$ and $\{Y^{(m)'}\}$ of $ \{Y^{(m)} \}$ as $Y^{(m)'}_l=U_l^{(m)} A^{(m)}_l $ and $V_l^{(m)} =X^{(m)'}_l B^{(m)}_l $. Clearly, $ Y^{(m)'}_l=U_l^{(m)} A^{(m)}_l \leq U_l^{(m)}$ and  $V_l^{(m)} =X^{(m)'}_l B^{(m)}_l \leq X^{(m)'}_l  $. Hence we have for any real number $s$ and sequence of positive constants $\{c^{(m)}\}$,
\begin{align}
\text{if } &s+\sum_{m} c^{(m)} \sum_{l=1}^{|\alpha|} U_l^{(m)} \leq \sum_{m} c^{(m)} \sum_{l=1}^{|\gamma|} V_l^{(m)} \nonumber \\
\text{then } & s+\sum_{m} c^{(m)} \sum_{l=1}^{|\alpha|} Y_l^{(m)} \leq \sum_{m} c^{(m)} \sum_{l=1}^{|\gamma|} X_l^{(m)}. \label{coupling}
\end{align}
Now we replace $U_l^{(m)}$ and  $V_l^{(m)}$ with $A_{ij}^{(m)} 1\{(i,j) \in \alpha(\hat{z},\bar{z})\}$ and $A_{ij}^{(m)} 1\{(i,j) \in \gamma(\hat{z},\bar{z})\}$ respectively, $p^{(m)}_l$ and $q^{(m)}_l$ with $\pi^{(m)}_{z_iz_i}/N$ and $\pi^{(m)}_{z_iz_j}/N$ respectively (recall $\pi^{(m)}$ was previously defined in the main article as the matrix of block connection probabilities in the MLSBM's $m$th layer) , $p^{(m)}$ and $q^{(m)}$ with $a^{(m)}/N$ and $b^{(m)}/N$ respectively and $s$ with $\sum_{m} k^{(m)} (|\gamma|-|\alpha|)$. Then we get using the result in Equation (\ref{coupling}) and Equation (\ref{Tzdiff}),
\begin{align*}
P(T(\hat{z}) \geq T(\bar{z})) & \leq P \left( \sum_{m} c^{(m)} \sum_{l=1}^{|\gamma|} X_i^{(m)} - \sum_{m} c^{(m)} \sum_{l=1}^{|\alpha|} Y_i^{(m)} \geq \sum_{m} k^{(m)} (|\gamma|-|\alpha|) \right) \\
& = P \left( \exp \left( t \sum_{m} c^{(m)} \sum_{l=1}^{|\gamma|} X_i^{(m)} - t \sum_{m} c^{(m)} \sum_{l=1}^{|\alpha|} Y_i^{(m)} \right) \geq \exp \left( t\sum_{m} k^{(m)} (|\gamma|-|\alpha|) \right) \right) \\
& \leq \exp \left(- t\sum_{m} k^{(m)} (|\gamma|-|\alpha|)\right) \left( E[e^{t\sum_{m} c^{(m)} X_1^{(m)}}]\right)^{|\gamma|} \left( E[e^{-t\sum_{m} c^{(m)} Y_1^{(m)}}]\right)^{|\alpha|},
\end{align*}
where the last inequality follows from Markov inequality. Now we  choose $t=t^{*}=1/2$. Then we have 
\begin{align*}
E[e^{t^{*}\sum_{m} c^{(m)} X_1^{(m)}}] & =\prod_{m} \left(\frac{1-b^{(m)}/N}{1-a^{(m)}/N}\right)^{1/2} \left(\frac{(a^{(m)}b^{(m)})^{1/2}}{N}+(1-\frac{a^{(m)}}{N})^{1/2}(1-\frac{b^{(m)}}{N})^{1/2}\right)\\
& =\exp(\sum k^{(m)}/2) \exp (-\sum_{m} I^{(m)}/2)
\end{align*}
and $Ee^{-t\sum_{m} c^{(m)} Y_1^{(m)}} = \exp(-\sum k^{(m)}/2)$
 $\exp (-\sum_{m} I^{(m)}/2)$. Consequently, we have
\begin{equation}
 P(T(\hat{z}) \geq T(\bar{z})) \leq e^{-\frac{|\gamma|+|\alpha|}{2} \sum_{m} I^{(m)}}.
 \label{Tzeq}
 \end{equation} 

A lower bound on the size of the sets $\alpha$ and $ \gamma$ was given in Lemma 5.3 of \citet{zhang15}. We use the results directly here : for an arbitrary assignment  $\hat{z} \in \Theta^{ML}_{0}$ satisfying $r(\bar{z},\hat{z})=R/N$, where $0<R<N$ is a positive integer, we have
\begin{equation}
\min(|\alpha (\hat{z},\bar{z})|, |\gamma (\hat{z},\bar{z}) |) \geq \begin{cases}
\frac{(1-\epsilon)NR}{K}-R^2, & \text{ if } R\leq \frac{N}{2K}, \\
\frac{2(1-\epsilon)NR}{9K}, & \text{ if } R> \frac{N}{2K}.
\end{cases}
\label{gammabound}
\end{equation}
Using this lower bound for both $|\alpha|$ and $|\gamma|$ immediately yields the result.
\end{proof}

Let $\Gamma(z)$ denotes an equivalent class for $z$ consisting of all permutations of $z$. In order to use an union bound for $P_R$, we need to count the cardinality of the set of $\Gamma$s which have distance $R$ from $\bar{z}$.  Next we use Proposition 5.2 in \citet{zhang15} which states that 
$$|\{ \Gamma : \exists \hat{z} \in \Gamma \text { s.t } r(\bar{z},\hat{z})=R/N  \}| \leq \min \{ (\frac{eNK}{R})^R, K^N \},$$
to conclude through a union bound that,
\begin{align*}
P_{R} & := \{ \exists \hat{z} \in \Theta^{ML}_{0} \text{ s.t } r(\bar{z},\hat{z})=R/N,  T(\hat{z}) \geq T(\bar{z}) \}  \\
& \leq |\{ \Gamma : \exists z \in \Gamma \text{ s.t } r(\bar{z},\hat{z})=R/N  \}| \max_{z, r(\bar{z}, \hat{z})=R/N} P(T(\hat{z}) \geq T(\bar{z}))
\end{align*}

The next result uses the above results to establish the upper bound.

\begin{lem}
(Upper bound) Under the assumption that $\frac{N\sum_{m}I^{(m)}}{K \log K} \rightarrow \infty$, for the penalized maximum likelihood estimator $\hat{z}$ defined in Equation (\ref{penalmle}), we have
\begin{equation}
\sup_{\bar{z} \in\Theta^{ML}} E[r(\bar{z},\hat{z})] \leq \begin{cases} 
\exp(-(1+\epsilon_N)\frac{N\sum_{m}I^{(m)}}{2}), & K=2, \\
\exp(-(1+\epsilon_N)\frac{N\sum_{m}I^{(m)}}{s K}), & K\geq 3,
\end{cases}
\end{equation}
for some sequence $\epsilon_N=o(1)$ and $s \in [1,5/\sqrt{3}]$.
\label{riskupper}
\end{lem}
\begin{proof}
The proof technique is similar to \citet{zhang15}; we only modify the proof in places to suit our objective while keeping the approach the same. We first prove the result for the subspace $\Theta^{ML}_{0}$ and then extend it for $\Theta^{ML}$. We first consider the case $K \rightarrow \infty$, break the assumption $\frac{N\sum_{m}I^{(m)}}{K \log K} \rightarrow \infty$ into 3 parts and verify that in each case $E[r(\bar{z},\hat{z})]$ is bounded by a term of the form  $\exp(-(1+o(1))\frac{N\sum_{m}I^{(m)}}{sK})$. Let $\eta=o(1)$ be a universal sequence independent of $N$ that converges to 0. We note that 
$$N E[r(\bar{z},\hat{z})] \leq \sum_{R=1}^{N} R P_R .$$

(1) If $\liminf_{N \rightarrow \infty} \frac{N\sum_{m}I^{(m)}}{K \log N} >1$, there exists a small constant $\epsilon>0$ such that $\frac{(1-2\eta)N\sum_{m} I^{(m)}}{K \log N} >1+\epsilon$. Let $\eta$ decay slowly such that both $\frac{\eta N\sum_{m} I^{(m)}}{K \log K}$ and $\frac{\eta N}{K}$ 
go to infinity. Let $B=N \exp(-(1-3\eta)N \sum_{m} I^{(m)}/K)$. Clearly, $P_1 =eNK \exp(-(\frac{(1-\eta)N}{K}-1) \sum_{m} I^{(m)}) \leq B$. This follows by replacing both $\log (eK)$ and $\sum_{m} I^{(m)}$ by a bigger term, $\eta N\sum_{m}I^{(m)}/K$.

We will show that $ E[r(\bar{z},\hat{z}^{MLE})]$ is bounded by $O(B/N)$. First let $R \in [2,\frac{\epsilon N}{3K}]$. Then, 
\begin{align*}
P_R &\leq \left( \frac{eNK}{2} \exp \left(-\frac{(1-\eta)N\sum_{m} I^{(m)}}{K} +R\sum_{m} I^{(m)} \right) \right)^R \\
&= \left( \frac{eNK}{2} \exp\left(-\frac{(1-\eta)N\sum_{m} I^{(m)}}{K} \right)\right) \\
& \quad  \left( \frac{eNK}{2} \exp\left(-\frac{(1-\eta)N\sum_{m} I^{(m)}}{K} +(R+\frac{R}{R-1}) \sum_{m} I^{(m)} \right)\right)^{R-1} \\
& \leq N \exp \left(-(1-\eta)N \sum_{m} I^{(m)}/K+\log (eK)\right)\\
&  \quad \left( N \exp \left(-\frac{(1-\eta)N\sum_{m} I^{(m)}}{K} +2 \frac{\epsilon N}{3K} \sum_{m} I^{(m)}+\log (eK)\right)\right)^{R-1} \\
& \leq \left( N \exp\left(-(1-2\eta)N \sum_{m} I^{(m)}/K\right)\right) \left(N \exp\left(-(1+\epsilon) \log N+ \frac{2\epsilon (1+\epsilon) \log N}{3(1-2\eta)}\right)\right)^{R-1}\\
& \leq B N^{(1-(1+\epsilon)(1-3 \epsilon /4))(R-1)}\\
& \leq B N^{- \epsilon (R-1)/6}.
\end{align*}
The penultimate step follows by replacing $1-2\eta$ by $8/9$ and the last step follows since $\epsilon/4-3\epsilon^2/4 \geq \epsilon/6$ for large $N$ and small $\eta$ and $\epsilon$ respectively. Hence
\begin{equation}
N E[r(\bar{z},\hat{z})]= P_1+\sum_{R=2}^{\epsilon N/3K} R P_R \leq P_1 +\sum_{R=2}^{\infty} RBN^{-\epsilon (R-1)/6}= P_1 + B\frac{N^{\epsilon/6}}{(N^{\epsilon/6}-1)^2}=O(B).
\label{smallR}
\end{equation}
The infinite sum in the last step can be obtained by differentiation the infinite series sum $\sum_{R=1}^{\infty}N^{-\epsilon(R)/6}$ with respect to $N$.

 Next we show that the same conclusion holds for $R \in [\frac{\epsilon N}{3K},N]$. First, note that for any $\frac{N}{2K} \geq R \geq \frac{\epsilon N}{3K}$, we have $\frac{2(1-\eta)N\sum_{m} I^{(m)}}{9K} \leq \frac{(1-\eta)N\sum_{m} I^{(m)}}{K} -R\sum_{m} I^{(m)}$. Hence,
\begin{align*}
P_R &\leq \left( \frac{eNK}{\epsilon N/3K} \exp(-\frac{2(1-\eta)N\sum_{m} I^{(m)}}{9K}\right)^R \\
& \leq \left(\exp(-\frac{(1-2\eta)N\sum_{m} I^{(m)}}{9K} - \frac{N\sum_{m} I^{(m)}}{9K} + \log (\frac{3eK^2}{\epsilon}) \right)^9 \left(\frac{3eK^2}{\epsilon} \exp(-\frac{2(1-\eta)N\sum_{m} I^{(m)}}{9K}\right)^{R-9} \\
& \leq \exp \left(-(1-2\eta)N\sum_{m} I^{(m)}/K\right) \exp\left(-\frac{2(1-2\eta)N\sum_{m} I^{(m)}}{9K}-2 \frac{\eta N\sum_{m} I^{(m)}}{9K} +\log (\frac{3eK^2}{\epsilon})\right)^{R-9} \\
& \leq B \exp\left(-\frac{2(1-2\eta)N\sum_{m} I^{(m)}}{9K} \right)^{R-9}\\
& \leq B \exp(-\frac{2}{9}(1+\epsilon) \log N)^{R-9}\\
& \leq B N^{-2(1+\epsilon)(R-9)/9} \leq B N^{-2 (R-9)/9}.
\end{align*}
By the same reasoning as above, $\sum_{R=\epsilon N/3K}^{N} R P_R \leq \sum_{R=1}^{\infty} R P_R $ is $o(B)$. Hence combining this result with Equation (\ref{smallR}), $ N E[r(\bar{z},\hat{z})]=O(B)$

For the remaining two cases, (2) $\limsup \frac{N\sum_{m}I^{(m)}}{K \log N} <1$ and (3) $ \frac{N\sum_{m}I^{(m)}}{K \log N} =1+o(1)$, the proof follows from the corresponding cases in \citet{zhang15} ( Proof of Theorem 3.2). Hence we omit the details and only write the results.

 (2) If $\limsup_{N \rightarrow \infty} \frac{N\sum_{m}I^{(m)}}{K \log N} <1$,  then there exists a small constant $\epsilon>0$ such that $\frac{(1-\eta)N\sum_{m} I^{(m)}}{K \log N} >1-\epsilon$.  Define $R_0= N \exp (-(1-K^{-\epsilon/2}) (1-\eta) N \sum_{m} I^{(m)}/K)$  and $R'= N/ K^{1+\epsilon}$. We have,
 \begin{equation*}
 P_R \leq \begin{cases}
(\frac{eNK}{R_0} \exp (-\frac{(1-\eta)N\sum_{m} I^{(m)}}{K}+R' \sum_{m} I^{(m)}) )^{R} \leq \exp (- \frac{(1-\eta)NR \sum_{m} I^{(m)}}{2K^{1+\epsilon/2}}) & R_0 \leq  R \leq R' \\
(\frac{eNK}{R'} \exp (-\frac{2(1-\eta)N \sum_{m} I^{(m)}}{9K}))^R \leq  \exp(-\frac{NR\sum_{m}I^{(m)}}{9K}) &  R' < R \leq N
 \end{cases}
 \end{equation*}
 and hence from the proof in \citet{zhang15} $ E[r(\bar{z},\hat{z})]=\exp(-\frac{(1-o(1))N \sum_{m} I^{(m)}}{K})$.
 
  (3) If $ \frac{N\sum_{m}I^{(m)}}{K \log N} =1+o(1)$,  then there exists a positive sequence $w=o(1)$ such that $|\frac{N \sum_{m} I^{(m)}}{K \log N}-1| \ll w$ and $\frac{1}{\sqrt{\log N}} \leq w $.  defining $R_0= N \exp (-(1-w) N \sum_{m} I^{(m)}/K)$  and $R'= w^2 N/K$ we have,
 \begin{equation*}
 P_R \leq \begin{cases}
(\frac{eNK}{R_0} \exp (-\frac{(1-\eta)N\sum_{m} I^{(m)}}{K}+R' \sum_{m} I^{(m)}) )^{R} \leq \exp (\frac{w(1-\eta) NR \sum_{m} I^{(m)}}{4K}) & R_0< R \leq R' \\
(\frac{eNK}{R'} \exp (-\frac{2(1-\eta)N \sum_{m} I^{(m)}}{9K}))^R \leq  \exp(-\frac{NR\sum_{m}I^{(m)}}{9K}) &  R' < R \leq N
 \end{cases}
 \end{equation*}
 and hence from the proof in \citet{zhang15} $ E[r(\bar{z},\hat{z})]=\exp(-\frac{(1-o(1))N \sum_{m} I^{(m)}}{K})$.

The proof for finite $K$ is similar and hence omitted.

Now we prove the upper bound result for the entire parameter space $\Theta^{ML}$. The proof for the case $K\geq 3$ is similar to the proof for $\Theta^{ML}_{0}$ with the result in (\ref{gammabound}) being replaced by Lemma A.1. of \citet{zhang15}. However, for $K=2$, we proceed as in Section A.2. of \citet{zhang15} and assume without loss of generality that $\frac{N}{2}=\lfloor{\frac{N}{2} \rfloor}$. Let $r(\bar{z},\hat{z}) =R/N$ and define the sets $\alpha$ and $\gamma$ as before. Note that $R \leq N/2$ since distance between the two class assignments $d(\bar{z},\hat{z})= \min (d_H(\bar{z},\hat{z}),N-d_H(\bar{z},\hat{z}))$. We also have $|\alpha|+|\gamma| =R(N-R)$ if $r(\bar{z},\hat{z}) =R/N$ \citep{zhang15}. Hence from Equation (\ref{Tzeq}) we have,
\begin{equation}
P(T(\hat{z}) \geq T(\bar{z})) \leq \exp\left(-\frac{R(N-R) \sum_{m} I^{(m)}}{2}\right).
\end{equation}

The proof is similar to the one for $\Theta^{ML}_{0}$ and we only specify the specific results here omitting the technicalities. Let $0 \leq \epsilon \leq 1/8$ and recall that our assumption for $K=2$ case is that $\frac{N \sum_{m} I^{(m)}}{2} \rightarrow \infty$. We have the following 3 cases in parallel to the 3 cases earlier,

 (1) If $\frac{N \sum_{m} I^{(m)}}{2 \log N} > (1+\epsilon) $, defining $B= N \exp (-(N-1) \sum_{m} I^{(m)} /2)$ we have $P_1 \leq B$. The for $1< R \leq \epsilon N/2$ we have,
 \begin{align*}
 P_R & \leq (eN)^R \exp (-\frac{R(N-R) \sum_{m} I^{(m)}}{2})  \leq ((eN) \exp (-\frac{(N-\epsilon N/2) \sum_{m} I^{(m)}}{2})^R \\
 & \leq ( eN\exp (-(1-\epsilon/2)(1+\epsilon) \log N))^R \leq B N^{-\epsilon R/4},
 \end{align*}
 and for $ \epsilon N/2 <R \leq N/2$ we have,
  \begin{equation*}
 P_R \leq (\frac{2eN}{eN})^R \exp (-\frac{NR \sum_{m} I^{(m)}}{4}) \leq B \exp(-\frac{N(R-4) \sum_{m}I^{(m)}}{8}).
 \end{equation*}
 and hence $ E[r(\bar{z},\hat{z})]=(1+o(1))B/N$.
 
 (2) If $\frac{N \sum_{m} I^{(m)}}{2 \log N} < (1-\epsilon) $, defining $R_0= N \exp (-(1-e^{-\epsilon N \sum_{m} I^{(m)}/2}) N \sum_{m} I^{(m)}/2)$  and $R'= N \exp (- N \sum_{m} I^{(m)}/8)$ we have,
 \begin{equation*}
 P_R \leq \begin{cases}
(\frac{2eN}{R_0})^R \exp (-\frac{R(N-R') \sum_{m} I^{(m)}}{2}) \leq \exp (-e^{-\epsilon N \sum_{m} I^{(m)}/2} \frac{NR \sum_{m} I^{(m)}}{4}) & R_0< R \leq R' \\
(\frac{2eN}{R'})^R \exp (-\frac{NR \sum_{m} I^{(m)}}{4}) \leq  \exp(-\frac{N(R-4) \sum_{m}I^{(m)}}{16}) &  R' < R \leq N/2 
 \end{cases}
 \end{equation*}
 and hence $ E[r(\bar{z},\hat{z})]=(1+o(1))R_0/N$.
 
  (3) If $\frac{N \sum_{m} I^{(m)}}{2 \log N} =1+o(1) $, then there exists a positive sequence $w =o(1)$ such that $|\frac{N \sum_{m} I^{(m)}}{2 \log N}-1| \ll w$ and $\frac{1}{\sqrt{\log N}} \leq w $.  Defining $R_0= N \exp (-(1-w) N \sum_{m} I^{(m)}/2)$  and $R'= w^2 N$ we have,
 \begin{equation*}
 P_R \leq \begin{cases}
(\frac{2eN}{R_0})^R \exp (-\frac{R(N-R') \sum_{m} I^{(m)}}{2}) \leq \exp (-\frac{wNR \sum_{m} I^{(m)}}{4}) & R_0< R \leq R' \\
(\frac{2eN}{R'})^R \exp (-\frac{NR \sum_{m} I^{(m)}}{4}) \leq  \exp(-\frac{N(R-4) \sum_{m}I^{(m)}}{8} )&  R' < R \leq N/2 
 \end{cases}
 \end{equation*}
 and hence $ E[r(\bar{z},\hat{z})]=(1+o(1))R_0/N$.
\end{proof}

%\end{multicols}

\bibliography{msbm}

%\end{multicols}
\end{document}